\documentclass[conference]{IEEEtran}
\pdfoutput=1 

\author{\IEEEauthorblockN{Mercedes Garcia-Salguero\IEEEauthorrefmark{1},
Jesus Briales\IEEEauthorrefmark{2} and Javier Gonzalez-Jimenez\IEEEauthorrefmark{3}} \\
\IEEEauthorblockA{Machine Perception and Intelligent Robotics (MAPIR) Group, System Engineering and Automation Department, \\ 
University of Malaga, Campus de Teatinos, 29071 Malaga, Spain \\
Email: \IEEEauthorrefmark{1}mercedesgarsal@uma.es,
\IEEEauthorrefmark{2}jesusbriales@uma.es,
\IEEEauthorrefmark{3}javiergonzalez@uma.es}}

\usepackage{amsmath,amssymb,amsfonts, amsthm, bm}
\usepackage{mathtools}
 
\usepackage{graphicx}
\usepackage{comment}
\usepackage{xcolor}

\usepackage{mathtools}
\usepackage{epsfig}
\usepackage{algorithmic}
\usepackage{textcomp}
\usepackage{subcaption}
\usepackage[linesnumbered,ruled,vlined]{algorithm2e}

\usepackage{generalnotation}

\newcommand{\sphere}{\SSpace^2}

\newcommand{\Ess}{\matE} 
\renewcommand{\refl}{\matP}

\newcommand{\eightpt}{\method{8pt}}

\providecommand{\cross}[1]{\ensuremath{[\boldsymbol{#1}]_{\times}}}

\providecommand{\symm}[1]{\ensuremath{\mathbb{S}^{#1}}}
\providecommand{\symmPlus}[1]{\ensuremath{\mathbb{S}_+ ^{#1}}}

\providecommand{\Reals}[1]{\ensuremath{\mathbb{R}^{#1}}}

\newcommand{\fHatE}{ f(\hat{\Ess})}
\newcommand{\dHatRlx}{ d_{\text{R}}(\vHatLambda)}
\newcommand{\fHatRlx}{ f_{\text{R}}(\vHatX)}

\newcommand{\dOptRlx}{ d^{\star}_{\text{R}}}
\newcommand{\fOpt}{ f ^{\star}}
\newcommand{\fOptRlx}{ f ^{\star}_{\text{R}}}
\newcommand{\Me}{\mathcal{M}_{\boldsymbol{E}}}
\newcommand{\SetEssentialMatrices}{\mathbb{E}} 
\newcommand{\SetRelaxedEssentialMatrices}{\mathbb{E}_\text{R}}

\newcommand{\constrSet}{\mathcal{C}}
\newcommand{\constrSetRelaxed}{\constrSet_{\text{R}}}
\newcommand{\constrIdx}{\text{2}}

\newcommand{\Jset}{\boldsymbol{J}_{\constrSet}}
\newcommand{\JsetIdx}{\boldsymbol{J}_{\constrSet - \constrIdx}}
\newcommand{\vHatLambda}{\hat{\bm{\lambda}}}
\newcommand{\vHatX}{\hat{\boldsymbol{x}}}
\newcommand{\hatE}{\hat{\Ess}}
\newcommand{\vStarX}{\boldsymbol{x}^{\star}}
\newcommand{\vStarLambda}{\bm{\lambda}^{\star}}

\newcommand{\fE}{f(\boldsymbol{E})}

\newcommand{\tauGap}{\tau_{\text{gap}}}
\newcommand{\tauMu}{\tau_{\mu}}

\newcommand{\obsi}{\boldsymbol{f}_i}
\newcommand{\obsip}{\boldsymbol{f}_i'}

\newcommand{\manopt}{\textsc{manopt}\xspace}

\newcommand{\Cfirst}{\mathfrak{C}_{23}}
\newcommand{\Csecond}{\mathfrak{C}_{45}}
\newcommand{\Cthird}{\mathfrak{C}_{356}}
\newcommand{\Call}{\mathfrak{C}_{\text{all}}}

\newcommand{\Essential}{\boldsymbol{E}}

\newcommand{\uts}{\emph{up-to-scale}\xspace}
\newtheorem{theorem}{Theorem}[section]
\newtheorem{corollary}{Corollary}[theorem]

\newcommand{\figureSize}{1.\linewidth}

\newtheorem{definition}{Definition}[section]
\renewcommand{\vec}[1]{\text{vec}(#1)}

\title{Certifiable Relative Pose Estimation}

\begin{document}
\maketitle

\begin{abstract}

In this paper
we present the first 
fast optimality certifier for
the non-minimal version of the Relative Pose problem
for calibrated cameras from epipolar constraints.
The proposed certifier is based on Lagrangian duality and relies on a novel closed-form expression for dual points.
We also leverage an efficient solver that performs local optimization on the manifold of the original problem's non-convex domain. 
The optimality of the solution is then checked via our novel fast certifier.
The extensive conducted experiments demonstrate that, despite its simplicity,
this certifiable solver performs excellently on synthetic data, repeatedly attaining the (certified \textit{a posteriori}) optimal solution and shows a satisfactory performance on real data. 

\smallskip
\end{abstract}


\section{Introduction} %
In this work we consider the central calibrated relative pose problem in which, given a set of \textit{N} pair-wise feature correspondences between two images coming from two calibrated cameras, one seeks the relative pose (rotation $\rot$ and translation $\trans$, \uts) between both cameras (see Figure \eqref{fig:relPoseKneip}) that minimizes the epipolar error.

\begin{figure}[b]
    \centering
        \centering
        \includegraphics[width=0.8\linewidth]{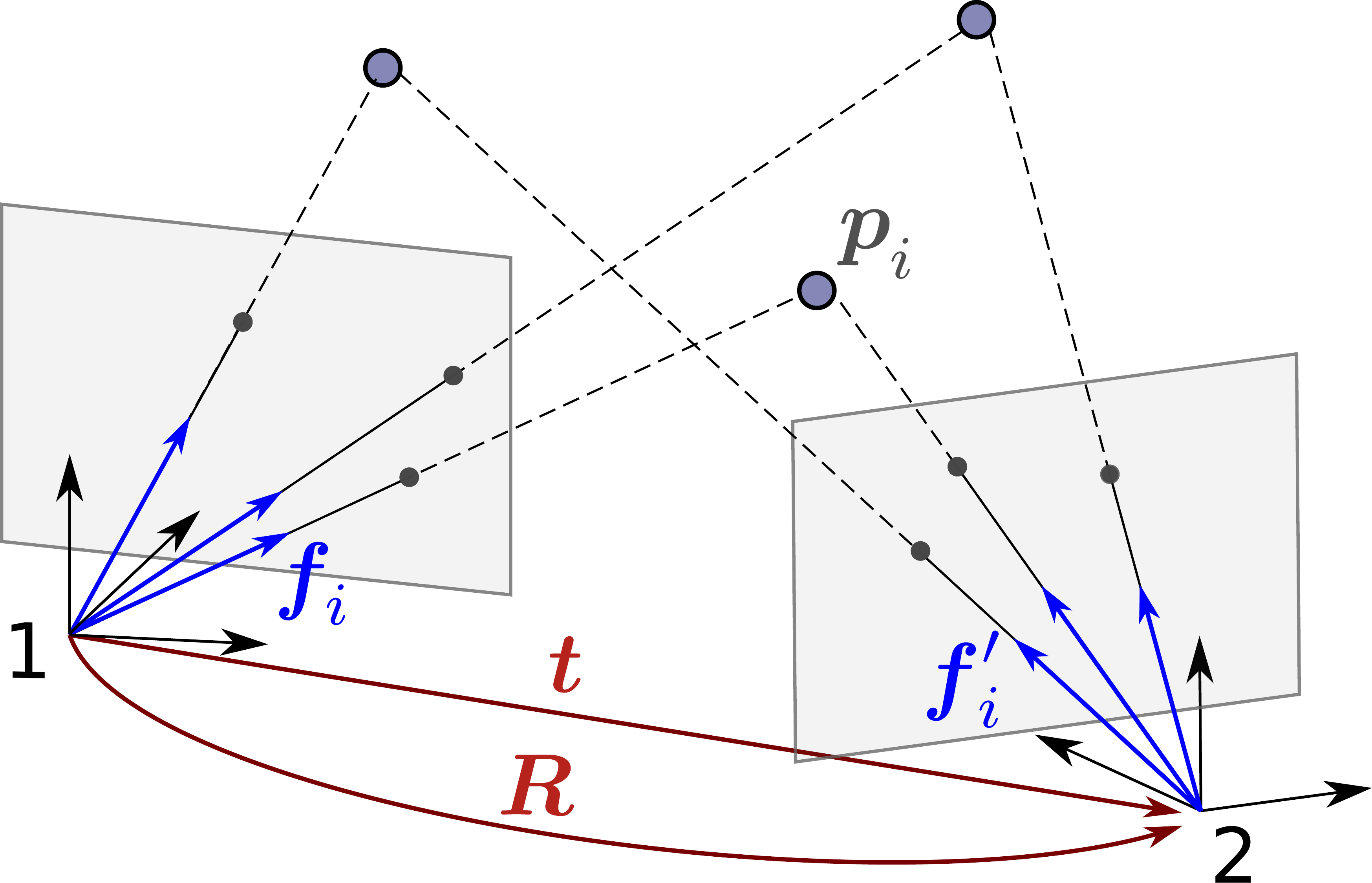}
        \caption{In the relative pose problem, we aim to estimate the relative relative rotation $\rot$ and the relative translation $\trans$ \uts between two calibrated cameras $1-2$ given a set of N correspondence pairs of unit bearing vectors $\{\obsip, \obsi\}_{i=1}^N$.}
    \label{fig:relPoseKneip}
\end{figure}

Estimating the relative pose between two calibrated views of a scene is specially relevant for visual odometry and also as a building block for more complex problems like Structure from Motion (SfM)~\cite{triggs1999bundle} or Simultaneous Localization and Mapping (SLAM)~\cite{mur2015orb, gomez2019pl}.

Whereas the gold standard for relative pose estimation is posing this as a 2-view Bundle Adjustment problem, this is a hard problem and it is common practice (see \eg~\cite{ozyecsil2017survey}) to bootstrap its initialization with a simpler formulation based on the epipolar error.

Despite this simplification, the problem to optimize is still non-convex and presents local minima~\cite{kneip2013direct}, which hinders the application of iterative approaches.
These suboptimal minima may lie arbitrarily far from the optimal solution, yet still explain the input data. Figure \eqref{fig:epipolarLinesExample} illustrates this situation, where the local minimal solution (green) leads to a relative pose $[\rot_{\text{loc}}, \trans_{\text{loc}}]$ far from the optimal solution (blue) $[\rot_{\text{opt}}, \trans_{\text{opt}}]$. 
Note that, 
in the presence of noise, the optimal solution may not longer be the ground truth pose.

 \begin{figure}[t]
    \centering
        \centering
        \includegraphics[width=\figureSize]{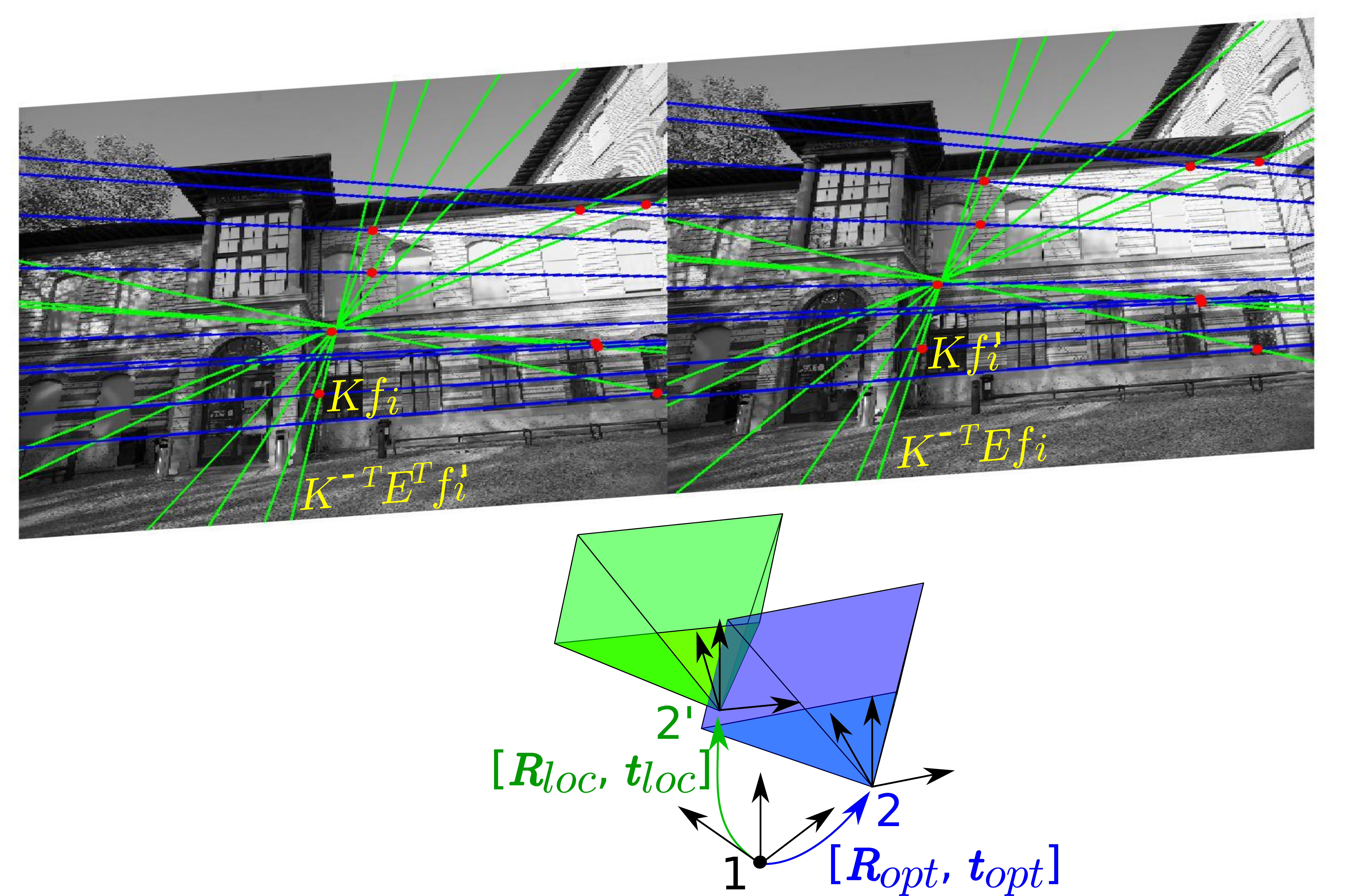}
    \caption{Suboptimal local minima (green) may lie far from the (globally) optimal solution (blue) and hinders subsequent algorithms, \textit{e.g.}, Bundle Adjustment, even if it agrees with the data $\{\obsi, \obsip \}_{i=1}^N$. For the sake of exactness, we depict the image counterparts of the data in pixels by applying the intrinsic camera parameters through $\boldsymbol{K}$, assuming both cameras have the same $\boldsymbol{K}$ and no lens distortion~\cite{hartley2003multiple}. 
   } 
    \label{fig:epipolarLinesExample}
\end{figure}
A suboptimal minima thus represents a wrong solution that, when passed to incremental methods can quickly cascade leading to failure. 
When passed to global methods, where the solution is averaged with many other estimations,
it provokes two negative effects. First, we are missing the opportunity to provide the method with more valid data, which, when healthy, improves the quality of the estimate. Second and most important, this generally far from correct solution will turn into anything from mild to gross outliers,
introducing biases and hindering the performance of global estimators in general. 
Thus, suboptimal local minima should be detected and avoided, yet the Relative Pose problem is still non convex and hard to solve globally.

In this context, a recent line of research has evinced that for some so-far-considered hard problems
(in Geometric Computer Vision but also in other fields \cite{bandeira2016note}),
though worst-case instances can remain intractable in terms of resolution or for the certification of optimality, 
real-world instances do not usually tend to these worst-cases. 
Interestingly, for many problem instances found in practice
it is possible to attain and even certify optimality.

Certifiable algorithms can be attained in multiple ways.
Perhaps one of the most straightforward realizations
consists of characterizing a \emph{tight} convex relaxation
for the given problem instance
and jointly solving for its primal and dual problems~\cite{boyd2004convex},
recovering at the same time a solution to the original problem
and a (dual) certificate of optimality.
This is the recent proposal of Briales \etal \cite{briales2018certifiably} or Zhao \cite{Zhao2019} for the Relative Pose problem, where they propose a (\emph{probably}) tight \textit{Semidefinite Problem} (SDP) relaxation that can be solved efficiently (in polynomial time).

While the approach above is simple and often provides a certifiable solution to the Relative Pose problem,
it is not the only nor the most efficient way to devise a certifiable solver \cite[Sec.~1]{bandeira2016note}.
One can devise a much faster certifiable approach 
by combining a fast solver for the original problem
(one that returns the optimal solution with high probability but no guarantees)
with a fast standalone certification method that produces an optimality (dual) certificate leveraging this solution (see \eg \cite{briales2016fast, carlone2015duality}), which finally brings us to the core contribution of the present work.

\textbf{Contributions}
In this line, we conceive a novel closed-form (linear) approach which allows to certify \textit{a posteriori} if the potentially optimal solution to a Relative Pose problem instance is indeed the optimal one.
With this certifier available,
we unblock the ability to build faster certifiable solvers
in the fashion proposed by Bandeira \cite{bandeira2016note}.
\ie by combining fast heuristic solvers with a fast optimality certifier.
To prove the value of this approach in the context of the Relative Pose problem,
we propose
a novel, simple and efficient iterative Riemannian Trust-Region solver that operates directly on the essential matrix manifold and tends to return the optimal solution when initialized, for example, 
with the classical 8-points (\eightpt) algorithm. The conducted experiments (in Section \eqref{sec:experiments}) show that, in practice, one can bootstrap the iterative method with the trivial identity matrix or a random essential matrix and still retrieve the optimal solution, mostly when considering more than 40 correspondences.
Combining both, we get a novel \emph{certifiable} approach for solving the Relative Pose problem. This pipeline represents our main practical contribution and we refer the reader who is only interested in its application to Section \eqref{sec:proposedPipeline} for a concise explanation.
Moreover, given the simplicity of its component blocks,
we see great potential on this kind of pipeline
to be streamlined,
in order to achieve excelling computational times, so that it becomes the new go-to state-of-the-art solver for the community.

The main technical contributions contained in the paper that were required to achieve the above are:

\begin{itemize}
  \item
  Characterize a family of relaxed quadratic formulations
  for the Relative Pose problem,
  whose Karush–Kuhn–Tucker (KKT) conditions fulfill
  the Linear Independence Constraint Qualification (LICQ) (in Section \eqref{sec:problemformulation}).
  \item
  Based on the above, design a fast approach
  to compute the potential dual candidate solution in closed-form,
  given the (potentially) optimal solution to the original problem (in Section \eqref{sec:lagrangiandual}).
  \item
  Define how to perform optimality certification,
  given the candidate dual solution from the approach above (in Algorithm \eqref{alg:naiveverification}).  
  \item
  Develop the required calculus 
  to implement the new iterative solver taking advantage of the optimization prototyping framework \manopt~\cite{manopt} (in Section \eqref{sec:proposedPipeline}). 
\end{itemize}

Extensive experiments with both synthetic and real data, covering a broad set of problem regimes, support the claims of this paper and show that our proposed pipeline performs excellently on synthetic data, consistently reaching the optimal solution with few iterations of the iterative solver when initialized with the \eightpt algorithm and certifying this optimality, for all but a few exceptional (0.53\%) cases among all tested problem instances. The preliminary results on real data show a satisfactory performance, while still leaving margin for future improvement. 
Note that, although this empirically supports that the strong duality condition usually holds and that the proposed relaxed formulation for the Relative Pose problem is indeed tight, a formal demonstration is not available (yet).

Finally, please notice that, although the proposed pipeline estimates the essential matrix, the relative pose (rotation and translation) can be recovered from it by classic Computer Vision algorithms~\cite{hartley2003multiple}.

\section{Related Work} \label{sec:related-work}

\subsection{Minimal Solvers}
The essential matrix has five degrees of freedom (three from 3D rotation, three from 3D translation and one less from the scale ambiguity) and therefore, only five correspondences (except for degenerate cases~\cite{hartley2003multiple}) are required for its estimation. This is the so-called minimal problem and since it provides us with an efficient hypothesis generator, it can be embedded into RANSAC paradigms to gain robustness against wrong correspondences, \ie outliers~\cite{kneip2013direct, botterill2011refining}. In this context, different works~\cite{stewenius2006recent, nister2004efficient} have reported efficient algorithms to solve this minimal problem, although they involve nontrivial (tenth degree) polynomial systems which are commonly solved by methods based on polynomial ideal theory and Gröbner basis, which are not always numerically robust~\cite{kukelova2007two}. Alternative approaches have tried to overcome this instability, such as \cite{kukelova2008polynomial}, where it was proposed an eigenvalue-based method, more stable than state-of-art approaches~\cite{stewenius2006recent, nister2004efficient} for the 5 and 6 points algorithms.

\subsection{8-point Algorithm}
The seminal work of Longuet-Higgins in \cite{longuet1981computer} showed for the first time that the relative pose between two calibrated views is encoded by the essential matrix and proposed the (linear) 8-point (\eightpt) algorithm, which led to many linear and nonlinear algorithms, among others
\cite{hartley1997defense, helmke2007essential, ma2001optimization}. Despite being designed for the fundamental matrix estimation, the \eightpt algorithm can be adapted to the essential matrix \ie, for calibrated cameras. Special attention must be given to the celebrated \emph{normalized 8-point} algorithm proposed by Hartley in \cite{hartley1997defense}. 
However, in both cases the solution is not guaranteed to be an essential matrix~\cite{hartley2003multiple}, but an approximation. Nevertheless, due to its simplicity, the \eightpt algorithm can be considered as the state-of-the-art initialization for further refinements.

\subsection{Iterative Optimization on the Essential Matrix Manifold}
Minimal solvers or the 8-point algorithm typically provide suboptimal solutions for the non-minimal N-point problem and therefore it is a common practice to refine these initial estimates by local, iterative methods~\cite{botterill2011refining}. 
Contrary to optimization problems on flat (Euclidean) spaces, these local optimization methods must respect the intrinsic constraints of the search space. 
In this context,
the essential matrix manifold has been characterized via different, yet (almost) equivalent formulations.
As it was shown in \cite{helmke2007essential}, these parameterizations may lead to different performances and convergence rates for non-linear optimization methods.
In \cite{lui2013iterative}, an iterative method built upon the 5 points estimation, which directly solves for the relative rotation, was proposed, achieving frame-rate speed.
In \cite{ma2001optimization} it was proposed the refinement of the initial estimation from the \eightpt algorithm on the manifold of the essential matrices, although the approach only converges in a small neighborhood of the true solution for their chosen manifold parameterization. 
Helmke \etal~\cite{helmke2007essential} improved the convergence properties of the iterative solver by proposing a different parameterization of said manifold.  
Recently, Tron and Daniilidis~\cite{tron2017space} present a characterization of the essential matrix manifold as a quotient Riemannian manifold which takes into account the symmetry between the two views and the peculiarities of the epipolar constraints. Interestingly, the reported optimization instances were able to converge in five iterations.

\subsection{Globally Optimal Solvers}
Despite its attractive as fast solvers, the above-mentioned proposals do not guarantee nor certify if the retrieved solution is optimal. In fact, finding said guaranteed optimal solutions for non-convex problems, such as the Relative Pose problem, is in general a hard task.
In \cite{yang2014optimal}, the authors extended the approach in \cite{tron2017space} to incorporate the presence of outliers as an inlier-set maximization problem, which was solved in practice via Branch-and-Bound global optimization. 
In \cite{hartley2007global}, it was first proposed the estimation of the essential matrix under a $L_{\infty}$ cost function by Branch-and-Bound.
In \cite{kneip2013direct}, an eigenvalue formulation equivalent to the algebraic error was proposed and solved in practice by an efficient Levenberg-Marquardt scheme and a globally optimal Branch-and-Bound.
Nonetheless, Branch-and-Bound presents slow performance and exponential time in worst-case scenarios.

A different approach which also certifies the optimality of the solution \textit{a posteriori}, relies on the re-formulation of the original problem as a Quadratically Constrained Quadratic Program (QCQP).
QCQP problems are still in general NP-hard to solve. However, one can relax this QCQP into a Semidefinite Relaxation Program (SDP) by Shor's relaxation \cite{boyd2004convex, ding2007efficient} and solve this SDP by off-the-shelf tools in polynomial time.
If the convex relaxation happens to be tight, one can recover the solution to the original problem with an optimality certificate.
This was the approach followed 
in \cite{briales2018certifiably}, where it was shown for the first time
that the non-minimal (epipolar) Relative Pose problem for calibrated cameras
could be formulated as a quadratic (QCQP) problem
whose Shor's relaxation 
resulted in an empirically tight (SDP) convex relaxation.
Beyond its value as a proof-of-concept convex approach,
the QCQP formulation chosen by the authors there resulted in
a quite large SDP problem,
which led to computation times of 1 second with a \textsc{matlab} implementation
using SDPT3 \cite{toh1999sdpt3} as solver. 
A subsequent contribution by Zhao \cite{Zhao2019} follows a similar approach
and proposes an alternative (still equivalent) QCQP formulation,
featuring a much smaller number of variables and constraints
(with the essential matrix at its core).
Applying Shor's relaxation to this formulation results in
a much smaller, although not always tight, SDP relaxation
with 78 variables and 7 constraints.
Thanks to the significantly smaller size of the SDP problem, a C++ based implementation 
and by fine-tuning an off-the-shelf solver like SDPA \cite{yamashita2003implementation}
they attain an efficient solver with times around 6ms.
This is fast by SDP solver standards,
but not as fast as desirable by real-time Computer Vision standards~\cite{mur2015orb, song2013parallel}.
Although tractable, solving these convex problems from scratch may not be the most efficient way to obtain a solution. 
As an alternative approach one may found the so-called Fast Certifiable Algorithms
recently characterized and motivated in \cite{bandeira2016note, yang2020teaser}. These algorithms typically leverage the existence of an optimality certifier
which, given the optimal solution obtained by any means,
may be able to compute a (dual) certificate of optimality from it.
A straightforward approach to get such dual certificate relies on the resolution of the \textit{dual problem}~\cite{boyd2004convex} from scratch, whose optimal cost value always provides a lower bound on the optimal objective for the original problem. In many real-world problem instances this bound is \emph{tight}, meaning both cost values are the same up to some accuracy and one can certify optimality from it. However, this naive approach would still be as slow as directly solving the problem via its convex relaxation.

On the other hand, in the context of Pose Graph Optimization it has been shown that the (potentially optimal) candidate solution to the original problem can be leveraged to obtain a candidate dual certificate in closed-form, providing a much faster way to solve the dual problem ~\cite{carlone2015duality, briales2016fast, carlone2015lagrangian}.
This enables a fast optimality verification approach with which we can augment fast iterative solvers with no guarantees into Fast Certifiable Algorithms~\cite{briales2017cartan,rosen2019se} while maintaining their efficiency. 
For the Relative Pose problem, though, such standalone fast optimality certifier has not been proposed yet, as none of the SDP relaxations previously proposed~\cite{briales2018certifiably,Zhao2019} allow for computing a dual candidate in closed-form as there exists for the Pose Graph Optimization case.


\section{Notation}
In order to make clearer the mathematical formulation in the paper, we first introduce the notation used hereafter. Bold, upper-case letters denote matrices, \eg $\boldsymbol{E, C}$; bold, lower-case denotes (column) vector, \eg $\trans, \boldsymbol{x}$; and normal font letters denote scalar, \eg $a, b$.
We reserve $\lambda$ for the Lagrange multipliers (Section \eqref{sec:lagrangiandual}) and $\mu$ for eigenvalues. Additionally, we will denote with $\mathbb{R}^{n \times m}$ the set of $n \times m$ real-valued matrices, $\symm{n} \subset \mathbb{R}^{n \times n}$ the set of symmetric matrices of dimension $n \times n$ and $\symmPlus{n}$ the cone of positive semidefinite (PSD) matrices of dimension $n \times n$. A PSD matrix will be also denoted by $\succeq$ , \ie, $\boldsymbol{A} \succeq 0 \Leftrightarrow \boldsymbol{A} \in \symmPlus{n}$. We denote by $\kron$ the Kronecker product and by $\boldsymbol{I}_n$ the (square) identity matrix of dimension $n$.
The operator $\vec{\matB}$ vectorizes the given matrix $\boldsymbol{B}$ column-wise. We denote by $\cross{t}$ the matrix form for the cross-product with a 3D vector $\trans = [t_1, t_2, t_3]^T$, \ie, $\trans \times (\bullet) = \cross{t} (\bullet)$ with
\begin{equation}
    \cross{t} = 
    \begin{bmatrix}
    0 & -t_3 & t_2 \\
    t_3 & 0 & -t_1 \\
    - t_2 & t_1 & 0 
    \end{bmatrix}.
\end{equation}
Last, we employ the subindex $R$ across the board to indicate a relaxation of the element \wrt the element without subindex. For example, we denote by $\SetRelaxedEssentialMatrices$ the set that is a relaxation of $\SetEssentialMatrices$ and therefore, a superset of the latter, \ie $\SetRelaxedEssentialMatrices \supset \SetEssentialMatrices$.

\section{Relative Pose Problem Formulation} \label{sec:problemformulation}

We consider the central calibrated Relative Pose problem in which one seeks the relative rotation $\rot$ and translation $\trans$ between two cameras, given a set of $N$ pair-wise feature correspondences between the two images coming from these distinctive viewpoints. In this work, the pair-wise correspondences are defined as pairs of (noisy) unit bearing vectors $(\obsi, \obsip)$ which should point from the corresponding camera center towards the same 3D world point.
A traditional way to face this problem is by introducing the essential matrix $\Ess$~\cite{longuet1981computer, hartley2003multiple}, a $3 \times 3$ matrix which encapsulates the geometric information about the relative pose between two calibrated views.
The essential matrix relates each pair of corresponding points through the \textit{epipolar constraint} $\obsi^T \Ess \obsip = 0$, provided observations are noiseless. With noisy data, however, the equality does not hold and $\obsi^T \Ess \obsip  = \epsilon_i$ defines what is called the \emph{algebraic error}. 

In the literature one can find previous works which exploit this relation and seek the essential matrix $\boldsymbol{E}$ that minimizes the squared algebraic error $\epsilon ^2$ and its variants~\cite{ma2001optimization, Zhao2019, helmke2007essential}. 
We will follow this approach and address the Relative Pose problem as an optimization problem. The cost function can be written as a quadratic form of the elements in $\boldsymbol{E}$ by defining the positive semi-definite matrix 
$\symmPlus{9} \ni \boldsymbol{C} = \sum_{i=1} ^N \boldsymbol{C}_i$, with $\boldsymbol{C}_i =  (\obsip\kron \boldsymbol{f}_i)(\obsip \kron \obsi)^T \in \symmPlus{9}$. Formally, the Relative Pose problem reads:

\begin{equation}
    \tag{O} \label{eq:originalproblem}
    \fOpt = \min_{\Ess \in \SetEssentialMatrices} \underbrace{\sum_{i = 1} ^N (\obsi^T \Ess \obsip ) ^2}_{\fE} = 
    \min_{\Ess \in \SetEssentialMatrices} \vec{\matE}^T \boldsymbol{C} \vec{\matE}.
\end{equation}
See the \suppl 
for a formal proof of this equivalence.

\subsection{The Set of Essential Matrices}
$\SetEssentialMatrices$ above stands for the set of (normalized) essential matrices, typically defined as
\begin{equation}
\label{eq:Me:[t]xR}
    \SetEssentialMatrices \doteq \{ \Ess \in \Reals{3 \times 3} \;|\; \Ess = \cross{t} \rot , \rot\in \rotSpace,\ \trans \in \sphere \}.
\end{equation}
Note that in \eqref{eq:Me:[t]xR} the translation is identified with points in the 2-sphere $\sphere \doteq \{ \trans \in \Reals{3} | \trans^T \trans = 1 \}$ since the scale cannot be recovered for central cameras~\cite{hartley2003multiple}. The rotation is represented by a $3 \times 3$ orthogonal matrix with positive determinant $\rot \in \rotSpace$ and $\rotSpace \doteq \{ \rot \in \Reals{3 \times 3} | \rot^T \rot = \boldsymbol{I}_3,\ det(\rot) = +1$ \}.

Other equivalent parameterizations are possible for this set~\cite{tron2017space,helmke2007essential}. E.g. Faugeras and Maybank~\cite{faugeras1990motion} proposed:
\begin{equation}
    \SetEssentialMatrices \doteq \{ \Ess \in \Reals{3 \times 3} \;|\; \Ess \Ess ^T = \cross{t}\cross{t}^T , \trans^T \trans = 1\}
    . \label{eq:Me:EEt}
\end{equation}
This parameterization, recently leveraged by Zhao in~\cite{Zhao2019}, features a lower number of variables (12) and constraints (7), yet it provides excellent results in the context of building SDP relaxations for the relative pose problem~\cite{Zhao2019}, resulting in a smaller problem than relaxations based on previous formulations \eqref{eq:Me:[t]xR}~\cite{briales2018certifiably}.

\subsection{Relaxed Formulation of the Relative Pose Problem}
Despite its advantages, the parameterization by Faugeras and Maybank \eqref{eq:Me:EEt} above
still does \emph{not} allow for the development of a \emph{fast optimality certifier} for the Relative Pose problem
in the fashion of that proposed \eg for Pose Graph Optimization in \cite{briales2016fast,carlone2015lagrangian}.
To attain this kind of certifiers,
it will be necessary to leverage a \emph{relaxed} version $\SetRelaxedEssentialMatrices$ of the essential matrix set $\SetEssentialMatrices$: 

\begin{equation}
    \SetEssentialMatrices \subset \SetRelaxedEssentialMatrices \doteq \{ \Ess \in \Reals{3 \times 3} \;|\; h_i(\Ess,\trans)=0, \forall h_i \in \constrSetRelaxed;\ \trans \in \Reals{3}\},
    \label{eq:Me:min}
\end{equation}
with $\constrSetRelaxed$ the \emph{relaxed constraint set} defined as
\begin{equation}
\constrSetRelaxed \equiv \begin{cases}
   h_1 \equiv \boldsymbol{t}^T \boldsymbol{t} - 1 = 0\\
   h_2 \equiv \boldsymbol{e}_1^T \boldsymbol{e}_1 - (t_2^2 + t_3^2) = 0\\
   h_3 \equiv \boldsymbol{e}_2^T \boldsymbol{e}_2 - (t_1^2 + t_3^2) = 0\\
   h_4 \equiv \boldsymbol{e}_3^T \boldsymbol{e}_3 - (t_1^2 + t_2^2) = 0\\
   h_5 \equiv \boldsymbol{e}_1^T \boldsymbol{e}_3 + t_1t_3 = 0\\
   h_6 \equiv \boldsymbol{e}_2^T \boldsymbol{e}_3 + t_2t_3 = 0
\end{cases},
\end{equation}
where we have denoted the rows of $\Ess$ by $\boldsymbol{e}_i \in \Reals{3}, \forall i \in \{1, 2, 3\}$.

These are almost the same constraints used by Zhao in~\cite{Zhao2019}, but we dropped the constraint $\boldsymbol{e}_1^T \boldsymbol{e}_2 + t_1t_2 = 0$.
Even though this constraint set differs from that by Faugeras and Maybank by just one constraint (6 versus 7),
it turns out this difference is instrumental to eventually enable our \emph{fast optimality certifier}, 
as motivated later in Section~\ref{sec:closed-form-LICQ}.
A formal proof of how $\SetRelaxedEssentialMatrices$~in~\eqref{eq:Me:min} defines a strict superset of $\SetEssentialMatrices$ is provided in the \suppl. 

With this relaxed set at hand, we define a \emph{relaxed} version \eqref{eq:relaxedproblem} of the original Relative Pose problem \eqref{eq:originalproblem}:
\begin{equation}
    \tag{R} \label{eq:relaxedproblem}
    \fOptRlx = 
    \min_{\Ess \in \SetRelaxedEssentialMatrices} \vec{\matE}^T \boldsymbol{C} \vec{\matE}.
\end{equation}
Problem \eqref{eq:relaxedproblem} is a relaxation of the original Problem \eqref{eq:originalproblem} and therefore the inequality $\fOptRlx \leq \fOpt$ holds, with equality only if the solution to \eqref{eq:relaxedproblem} is also an essential matrix, and hence feasible for \eqref{eq:originalproblem}.

Interestingly enough though,
we observed that equality holds ($\fOptRlx = \fOpt$) in many problem instances in practice,
meaning that the relaxed problem \eqref{eq:relaxedproblem} is very often a tight relaxation of the original problem \eqref{eq:originalproblem}.
We have no theoretical proof as to why the behavior above holds so often,
and our support to this claim is fundamentally empirical (given by extensive experiments in Section~\eqref{sec:experiments}).

There exists other examples in the literature of problem relaxations that remain tight for common instances,
such as the relaxation of $\rotSpace$ onto $\ortSpace$ in the context of Pose Graph Optimization \cite{briales2016fast,carlone2015lagrangian}.
Yet, it is remarkable that whereas $\ortSpace~\supset~\rotSpace$ has two disjoint components,
$\SetRelaxedEssentialMatrices~\supset~\SetEssentialMatrices$ here features a single connected component,
which makes the observed behavior less expectable.

\subsection{Relaxed QCQP Formulation}
We can now re-formulate our relaxed optimization problem~\eqref{eq:relaxedproblem} as a canonical instance of QCQP by writing explicitly the constraints in $\constrSetRelaxed$.
%
Let us define for convenience the 9D vector $\boldsymbol{e} = \vec{\matE} \in \Reals{9}$ and the 12-D vector $\boldsymbol{x} = [\boldsymbol{e}^T, \trans^T ] ^T$. 
The relaxed canonical QCQP formulation employed in this work for the Relative Pose problem (also referred to as the \emph{primal problem}) is
\begin{align}
    \fOptRlx = & \min_{\boldsymbol{x} \in \Reals{12}} \boldsymbol{x} ^T  \boldsymbol{Qx}  \nonumber \\
    & \text{subject to}
    \begin{aligned}
        & \boldsymbol{x}^T \boldsymbol{A}_1 \boldsymbol{x} = 1 \\
        & \boldsymbol{x}^T \boldsymbol{A}_i \boldsymbol{x} = 0, \quad i = 2, ..., 6
    \end{aligned}
    \tag{P-R}\label{eq:primalproblem}
\end{align}
where $\{\boldsymbol{A}_i\}_{i=1} ^6$ are the $12 \times 12$-symmetric corresponding matrix forms of the quadratic constraints, so that $h_i(\Ess,\trans) \equiv \boldsymbol{x}^T \boldsymbol{A}_i\boldsymbol{x} - c_i = 0, c_i \in \Reals{}$, and $\boldsymbol{Q}$ is the $12 \times 12$-symmetric data matrix of compatible dimension with $\boldsymbol{x}$, \ie
\begin{equation}
    \boldsymbol{Q} = 
    \begin{bmatrix}
    \boldsymbol{C}  & \boldsymbol{0}_{9\times 3} \\
    \boldsymbol{0}_{3 \times 9} & \boldsymbol{0}_{3 \times 3} 
    \end{bmatrix} 
    \in \symmPlus{12}.
\end{equation}

Problem \eqref{eq:primalproblem} is exactly equivalent to the relaxed Problem \eqref{eq:relaxedproblem}.
Nonetheless, Problem \eqref{eq:primalproblem} is still a Quadratically Constrained Quadratic Program (QCQP), in general NP-hard to solve. However, it allows us to derive an optimality certifier by exploiting the so-called \textit{Lagrangian dual problem}, which we present next.

\begin{figure*}[h]
    \centering
    \includegraphics[width=\figureSize]{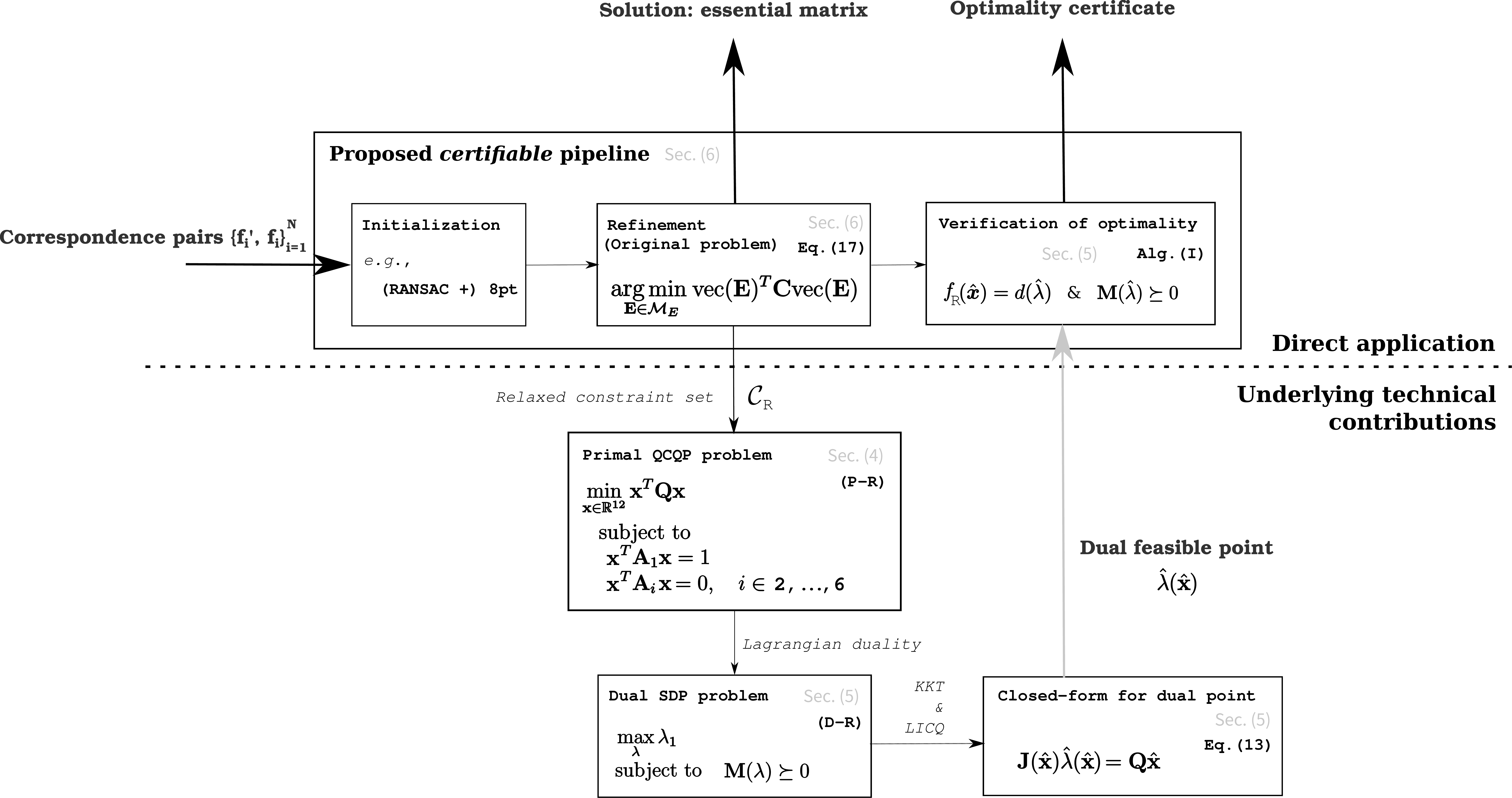}
    \caption{Proposed \textit{certifiable} pipeline: its user-supplied input \textit{Correspondence pairs} $\{\obsip, \obsi\}_{i=1}^N$; and two outputs (\textit{essential matrix} and \textit{optimality certificate}). We show the underlying formulations (primal and dual problems), along with the novel closed-form expression for potential dual points $\vHatLambda$ given a potential optimal primal point $\vHatX$. This dual point is \emph{not} provided by the user.}
    \label{fig:formulation-pipeline}
\end{figure*}

\section{Fast Optimality Certifier} \label{sec:lagrangiandual}
Our interest in the dual problem is twofold. First, the dual problem presents a relaxation of the primal program \eqref{eq:primalproblem} and hence provides a lower bound for the optimal objective of the latter, principle known as \textit{weak duality}~\cite{boyd2004convex}. In many situations, as it is shown in \SEC{sec:syntheticexperiments}, this relaxation is exact, meaning that the optimal objective of the dual ($\dOptRlx$) and primal ($\fOptRlx$) problems are the same up to some accuracy. When this occurs, we say there is \textit{strong duality} and that the \textit{duality gap} $\fOptRlx - \dOptRlx$ is zero. Second, when strong duality holds, we can recover the primal solution from the dual and vice versa (assuming some conditions hold), without actually solving the primal (or dual) problem. 

\begin{definition}[Dual problem of the Primal program \eqref{eq:primalproblem}]
\label{th:dualproblem}
    The Dual problem of the program in \eqref{eq:primalproblem} is the following constrained SDP:
    \begin{align}
    \dOptRlx &=  \max_{\bm{\lambda} } \lambda_1  \tag{D-R} \label{eq:dualproblem} \\
    & \text{subject to  }  \boldsymbol{M}(\bm{\lambda})\succeq 0  \nonumber
\end{align}
where $ \boldsymbol{M}(\bm{\lambda}) \doteq \boldsymbol{Q} - \sum_{i=1} ^6 \lambda_i \boldsymbol{A}_i$ is the so-called \textit{Hessian of the Lagrangian} and $\bm{\lambda} = \{\lambda_i\}_{i=1} ^6$ are the \textit{Lagrange multipliers}. 
The derivation of this problem is given in the \suppl. 
\end{definition}

A classic duality principle relates the objectives attained by \eqref{eq:primalproblem} and \eqref{eq:dualproblem} as the chain of inequalities~\cite{boyd2004convex}
\begin{equation}
   \dHatRlx \leq \dOptRlx \leq \fOptRlx \leq \fHatRlx, \label{eq:chain-ineq-relaxation}
\end{equation}
where we employ $\vHatLambda, (\text{resp. } \vHatX)$ to denote dual (resp. primal) feasible points, \textit{i.e.}, points which fulfill the constraints required by the dual \eqref{eq:dualproblem} (resp. primal \eqref{eq:primalproblem}) problem. The first and third inequalities hold by definition of optimality, while the second stands for the \emph{weak duality} principle.
Further, 
for any essential matrix $\hat{\Ess} \in \SetEssentialMatrices$ the following chain of inequalities always holds:
\begin{equation}
   \dHatRlx \leq \dOptRlx \leq \fOptRlx \leq \fOpt \leq \fHatE, \label{eq:chain-ineq-original}
\end{equation}
where the first two inequalities come from \eqref{eq:chain-ineq-relaxation}, the relation $\fOptRlx \leq \fOpt$ stands due to the fact that \eqref{eq:primalproblem} is a relaxation of \eqref{eq:originalproblem} and the inequality $\fOpt \leq \fHatE$ holds by definition of optimality.

Therefore, the dual problem allows us to verify if a given primal feasible point $\vHatX$ is indeed optimal. Although the dual problem \eqref{eq:dualproblem} is a SDP and can be solved by off-the-shelf solvers (\eg, SeDuMi~\cite{sturm1999using} or SDPT3~\cite{toh1999sdpt3}) in polynomial time, inspired by \cite{briales2016fast,carlone2015lagrangian} we propose here a faster optimality verification based on a closed-form expression for dual candidates, thus avoiding the resolution of the SDP from scratch.

\subsection{Closed-form Expression for Dual Feasible Points} \label{sec:closed-form-LICQ}
Following, we show how to compute dual feasible points (candidates) in closed-form.
Assuming strong duality holds, we know from duality theory~\cite{boyd2004convex} that a primal optimal point $\vStarX$ is a minimizer of the dual problem \eqref{eq:dualproblem}, that is, a minimizer of the \textit{Lagrangian} 
evaluated at the dual optimal point $\vStarLambda$. Since $\boldsymbol{M}(\bm{\lambda})$ is positive semidefinite for any feasible dual point $\bm{\lambda}$, by definition $\boldsymbol{x}^T\boldsymbol{M}(\bm{\lambda})\boldsymbol{x} \geq 0$ for any 12D vector $\boldsymbol{x}$ and its minimum value is achieved at 0. Next, we can re-formulate this requirement as:

\begin{equation}
    \boldsymbol{x}^{\star T} \boldsymbol{M}(\vStarLambda) \vStarX = 0 \Leftrightarrow \boldsymbol{M}(\vStarLambda) \vStarX = \boldsymbol{0}_{12 \times 1}, \label{eq:hessianlagrangian}
\end{equation}
that is, $\vStarX$ lies in the nullspace of $\boldsymbol{M}(\vStarLambda)$.  This is known as the \textit{complementary slackness condition}~\cite{boyd2004convex}, which provides us with a set of linear constraints relating the optimal values for the primal and dual variables (always under the assumption of strong duality).

With this in mind and recalling the structure of $\boldsymbol{M}(\bm{\lambda})$ in  \eqref{th:dualproblem}, we re-write \eqref{eq:hessianlagrangian} as 
\begin{align}
     \boldsymbol{0}_{12 \times 1} &= \boldsymbol{M}(\vStarLambda) \vStarX   =  (\boldsymbol{Q} - \sum_{i=1} ^6 \lambda_i^{\star} \boldsymbol{A}_i)  \vStarX  \Leftrightarrow \\ 
    \Leftrightarrow \boldsymbol{Q} \vStarX &=  \sum_{i=1} ^6 \lambda_i^{\star} \boldsymbol{A}_i  \vStarX. \label{eq:lagrangian}
\end{align}

We stack the 12D vectors $\{\boldsymbol{A}_i \vStarX \}_{i=1}^6$ column-wise in the $12 \times 6$ matrix 
$\boldsymbol{J}(\vStarX)$ 
and the Lagrange multipliers as the 6D vector $\vStarLambda$, obtaining the following linear system \wrt $\vStarLambda$:
\begin{equation}
    \underbracket{\boldsymbol{J} (\boldsymbol{x^{\star}}) }_{12 \times 6} \underbracket{\vStarLambda}_{6 \times 1} = \underbracket{\boldsymbol{Q} \vStarX }_{12 \times 1}. \label{eq:systeminlambda}
\end{equation}

\eqref{eq:systeminlambda} enables us to compute a candidate dual solution $\vHatLambda$ given a potential primal solution $\vHatX$.
With this, we can certify if the given (feasible) solution $\vHatX$ is indeed optimal through the following Theorem:

\begin{theorem}[Verification of Optimality]
    \label{theorem:verification}
    Given a putative primal solution $\vHatX$ for Problem \eqref{eq:primalproblem}, if there exists a \emph{unique} solution $\vHatLambda$ to the linear system
    \begin{equation}
        \boldsymbol{J} (\vHatX) \vHatLambda = \boldsymbol{Q} \vHatX ,
        \label{eq:systeminlambdahat}
    \end{equation}
    and $\boldsymbol{M}\bm{(\hat{\lambda})} \succeq 0$, then we have strong duality and the putative solution $\vHatX$ is indeed optimal. 
\end{theorem}
\begin{proof}
Assume $\vHatLambda$ is a solution of \eqref{eq:systeminlambdahat} and dual feasible. Then, by \eqref{eq:hessianlagrangian}: $\boldsymbol{M}(\vHatLambda) \vHatX = \boldsymbol{0}_{12 \times 1} \Leftrightarrow \vHatX^T \boldsymbol{M}(\vHatLambda) \vHatX = 0$. Given the definition of $\boldsymbol{M}(\vHatLambda)$ in \eqref{th:dualproblem} and the quadratic constraints in \eqref{eq:primalproblem},
\begin{align}
     \vHatX^T \boldsymbol{Q} \vHatX & = \vHatX^T \sum_{i=1} ^6 \Big(\hat{\lambda}_i \boldsymbol{A}_i \Big) \vHatX = \nonumber \\
     &= \hat{\lambda}_1 \Big ( \vHatX^T   \boldsymbol{A}_1 \vHatX \Big) = \hat{\lambda}_1 \Leftrightarrow
   \fHatRlx =\dHatRlx,
\end{align}
which implies that the chain on inequalities given in \eqref{eq:chain-ineq-relaxation} is tight: $d(\vHatLambda) = \dOptRlx = \fOptRlx = \fHatRlx$, that is, the primal candidate $\vHatX$ achieves the optimal objective $\fHatRlx = \fOptRlx$, therefore it is the optimal solution and we have strong duality.
\end{proof}

From Theorem \eqref{theorem:verification}, the next statement follows:
\begin{corollary} \label{cor:verification-essential}
Given a potentially optimal solution $\hatE$ for problem \eqref{eq:originalproblem} and its equivalent form $\vHatX = [\vec{\hatE}^T, \hat{\trans}^T]^T$ where $\hat{\trans}$ is the associated translation vector, if there exists a \emph{unique} solution $\vHatLambda$ to the linear system in \eqref{eq:systeminlambdahat} and it is dual feasible (\ie $\boldsymbol{M}\bm{(\hat{\lambda})} \succeq 0$), then we can state that: (1) strong duality holds between problems \eqref{eq:primalproblem} and \eqref{eq:dualproblem}; (2) the relaxation carried out in \eqref{eq:primalproblem} is \emph{tight}; and (3) the potentially optimal solution  $\hatE $ is optimal for both \eqref{eq:primalproblem} and \eqref{eq:originalproblem}.
\end{corollary}
\begin{proof}
Since $\hatE$ is feasible for \eqref{eq:originalproblem}, it is also a primal feasible point for \eqref{eq:primalproblem} since the feasible set of \eqref{eq:primalproblem} is a relaxation of the set in \eqref{eq:originalproblem}. Therefore, we can apply Theorem \eqref{theorem:verification} to the feasible point $\vHatX$ considering it as a potentially optimal solution for \eqref{eq:primalproblem}. If there exists a unique dual feasible point $\vHatLambda$, by Theorem \eqref{theorem:verification} the chain of inequalities in \eqref{eq:chain-ineq-relaxation} is tight, which implies that strong duality holds (statement (1) of the corollary). Further, since the objective functions of \eqref{eq:primalproblem} and \eqref{eq:originalproblem} are equivalent $\forall \Ess \in \SetEssentialMatrices$, the attained cost values agree $\fHatE = \fHatRlx$ and the chain of inequalities in \eqref{eq:chain-ineq-original} becomes also tight: $   \dHatRlx = \dOptRlx = \fOptRlx = \fOpt = \fHatRlx = \fHatE$, which implies that the relaxation is tight $\fOptRlx = \fOpt$ (proving the statement (2)) and that the same solution is optimal for both problems $\fOptRlx = \fOpt = \fHatE = \fHatRlx$ (statement (3) in the corollary).
\end{proof}

Before we continue, we want to point out that Theorem \eqref{theorem:verification} and Corollary \eqref{cor:verification-essential} can only either certify the given primal solution is indeed optimal, or it is inconclusive about its optimality. In the latter case it might be that the solution is suboptimal or that the chosen dual problem is not tight.

That being said, notice that Theorem \eqref{theorem:verification} requires the existence of a \emph{unique} solution to the system \eqref{eq:systeminlambdahat}, that is, it requires the 
existence and uniqueness of the Lagrange multipliers. However, this is not necessary to obtain strong duality, and there exist other conditions under which the problem has it, while still being the first-order optimality conditions (KKT) satisfied by the pair primal/dual optimal points~\cite[Sec.~5.5]{boyd2004convex}. Nevertheless, this uniqueness is the cornerstone of our optimality certifier. Luckily, both the 
existence and the uniqueness of the Lagrange multipliers are assured by the \textit{regularity condition} or \textit{constraint qualification} (CQ) known as \textit{Linear Independence Constraint Qualification} (LICQ) \cite[Sec.~12.2]{nocedal2006numerical}, \cite{wachsmuth2013licq}. 
The following theorem adapts the LICQ to our primal non-convex problem \eqref{eq:primalproblem}.

\begin{theorem}[LICQ for the primal problem \eqref{eq:primalproblem}]
    We say that the Linear Independence Constraint Qualification (LICQ) holds at a primal feasible point $\vHatX \in \Reals{12}$ for \eqref{eq:primalproblem} with the set of 6 (differentiable) equality constraints $\{ c_i - \vHatX^T \boldsymbol{A}_i \vHatX = 0 \}_{i=1}^6$ if
    \begin{equation}
        rank\Big(\nabla(1 -\vHatX^T \boldsymbol{A}_1 \vHatX), \dots, \nabla(-\vHatX^T \boldsymbol{A}_6 \vHatX) \Big) = 6,
    \end{equation}
    where $\nabla(f(\boldsymbol{x}))$ denotes the gradient of the function $f(\boldsymbol{x})$ \wrt $\boldsymbol{x}$.
\end{theorem}

Therefore, LICQ assures that the Lagrange multipliers are unique if and only if the gradients of the \textit{active set of constraints} (all the equality constraints in \eqref{eq:primalproblem} for our problem) are linearly independent or equivalently, the Jacobian $-2\boldsymbol{J}(\vHatX)$ of these constraints evaluated at the feasible point $\vHatX$,
\begin{align}
    \Reals{12 \times 6} &\ni [\nabla(1 -\vHatX^T \boldsymbol{A}_1 \vHatX), \dots, \nabla(-\vHatX^T \boldsymbol{A}_6 \vHatX)] = \nonumber \\
    &= -2[ \boldsymbol{A}_1 \vHatX, \dots, \boldsymbol{A}_6 \vHatX] = -2 \boldsymbol{J}(\vHatX), \label{eq:jacobian-matrix-constraints}
\end{align}
is full (column) rank.

As introduced before, the analysis of the dual problem \eqref{eq:dualproblem} and concretely, the linear system in \eqref{eq:systeminlambdahat}, 
allowed us to detect the constraint in the original set employed in~\cite{Zhao2019} that blocked the development of our \emph{fast optimality certifier}.
While we include the full analysis in the \suppl, %
we briefly sketch the main conclusions here. 
The set in~\cite{Zhao2019} generates a $12 \times 7$ 
Jacobian matrix with rank 6 for all feasible primal points, 
yielding to a pencil of solutions to the linear system in \eqref{eq:systeminlambdahat}; 
that is, the solution is not unique and neither LICQ nor Theorem \eqref{theorem:verification} hold. 
This rank deficiency is corrected by eliminating 
\emph{one} of the constraints associated with the expression $\Ess \Ess^T$ in \eqref{eq:Me:EEt}, 
leading to the Jacobian matrix in \eqref{eq:jacobian-matrix-constraints} 
which is a scaled version of the coefficient matrix $\boldsymbol{J}(\boldsymbol{x}) \in \Reals{12 \times 6}$ in \eqref{eq:systeminlambdahat}. 
The system becomes fully determined and 
over-constrained in all scenarios (proof is given in the \suppl) 
and 
either 1 or 0 solutions exist. 
In practice, and mainly due to numerical errors, 
the exact solution  may not exist, 
\ie the vector $\matQ\vHatX$ 
does not lie on the range space of $\boldsymbol{J}(\vHatX)$. 
In these cases, one can always compute the ``closest'' solution in the least-squares sense.

Note that 
if we drop more constraints, 
we can still derive a closed-form expression 
for dual candidates 
and a fast certification algorithm. 
However, 
for each constraint that we discard  
we are creating an even looser relaxation 
of the original problem~\eqref{eq:originalproblem}.   
At some relaxation, 
the dual problem may not provide with a 
tight lower bound on the original problem, 
and hence, 
the certifier associated to that relaxation will not 
detect the optimal solution. 
Formally, 
let us consider a relaxed set of the essential matrix $\mathbb{E}_i$, 
which differs from the original one $\SetEssentialMatrices$ 
in $i$ constraints 
and, 
with the same notation, 
an even looser relaxed set $\mathbb{E}_j$, 
with $5 > j > i \geq 1$. 
Let us denote by $f_i ^{\star}, f_j^{\star}$ the 
optimal costs of 
the associated minimization problems 
(akin to \eqref{eq:relaxedproblem}). 
Since $\mathbb{E}_j$ is a superset of $\mathbb{E}_i$, 
by the same reasons our relaxed set 
$\SetRelaxedEssentialMatrices$ 
was a superset of $\SetEssentialMatrices$, 
the problem with $\mathbb{E}_j$ as feasible set 
is a relaxation of the problem with $\mathbb{E}_i$, 
and we have that $d_j^{\star} \leq f_j^{\star} \leq f_i^{\star} \leq \fOpt$ always. 
We can then apply Corollary~\eqref{cor:verification-essential} 
to the relaxation with $\mathbb{E}_j$. 
See that  
if the relaxed problem has strong duality 
for an essential matrix $\Essential$, 
then all the tighter relaxations 
(more constraints, \ie sets $\mathbb{E}_i$) 
will also have strong duality 
for this same problem
\footnote{See that the constraints that were not present in $\mathbb{E}_j$ 
can be seen as redundant constraints in $\mathbb{E}_i$ 
for this problem 
since  
the certified optimal solution $\Essential$ 
(an essential matrix) 
for the problem with $\mathbb{E}_j$ 
trivially satisfied the rest of the constraints.  
It has been shown in the literature, 
see \eg~\cite{briales2018certifiably, yang2020teaser} 
that redundant constraints \textit{tighten} the dual problem. 
Since the relaxation ($\mathbb{E}_j$) 
is already tight,
it follows that the ``redundant`` relaxation
with $\mathbb{E}_i$ 
is also tight.}; 
the contrary is not true. 
The condition is then sufficient but not necessary. 
This can be extended to the performance of 
the associated certifier. 
In practice, 
when we discard some constraints, 
we are only changing the form of the closed-form expression 
for the candidates to dual points 
(equation~\eqref{eq:systeminlambdahat}), 
and the associated Hessian of the Lagrangian.  
While the computation cost for this certifier could 
be potentially lower, 
we do not know a priori 
which relaxation performs better  
in terms of certification of essential matrices. 
A study of the performance of each relaxation 
is out of the scope of this paper. 
In summary, 
while looser relaxations can still 
be leveraged to develop similar certifiers, 
in terms of number of detected optimal solutions  
they will generally perform worse than tighter relaxations; 
the tightest relaxation 
for the parameterization 
of the essential matrix set 
employed in this work 
that still assures that a closed-form expression exists  
with unique solution
is obtained by dropping only \textit{one} constraint.  

\medskip
\emph{In practice}, one can certify the optimality of a given primal feasible point $\vHatX$ for \eqref{eq:primalproblem} and, by Corollary \eqref{cor:verification-essential}, of a given feasible solution $\hatE$ for \eqref{eq:originalproblem} by following in both cases Algorithm \eqref{alg:naiveverification}.
To universalize the Algorithm, let us denote by $\vHatX = [ \vec{\hatE}^T, \hat{\trans}^T]^T$ the feasible solution for \eqref{eq:primalproblem} or for \eqref{eq:originalproblem}.
Further, for any $\hatE \in \SetEssentialMatrices$, the attained objective value in \eqref{eq:originalproblem} $\big( \fHatE \big) $ agrees with the objective value in \eqref{eq:primalproblem} $\big( \fHatRlx \big)$ since $\SetEssentialMatrices \subset \SetRelaxedEssentialMatrices$; hence we employ $\fHatRlx$ to denote the corresponding objective value in both cases without confusion.
Recall that our certification has two possible outcomes, either \textsc{Positive} (the solution is optimal) or \textsc{unknown} (the certification is inconclusive).
From a practical point of view, we write the condition $\boldsymbol{M}(\vHatLambda) \succeq 0$ as its smallest eigenvalue $\mu_M$ being greater than a negative threshold $\tauMu$ and assure strong duality by applying a (positive) threshold $\tauGap$ to the absolute value of the dual gap $|\fHatRlx - \dHatRlx|$, which allow us to accommodate numerical errors. In practice, we fix the tolerances to $\tauMu = -0.02$ and $\tauGap = 10^{-14}$. 
If either the minimum eigenvalue is negative and/or the dual gap is greater than zero (considering the tolerance), the verification procedure is inconclusive. This could occur if the solution is not optimal or if strong duality does not hold for this particular problem instance and/or the chosen relaxation.

 \begin{algorithm}
    \caption{Verification of Optimality} \label{alg:naiveverification}
    \hspace*{\algorithmicindent} {\textbf{Input}: Compact data matrix $\boldsymbol{C}$; putative primal solution $\vHatX$}\\
    \hspace*{\algorithmicindent} {\textbf{Output}: Optimality certificate \textsc{isOpt} $\in \{\text{True}, \text{unknown}\}$} \\
    
    Compute $\fHatRlx$ from \eqref{eq:primalproblem}\;
    
        Compute $\vHatLambda$ by solving the linear system in \eqref{eq:systeminlambdahat} and set $\dHatRlx =  \hat{\lambda}_1 $\;

        Compute min. eigenvalue $\mu_{M}$ of $\boldsymbol{M}(\vHatLambda)$\;
            \eIf {$\mu_M \geq \tauMu$ \text{ and } $|\fHatRlx  -\dHatRlx| \leq \tauGap $}
           {
                \textsc{isOpt} = True\;
            }
            {  \tcp{Dual candidate is not feasible} \textsc{isOpt} = unknown\; }
       
    \end{algorithm}

\section{Proposed Fast Certifiable Pipeline} \label{sec:proposedPipeline}
Rather than solving the original problem via its convex SDP relaxation, in this work we propose to solve the relative pose problem through an iterative method that respects the intrinsic nature of the essential matrix set, but comes with no optimality guarantees, and to certify \textit{a-posteriori} the optimality of the solution leveraging our fast optimality certifier. Current iterative methods work well and usually converge to the global optima despite initialization, in addition to be faster than the methods employed in convex programming, \textit{e.g.}, Interior Point Methods (IPM).
Following, we enumerate and briefly explain the three major stages in which the proposed pipeline is separated (see also Figure~\eqref{fig:formulation-pipeline} for a graphic representation):
\begin{enumerate}
    \item \textbf{Initialization}: One starts by generating an initial guess with any standard algorithms, \textit{e.g.}, (RANSAC + ) \eightpt algorithm~\cite{hartley2003multiple}, or simply providing the trivial identity matrix or a random essential matrix. 
    \item \textbf{Refinement (Optimization on Manifold)}:
    We seek the solution to the original primal problem \eqref{eq:originalproblem} by refining the initial guess with a local iterative method that operates within the essential matrix manifold $\Me$ (always fulfilling constraints):
    \begin{align}
     \label{eq:optim-on-manifold}
     \hat{\Ess} = \argmin_{\Ess \in \Me}  \vec{\matE}^T \boldsymbol{C} \vec{\matE}.
\end{align}
    \item \textbf{Verification of optimality}: The candidate solution $\hat{\Ess}$ returned by the iterative method can be verified as globally optimal with the proposed Algorithm \eqref{alg:naiveverification} if the underlying dual problem \eqref{eq:dualproblem} is tight. 
\end{enumerate}
Note that the above-mentioned pipeline estimates the essential matrix, which encodes the relative pose. Both the rotation and translation \uts can be extracted from it by any classic computer vision algorithm~\cite{hartley2003multiple}.

\subsection{Implementation Details about the Optimization on Manifold Stage}
Riemannian optimization toolboxes, such as \manopt, decouple the optimization problem into manifold (domain), solvers and problem description, making it quite straightforward to implement an iterative solver for \eqref{eq:optim-on-manifold} as proposed above:

\smallskip
\textbf{Solver}
We choose here an iterative \emph{truncated-Newton Riemannian trust-region} (RTR) solver~\cite{absil2009optimization}.
RTR has shown before~\cite{briales2017cartan, rosen2019se} a very good trade-off between a large basin of convergence and superlinear convergence speed.

\smallskip
\textbf{Manifold}
As mentioned earlier in Section \eqref{sec:related-work}, many characterizations have been proposed in the literature for the essential matrix manifold. Here we pick the state-of-the-art proposal from Tron~\cite{tron2017space}.
This characterization and its associated operators are already provided by \manopt as \textsc{essentialfactory}.

\smallskip
\textbf{Problem}
Hence we only need to specify the cost function and its Euclidean derivatives (gradient $ \nabla \fE $ and Hessian-vector product $\nabla^2 \fE[\boldsymbol{U}]$) which only depend on $\boldsymbol{E}$:
\begin{align}
   \fE &=\vec{\matE}^T \boldsymbol{C}\vec{\matE} \\
   \quad \nabla \fE &= 2\boldsymbol{C}\vec{\matE} \\
   \quad \nabla^2 \fE[\matU] & = 2\boldsymbol{C}\vec{\matU}.
\end{align}
We provide these developments in the \suppl 
for completeness.


\begin{figure*}[h]
    \begin{subfigure}[t]{0.32\textwidth}
        \centering
        \includegraphics[width=\figureSize]{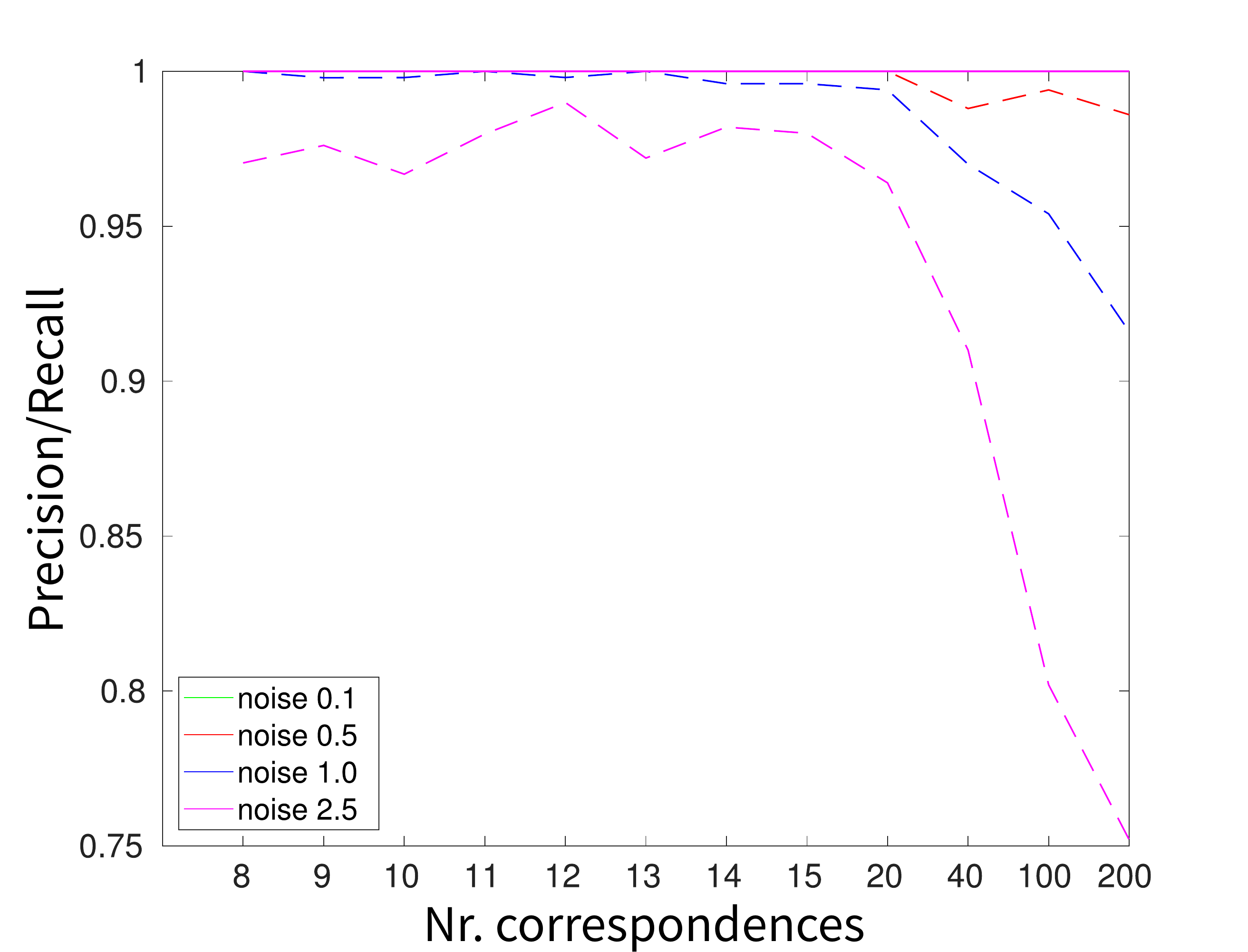}
    \caption{
    }
    \label{fig:precisionrecall}
    \end{subfigure}
    \begin{subfigure}[t]{0.32\textwidth}
        \centering
        \includegraphics[width=\figureSize]{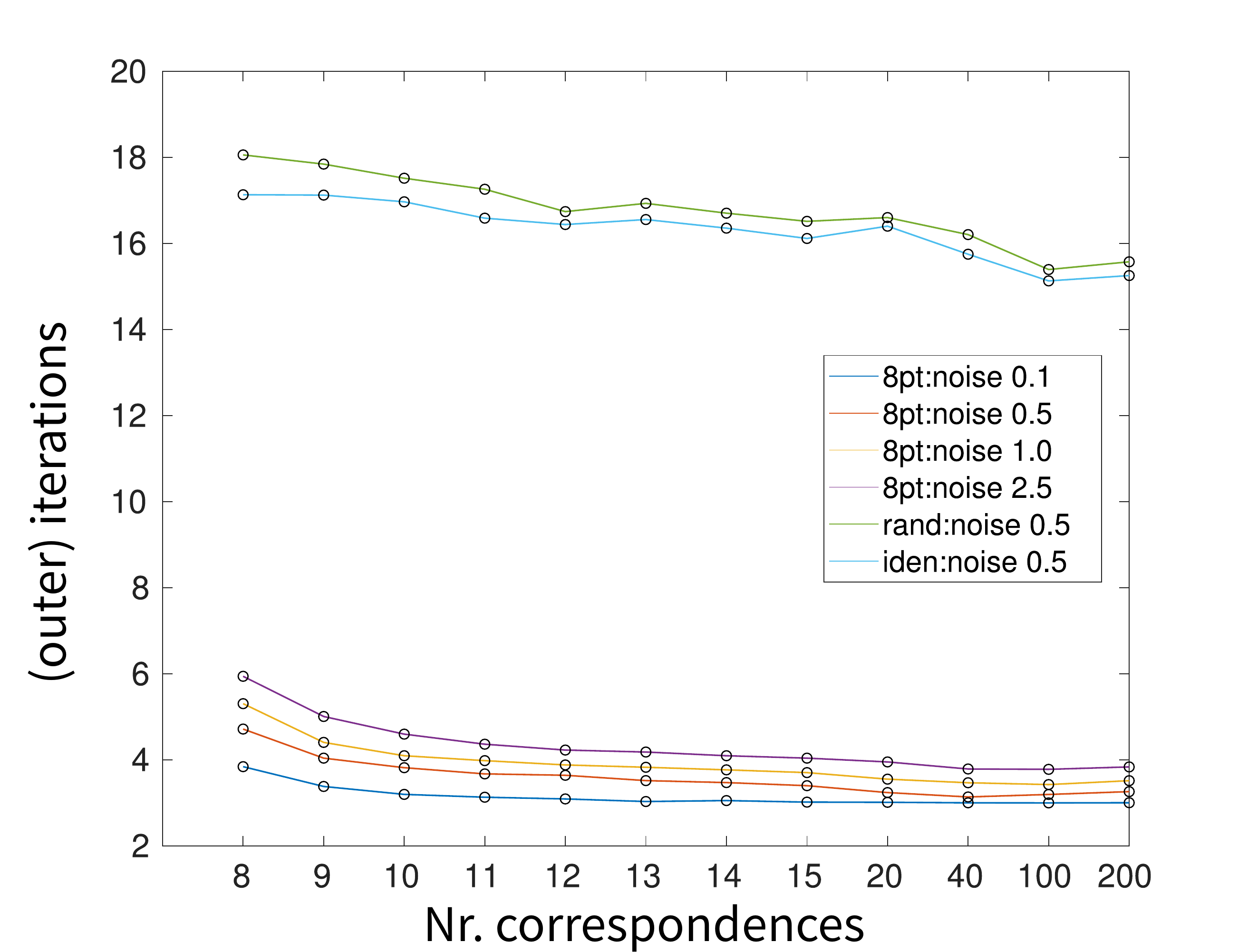}
    \caption{
    }
    \label{fig:nItersCases}
    \end{subfigure}
    \begin{subfigure}[t]{0.32\textwidth}
        \centering
        \includegraphics[width=\figureSize]{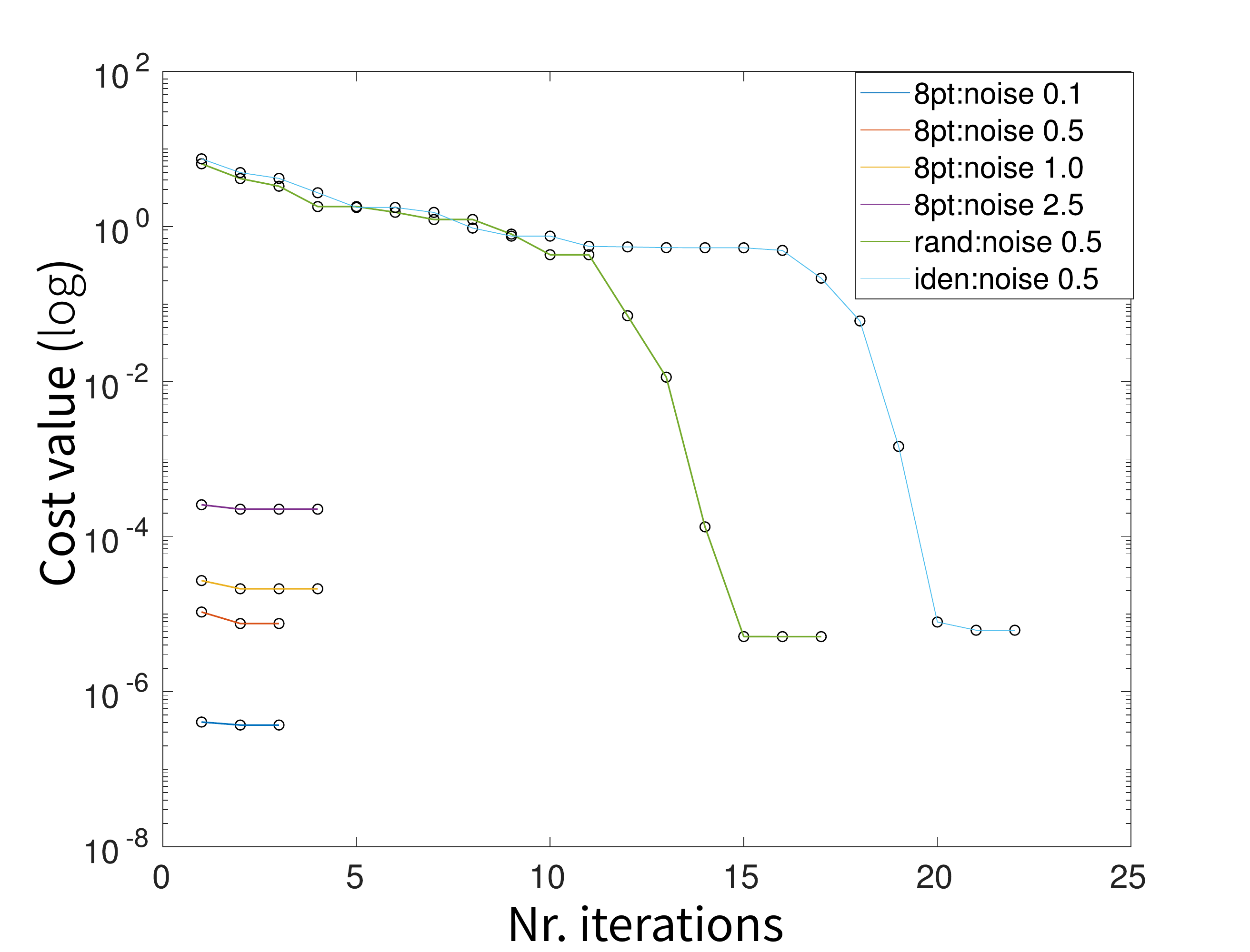}
    \caption{}
    \label{fig:evCostFcnIters}
    \end{subfigure}
    \caption{(a) Precision (solid line) and recall (dashed line) metrics for the four levels of noise considered as a function of the number of points. (b) Averaged number of (outer) iterations for different initialization (identity matrix \textsc{iden}, random matrix \textsc{rand} and \eightpt algorithm \textsc{8pt}) and levels of noise (0.1, 0.5, 1.0, 2.5 \pixels). (c) Evolution of the cost value as a function of the number of iterations; an example of each instance is shown. Note the logarithmic scale in the last figure and the non-linear scale for the \textsc{X} axis.}
\end{figure*}

\section{Experimental Validation} \label{sec:experiments}
In this Section we validate through extensive experimentation with synthetic and real data the utility of the presented verification technique (Algorithm \eqref{alg:naiveverification}) and the performance of the proposed \emph{certifiable} pipeline (Section \eqref{sec:proposedPipeline}). 

\subsection{Experimental Validation with Synthetic Data}

\label{sec:syntheticexperiments}
Similarly to \cite{briales2018certifiably}, we generate random problems as follows: We place the first camera frame at the origin 
(identity orientation and zero translation) 
and generate a set of random 3D points 
within a truncated pyramid (frustum) zone 
with depth ranging from one to eight meters 
(approx. from $75$ to $615$ focal units)  
measured from the first camera frame and 
inside its Field of View (FoV). 
Then, we generate a random pose for the second camera 
whose translation parallax is constrained within a spherical shell 
with minimum radio $||\trans||_{\text{min}}$, 
maximum radio $||\trans||_{\text{max}}$ and 
centered at the origin. 
We also enforce that all the 3D points lie within the second camera's FoV. 
This configuration is closer to those found in real scenarios. 
Next, we create the correspondences as unit bearing vectors and 
add noise by assuming a spherical camera, 
computing the tangential plane at each bearing vector and 
introducing a random error sampled from the standard uniform distribution, considering a focal length of 800 pixels for both cameras 
and images of approx. $1900$ \pixels. 
The principal point is placed at the center of the image plane 
for both cameras.

In the following, we will fix the FoV to 100 degrees and the translation parallax $||\boldsymbol{t}||_2 \in [0.5, 2.0] \;(\text{meters})$. 
We generate four sets of experiments, each of them with a different level of noise $\sigma_{\text{noise}} \in \{0.1, 0.5, 1.0, 2.5\}\;\pixels$. 
Further, for each level of noise, we consider instances of the Relative Pose problem with different number of correspondence pairs in $N \in \{8, 9, 10, 11, 12, 13, 14, 15, 20, 40, 100, 200 \}$. 
We generate 500 random problem instances for each combination of noise/number of correspondences.

\smallskip
\textbf{Effectiveness of the Verification Algorithm}. 
First, we show the effectiveness of the verification technique in Theorem \eqref{theorem:verification} 
under the setup provided by the above-mentioned four sets of experiments, 
which consider different combinations of 
noise/number of correspondences for the Relative Pose problem. 
For each instance, we compute the optimal solution $\vStarX$ by solving\footnote{Since the original code is not publicly available, 
we execute our own implementation: The SDP was modelled in 
\textsc{Matlab} with \textsc{cvx}~\cite{gb08} and 
solved with \textsc{SDPT3}~\cite{toh1999sdpt3}.} 
the SDP relaxation of \eqref{eq:originalproblem} as it was executed in \cite{Zhao2019} and 
consider the returned solution as optimal if the rank guarantees the tightness of the relaxation.
If the relaxation is not tight,
we treat the solution as suboptimal $\vHatX$ and 
project it onto the essential matrix set to obtain a primal feasible point 
(we refer the reader to the original work for further details). 
We detected in 343 out of 24000 occasions that the relaxation was not tight. 
Further, for each instance of the problem we also compute the essential matrix by the $\eightpt$ algorithm and treat it as suboptimal.
Next, we apply the verification technique in Theorem \eqref{theorem:verification} for each instance and compute the well-known classification metrics \textit{precision} and \textit{recall}. Let us denote the globally optimal $\vStarX$ classified as such by \texttt{TP} (True Positive), the suboptimal point $\vHatX$ classified as optimal by \texttt{FP} (False Positive) and the true optimal that couldn't be verified (the verification was inconclusive) by \texttt{FNP} (False Non-Positive)\footnote{
One may note that the classic classification metrics employ False Negative instead of False Non-Positive. However, our algorithm does not output positive/negative but positive/inconclusive, and hence False Non Positive seems a better choice for the cases in which our verification technique outputs \emph{inconclusive} for a known true optimal solution.}. 
Hence, the classification metrics take the form:
\begin{equation}
    \text{precision} = \frac{\mathtt{TP}}{\mathtt{TP} + \mathtt{FP}}, \quad \text{recall} = \frac{\mathtt{TP}}{\mathtt{TP} + \mathtt{FNP}}.
\end{equation}
Figure~\eqref{fig:precisionrecall} 
depicts the metrics for the four sets of experiments 
as a function of the number of points. 
As expected, the precision is always 1 while the recall decreases 
with the number of correspondences when the noise level is high 
$ [1.0, 2.5]\  \pixels$, 
but remains stable (above $95$ \%) with the typical level of noise $[0.1, 0.5]\  \pixels$. 
Note that the set of experiments with noise $0.1$ \pixels 
attains also a recall metric of one for all the tests regardless the number of correspondences considered.

\textbf{Further Experiments on Effectiveness: Focal Length}
With the default set of parameters, 
we let the image size constant 
(approx. $1900$~\pixels) and  
modify the focal length of both cameras,  
selecting values from $f \in \{300, 400, 500, 700\}$ 
\pixels; 
the $\FOV$ of the camera is changed accordingly 
to (approx.) $\FOV \in \{  145, 134, 124, 107 \} $~\degrees, 
respectively. 
For each focal length, 
we generate random problem instances 
with the same range of correspondences 
employed before. 
We then compute the precision and recall metrics 
for these experiments 
\wrt the solution obtained from the SDP solver.  
Figure~\eqref{fig:prec-rec-focal} depicts the results,  
where we can observe that varying the 
focal length in this way  
leads to similar results than changing the noise level.  
For the smallest focal length 
($f = 300~\pixels$) 
we obtain the worst 
results, 
but we are still able to certify $90\%$ of the optimal solutions. 
For all the cases, we obtain a precision of $100\%$, 
which is not shown in the graphic for clarity. 
We perform similar experiments with 
constant $\FOV$ and varying focal length, 
hence changing the image size.  
For the same set of values for the focal length,  
we obtain similar results that 
are not included here due to space limits. 
This similarity between focal length 
and noise level 
agrees on how we are introducing the noise 
in our correspondences: 
we scale the (unit norm) vectors by the focal length, 
and then add noise sampled from a uniform distribution, 
which is independent of the focal length. 
Hence, the same amount of noise 
has a greater effect when the focal length is smaller.

\begin{figure}[t]
    \centering
    \includegraphics[width=0.9\linewidth]{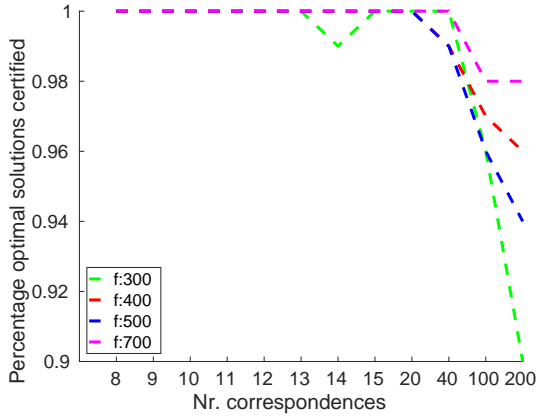}
    \caption{Recall metrics for 
    the set of experiments with 
    varying focal length 
    from $f \in \{300, 400, 500, 700\}$ 
    and default parameters. 
    The precision metrics are $100\%$ for all cases 
    and are not shown for clarity.}
    \label{fig:prec-rec-focal}
\end{figure}

\smallskip
\textbf{Performance of the Proposed Certifiable Pipeline} 

\begin{figure*}[h]
    \begin{subfigure}[t]{0.32\textwidth}
        \centering
        \includegraphics[width=\figureSize]{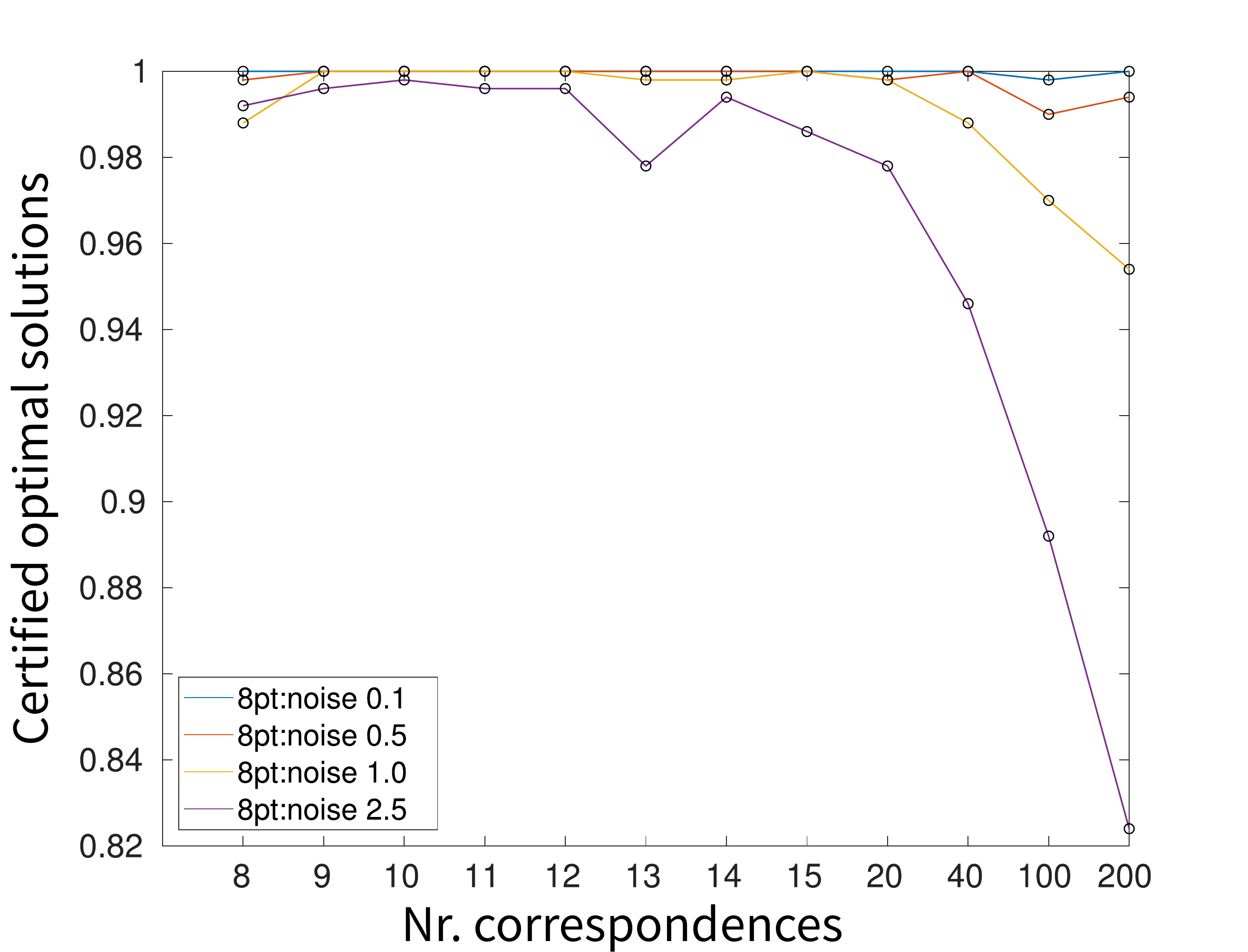}
        \caption{}
        \label{fig:cases-opt-8pts}
    \end{subfigure}
    \begin{subfigure}[t]{0.32\textwidth}
        \centering
        \includegraphics[width=\figureSize]{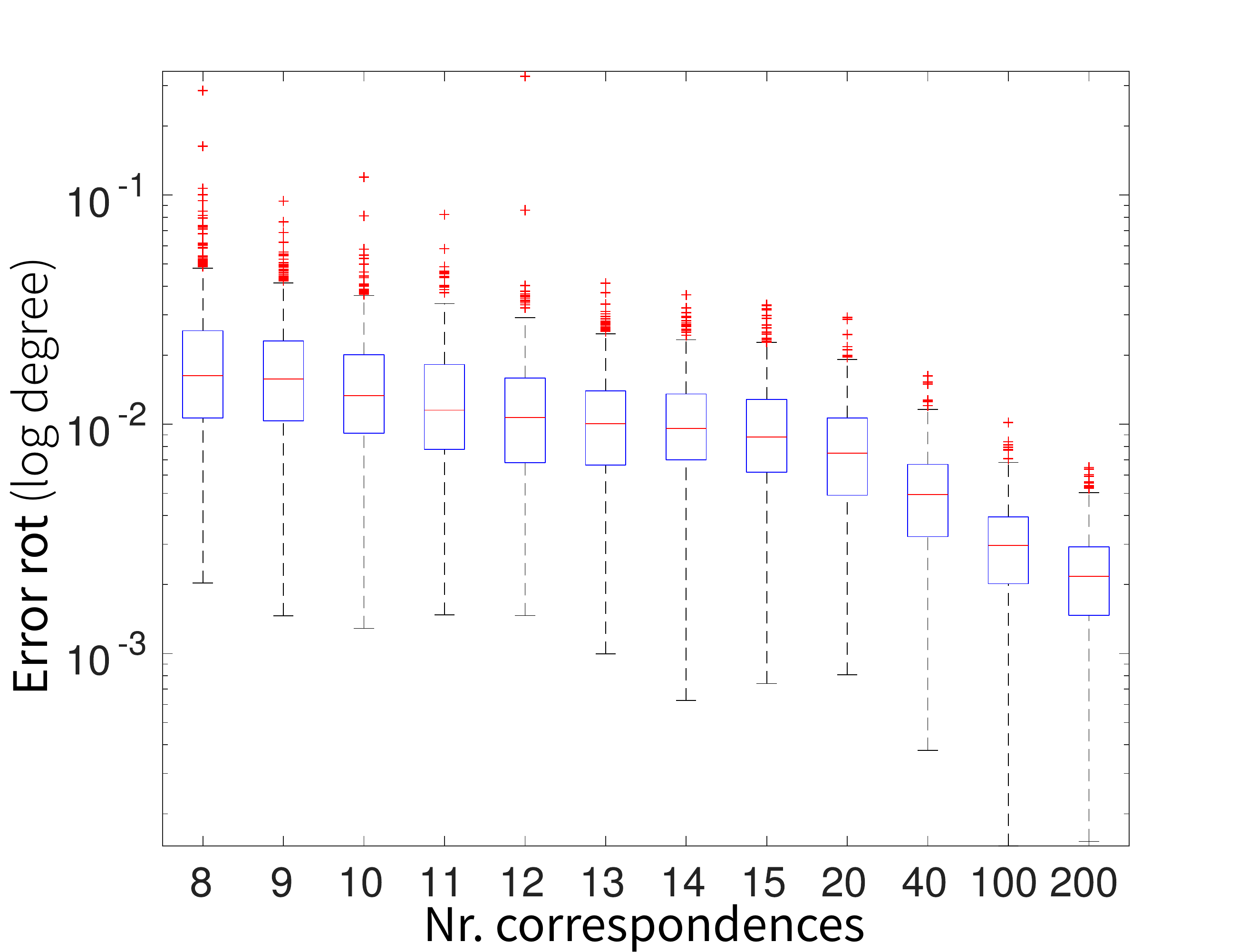}
        \caption{}
        \label{fig:error-rot-noise-01}
    \end{subfigure}
    \begin{subfigure}[t]{0.32\textwidth}
        \centering
        \includegraphics[width=\figureSize]{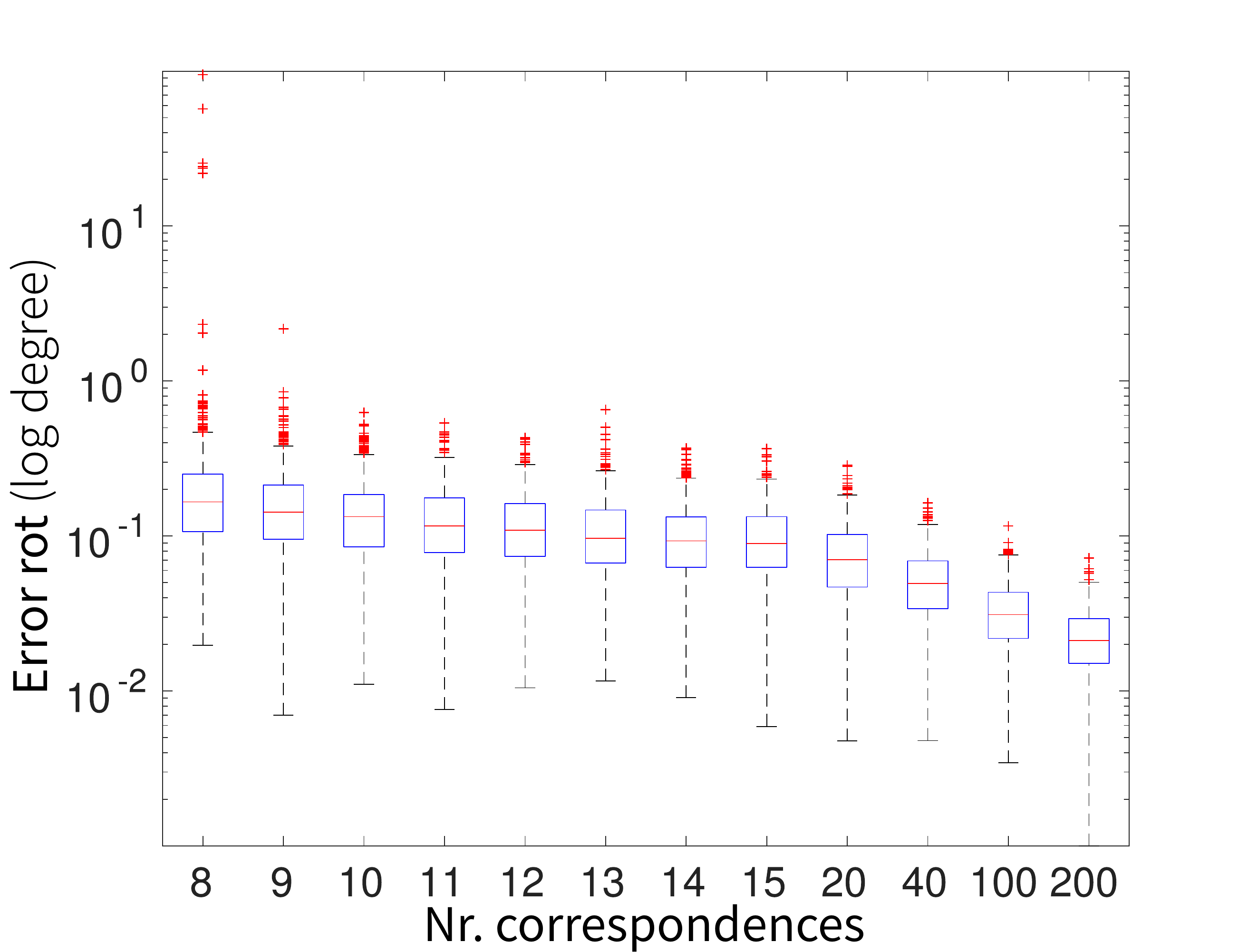}
        \caption{}
        \label{fig:error-rot-noise-10}
    \end{subfigure}
    \caption{(a) Percentage of cases in which the algorithm could certify optimality as a function of the number of correspondences and noise level. Error in rotation for the instances of the problem with noise 0.1 \pixels (b) and 1.0 \pixels (c). Please, note the logarithmic scale in the Y axis in figures (b) and (c).}
    \label{fig:syntheticnoise}
\end{figure*}

We consider the same four sets of experiments with the different combinations of noise/number of correspondences.
We feed the proposed pipeline with the initial guess estimated by the $\eightpt$ algorithm~\cite{hartley2003multiple} 
and apply the verification technique to the solution returned by the iterative method to certify its optimality \textit{a posteriori}.
Due to space limits, we only show the results for the cases with noise 0.1 and 1.0 but include the remainders in the \suppl. 
Figure \eqref{fig:cases-opt-8pts} depicts the percentage of cases in which the verification algorithm could certify optimality as a function of the number of points. One may note how the number of cases certified as optimal decreases with high levels of noise and large number of correspondences, following the tendency of the \textsc{recall} metric in Figure \eqref{fig:precisionrecall}. 
As a more intuitive measurement, Figure \eqref{fig:error-rot-noise-01} and \eqref{fig:error-rot-noise-10} plot the error in rotation (in degrees) for the set of experiments. The error is measured in terms of the geodesic distance between the estimated rotation  $\hat{\rot}$ and ground truth $\rot_{gt}$:
    \begin{equation}
        \epsilon_{\text{rot}} = \arccos \Big ( \frac{\tr(\hat{\rot}^T \rot_{\text{gt}}) - 1}{2}\Big) \frac{180}{\pi} [\degrees].
    \end{equation}

For completeness, under the same setup, we generate additional instances of the Relative Pose problem by fixing the level of noise to 0.5 \pixels and varying the FoV, parallax and number of correspondences, one at a time. While the statistics are given in the \suppl 
due to space limits, we can sketch the following conclusions:
(1) the number of (outer) iterations remains stable under five steps for all the conducted experiments when the iterative method is initialized with the \eightpt; and (2) while the number of certified optimal solutions does not seem affected by the varying FoV, the variation of the parallax produces a similar behavior to the change of noise. In \cite{ma2001optimization} it was remarked that, for a fixed level of noise, a variation on the parallax is equivalent to varying the signal-to-noise ratio, which agrees with our results.

Further, we observe in all these cases that instances of the relative pose problem whose numbers of correspondences are closer to the minimum (8-9) are more sensible to the noise level, FoV and parallax (see \eg Figure \eqref{fig:error-rot-noise-10}).

\begin{figure*}[ht]
    \begin{subfigure}[t]{0.32\textwidth}
        \centering
        \includegraphics[width=\figureSize]{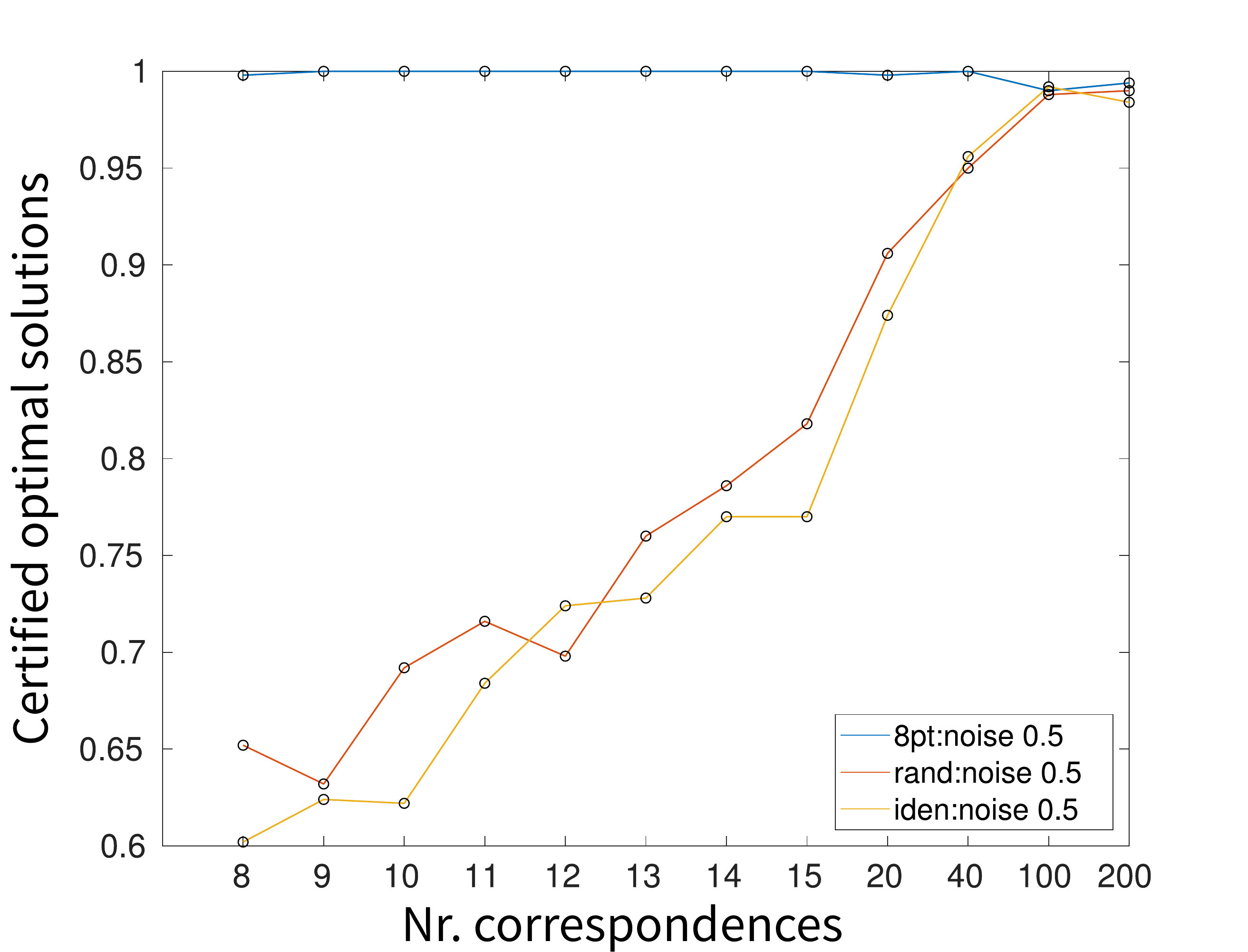}
        \caption{}
        \label{fig:cases-opt-init}
    \end{subfigure}
    \begin{subfigure}[t]{0.32\textwidth}
        \centering
        \includegraphics[width=\figureSize]{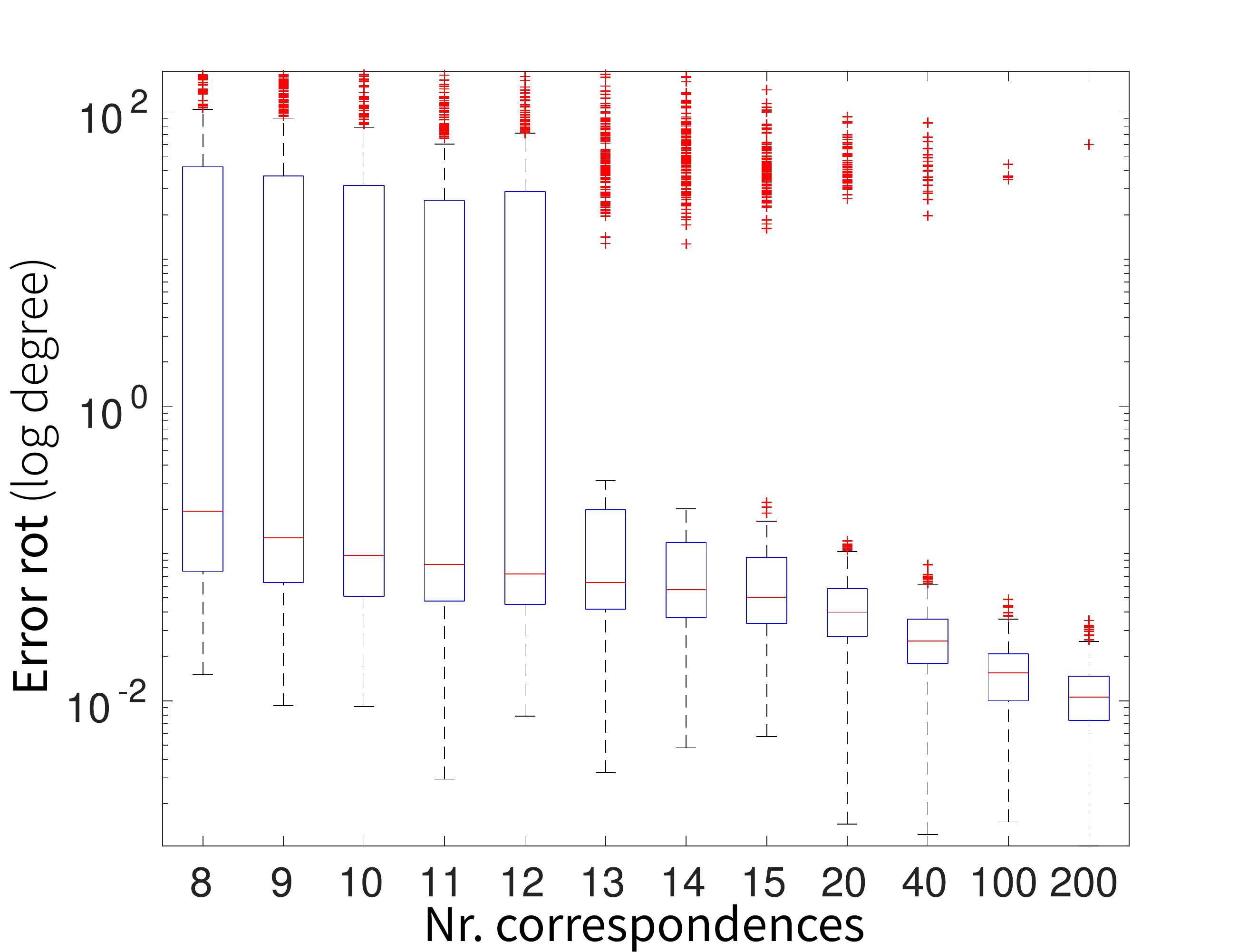}
        \caption{}
        \label{fig:error-rot-rand}
    \end{subfigure}
    \begin{subfigure}[t]{0.32\textwidth}
        \centering
        \includegraphics[width=\figureSize]{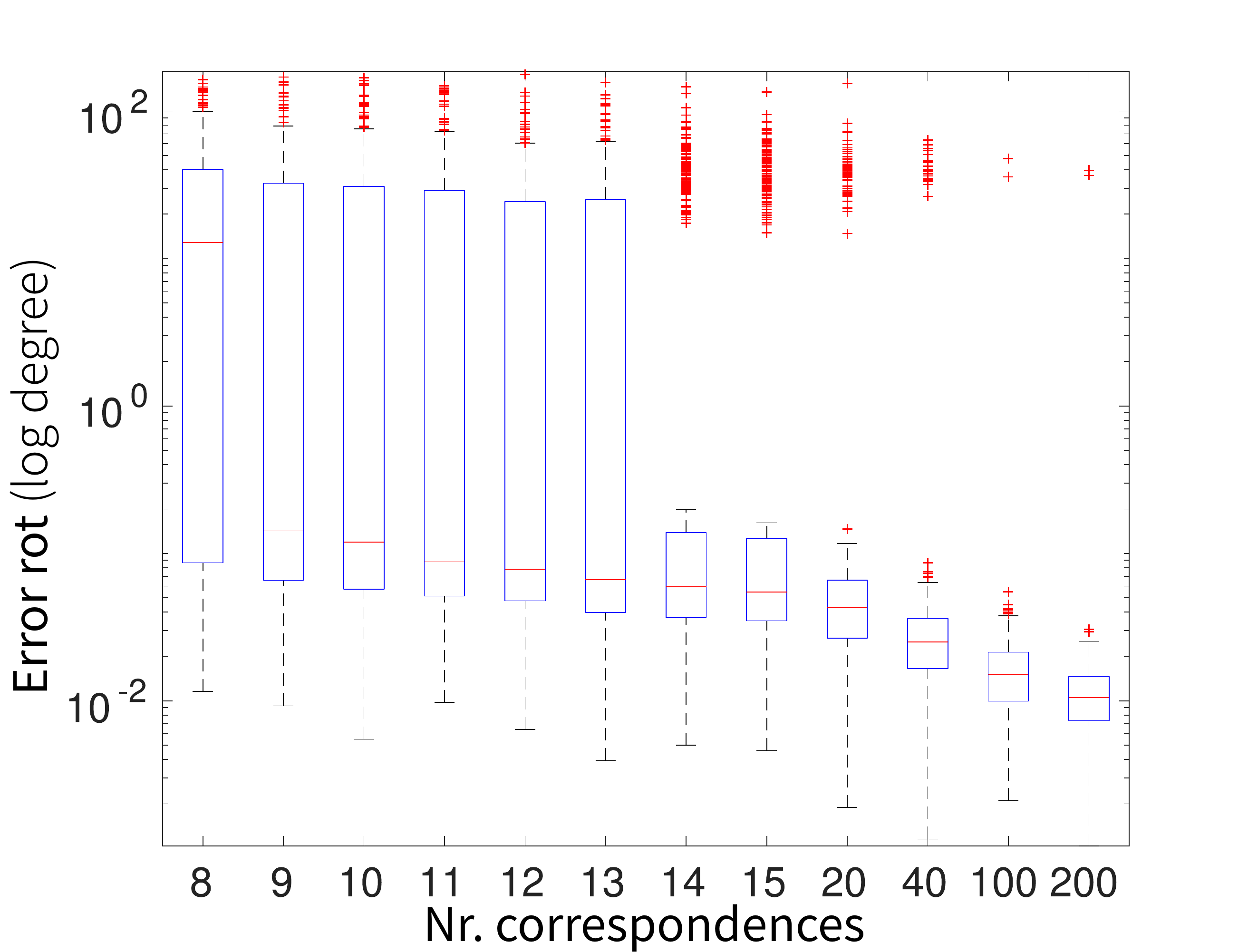}
        \caption{}
        \label{fig:error-rot-eye}
    \end{subfigure}
    \caption{(a) We plot the percentage of cases in which the algorithm certifies optimality for instances of the Relative Pose problem with fixed noise 0.5 \pixels  and different initialization. Error in rotation [degrees] for the final estimations whose refinement stages were initialized with a random guess (b) and identity matrix (c). Please, note the logarithmic scale in the Y axis in figures (b) and (c).}
    \label{fig:diffinitialguess}
\end{figure*}

\begin{figure*}[h]
\centering
    \begin{subfigure}[t]{0.49\textwidth}
        \centering
        \includegraphics[width=\figureSize]{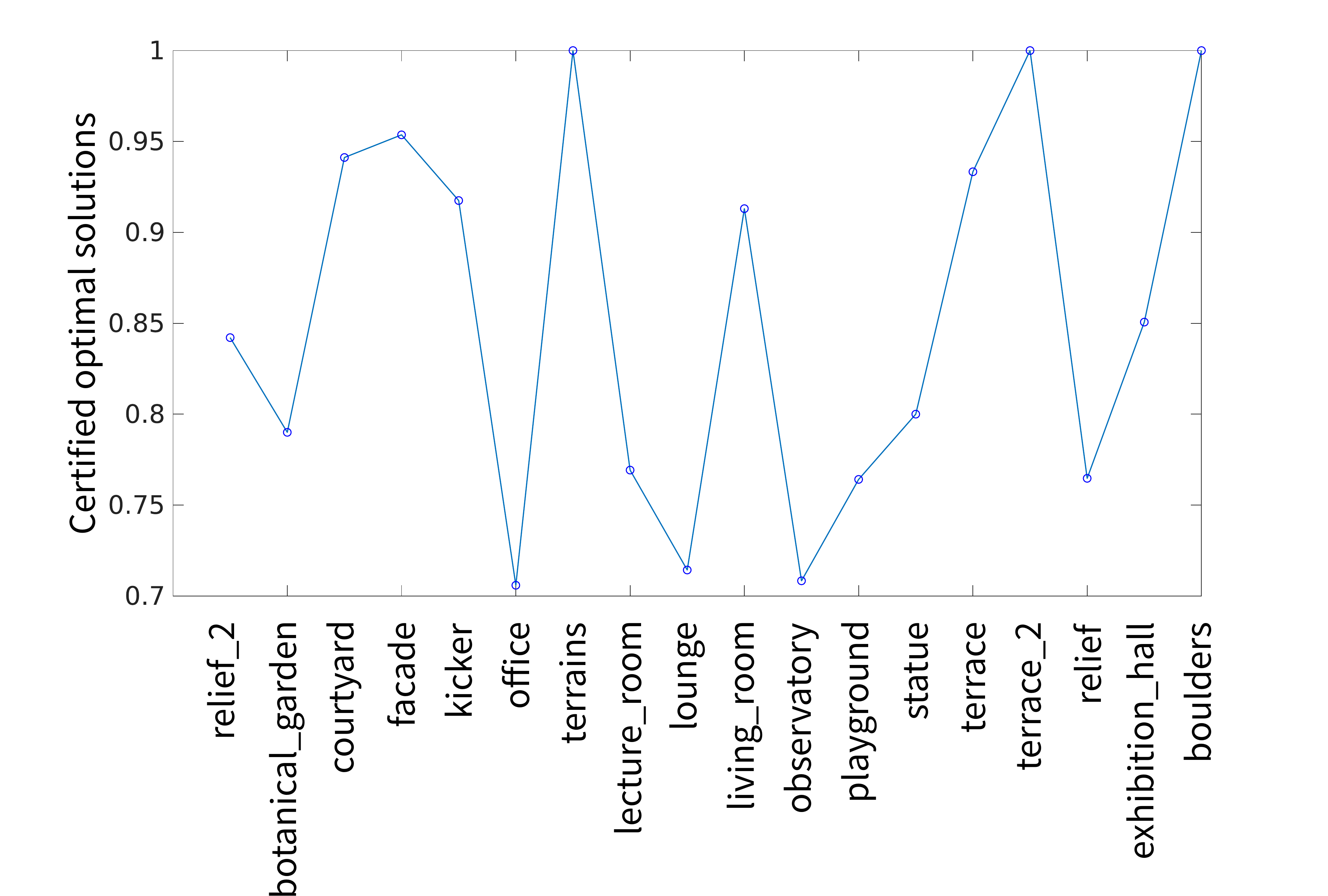}
        \caption{}
        \label{fig:real-exp-clean-opt}
    \end{subfigure}
    \begin{subfigure}[t]{0.49\textwidth}
        \centering
        \includegraphics[width=\figureSize]{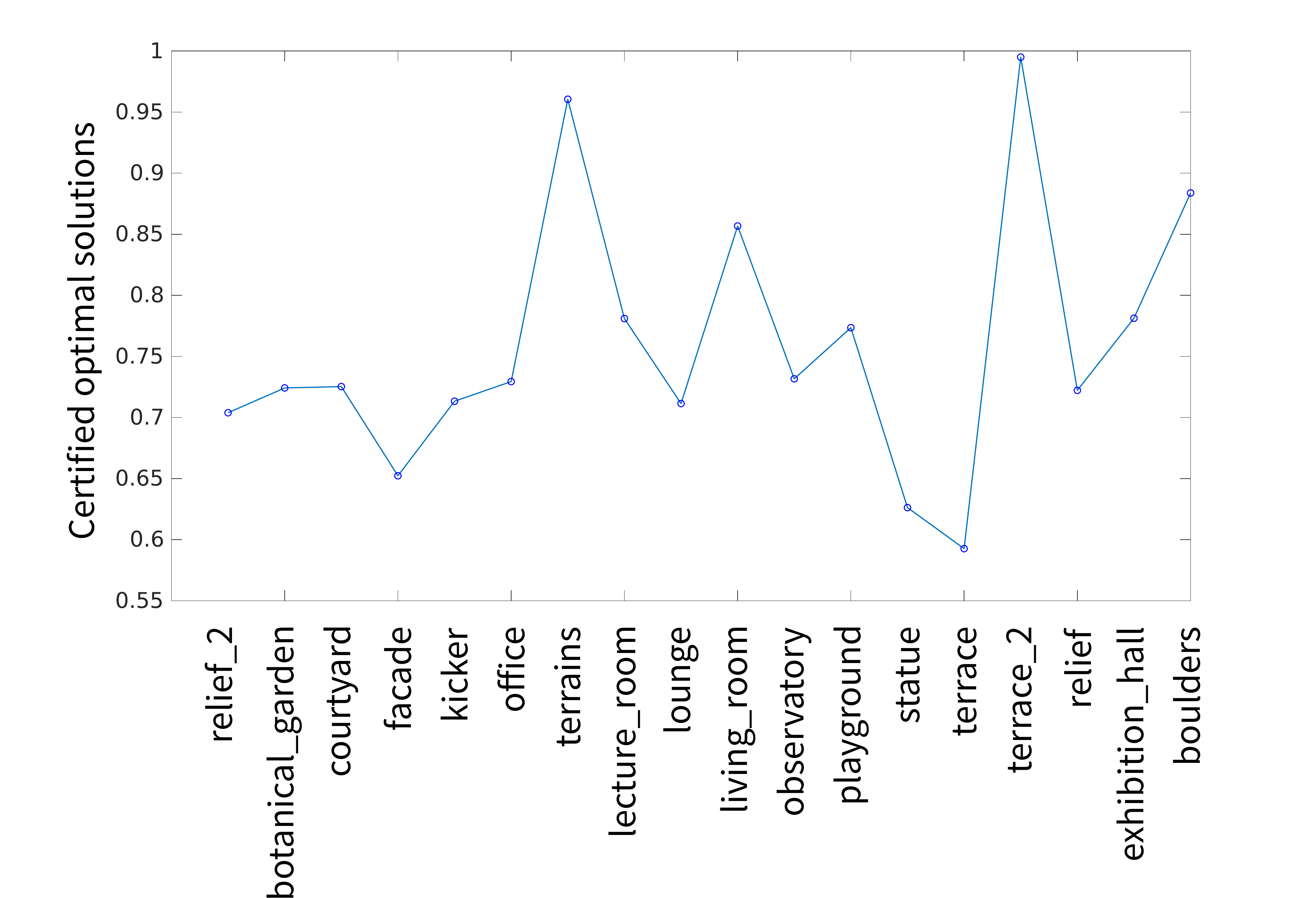}
        \caption{}
        \label{fig:real-exp-opt}
    \end{subfigure}
    \caption{(a) Averaged percentage of cases in which the algorithm certified optimality for the eighteen different sequences with pre-filtered outlier-free correspondences and (b) embedded in a RANSAC scheme.}
    \label{fig:real-exp-clean}
\end{figure*}
    
\begin{figure*}[h]
    \begin{subfigure}[t]{0.49\textwidth}
        \centering
        \hspace{-0.3cm}
        \includegraphics[width=0.9\textwidth]{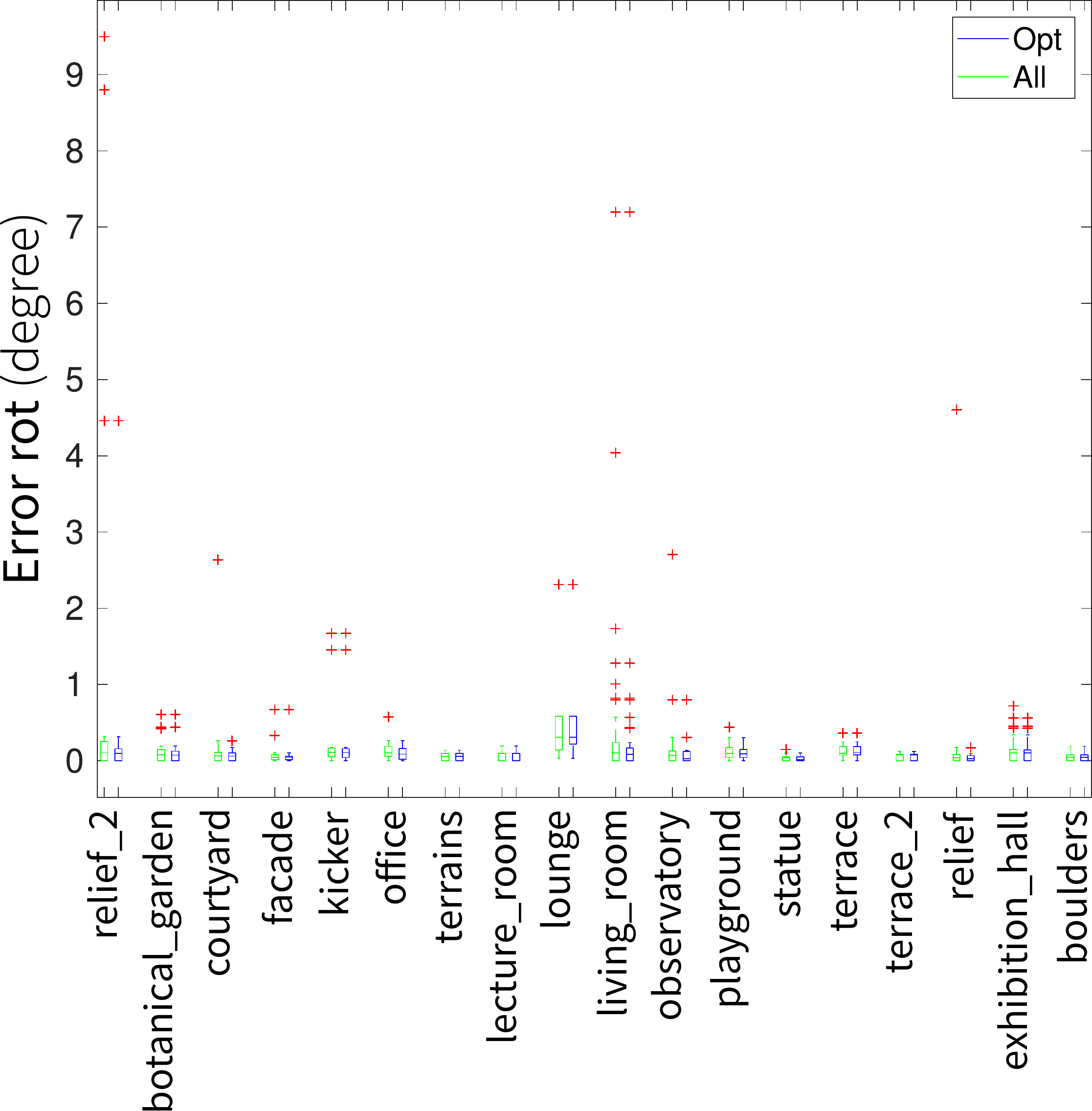}
        \caption{}
        \label{fig:real-exp-clean-error}
    \end{subfigure}
    \begin{subfigure}[t]{0.49\textwidth}
        \centering
        \hspace{-0.3cm}
        \includegraphics[width=0.91\textwidth]{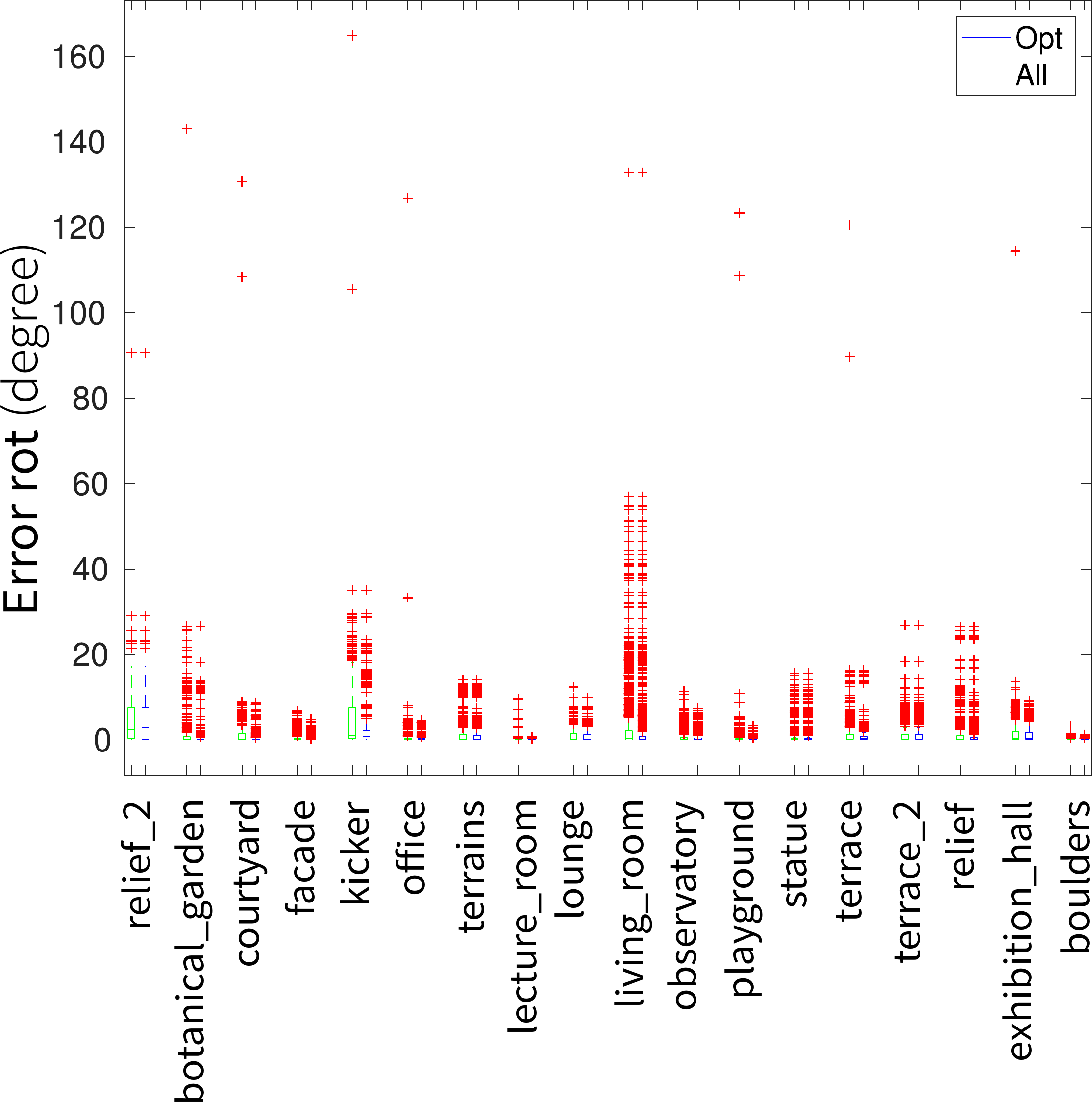}
        \caption{}
        \label{fig:real-exp-error}
    \end{subfigure}
    \caption{(a) Error in rotation (degrees) for the eighteen different sequences with pre-filtered outlier-free correspondences and directly (b) embedded in a RANSAC scheme. We plot all the results from the optimization (\textsc{All}) and the errors for those cases detected as optimal (\textsc{Opt}).}
    \label{fig:real-exp}
\end{figure*}

\smallskip
\textbf{Further Experiments}: 

A good initial guess is crucial in many problems which rely on iterative solvers in order to avoid local minima (for example, Pose Graph Optimization \cite{carlone2015duality}). To analyze the sensibility of the proposed pipeline to initialization quality, we generate 500 random instances of the Relative Pose problem following the above-mentioned procedure, with fixed parameters $\text{FoV} = 100\ (\degrees)$, $||\trans||_2 \in [0.5, 2.0] (\text{m})$, $\sigma = 0.5 \pixels$ and varying the number of correspondences, as it was previously done. In this case, however, we feed the iterative method with three different initial guesses: the trivial identity matrix, a random essential matrix and the resulting estimate from the \eightpt algorithm.  
Then, we apply our verification technique to each instance. 

We plot the percentage of cases in which the verification technique certified the solution as optimal as a function of the number for each initialization considered in Figure \eqref{fig:cases-opt-init} and the error in rotation (degrees) in \eqref{fig:error-rot-eye}, \eqref{fig:error-rot-rand} for the instances initialized with the identity matrix and a random guess, respectively. Figure \eqref{fig:nItersCases} shows the number of (outer) iterations required by the RTR solver to converge.
For the cases with the estimate from \eightpt algorithm as initial guess, 
the iterative method required less than five iterations to converge, 
while both the random and identity cases increase the iterations up to sixteen. 
We want to point out the (almost linear) decreasing tendency in the cases detected as suboptimal and
the error in rotation with the number of correspondences 
for the set of experiments with identity and random initial guesses. 
Further, with a large number of correspondences, 
the iterative method tends to return the optimal solution and 
the initial guess only affects the convergence rate.

\medskip 

\textbf{Analysis of Computational Cost}: 
The greatest advantage of our proposal against 
the state-of-the-art solvers
~\cite{ briales2018certifiably, Zhao2019} 
is the low computational cost, 
both for the actual solver 
(initialization and optimization on manifold) 
and for the certificate of optimality 
(Algorithm~\eqref{alg:naiveverification}). 
Note however that we implement our proposal in matlab; 
a faster, more efficient implementation 
is considered as future work. 
Nevertheless we can still draw some intuition 
for the computational cost of our proposal 
and its potential. 
The following times are the average values 
for all our experiments, 
which were performance 
with a standard PC: CPU i7-4702MQ,
2.2GHz and 8 GB RAM. 
The candidate computation 
(solving Equation~\eqref{eq:systeminlambdahat}) 
takes $0.23723$ milliseconds, 
while defining and 
computing the minimum eigenvalue of the Hessian 
goes to $0.13332$ msecs. 
In total, considering 
initialization 
(here we employ the \eightpt as standard initialization), 
refinement on manifold 
($3-4$ iterations till convergence, 
see Figure~\eqref{fig:evCostFcnIters})  
and certification, 
together with some other operations, 
\eg projection on manifold or definition of problem, 
our certifiable pipeline takes only $41$ milliseconds. 
Since the faster certifiable solver
\footnote{Date of this document} 
is the one proposed by Zhao in~\cite{Zhao2019}, 
we only measure the computational cost 
of this proposal. 
The interested reader is referred to 
that work for a comparison 
against other state-of-the-art solvers. 
To perform a fair comparison, 
this SDP solver was implemented in \textsc{Matlab}, 
modeled with \textsc{cvx} and the IPM was solved with \textsc{SDPT3}. 
As expected, the computational cost is larger 
and takes around $150$ milliseconds, 
without including modeling 
or other operations, 
such as data matrix creation, 
low rank decomposition 
or relative pose recovery. 


\subsection{Experimental Validation with Real Data}
We conclude this Section with the evaluation of the proposed certifiable pipeline on real data. We sample pairs of images 
from 18 different (multi-view) sequences in the ETH3D dataset~\cite{schops2017multi}, which covers both indoor and outdoor scenes and provides with ground-truth camera poses and intrinsic camera calibration. 

To generate the correspondences, we proceed as follows. First, we extract and match 100 SURF~\cite{bay2006surf} features per pair of images. Next, we obtain the corresponding bearing vectors by using the pin-hole camera model~\cite{hartley2003multiple} with the intrinsic parameters provided for each frame. From here, we conduct two types of experiments with all the image pairs. These sets of experiments are aimed to mainly reflect: (first type) the performance of the proposed pipeline under real noise; and (second type) the importance of our novel certification technique when dealing explicitly with outliers. Further, for each image pair we execute 100 times each type of experiment to provide statistically meaningful results.

\smallskip
\textbf{Experiments on real data with pre-filtered outliers}: 
We filter outliers by removing all the correspondences whose associated algebraic error \wrt the ground truth essential matrix is greater than a given fixed threshold $\epsilon_{\text{error}}$, \ie we consider as inliers all the correspondences ($\obsip, \obsi$) such that $(\obsi^T\Ess_{\text{gt}}\obsip) ^2 < \epsilon_{\text{error}}$. We run the proposed certifiable algorithm with this set of outlier-free correspondences. Figure \eqref{fig:real-exp-clean-opt} shows the averaged percentage of cases in which we could certify an optimal solution and, as in the previous section, we also plot the error in rotation measured in degrees \wrt the provided ground-truth in Figure \eqref{fig:real-exp-clean-error}. We depict the rotation error for all the returned solutions regardless their optimality (\textsc{All}) and only the values for those cases detected as optimal (\textsc{Opt}). We certify more than 70\% of image pairs as optimal in all the sequences, which is reflected as low rotation errors in Figure \eqref{fig:real-exp-clean-error}. We want to point out that those cases with optimal solutions tend to attain lower errors \wrt the ground truth rotation.

\smallskip
\textbf{Experiments on real data embedded in a RANSAC scheme}: 

In this second type of experiments, we filter outliers by embedding the initialization stage into a RANSAC~\cite{fischler1981random} scheme in order to both detect the set of inliers and obtain the initial guess. Note that in this case, both the initialization and the set of employed points during the subsequent Riemannian optimization depend on the RANSAC solution, which may not be accurate enough. Bad RANSAC solutions may lead to an optimization problem with corrupted data (outliers remain in the inlier set) and/or missing inliers~\cite{botterill2011refining}.
However, even in these cases, 
one can still find the optimal solution to the problem given the provided (corrupted) data. Nevertheless, in these cases it is expected that the optimal solutions may attain large rotation errors \wrt the ground truth. 
We want to remark that one can employ other paradigms to select the set of inliers; here we choose RANSAC as example for being widely employed in the literature~\cite{fischler1981random}.

Figure \eqref{fig:real-exp-opt} shows the averaged number of cases in which our algorithm could certify optimality given the same set of image pairs and points that in the previous type of experiments, while Figure \eqref{fig:real-exp-error} shows the error in rotation. We again depict the error for all the cases after the optimization \textsc{All} and the error for only those cases with certified optimal solution \textsc{Opt}. 
As expected, the rate of certified optimal solutions decreases, which is reflected as larger rotation errors (a similar behavior is depicted in~\cite{botterill2011refining}). Note, however, how we obtain lower errors when only considering the optimal solutions; we also note that some instances with optimal solutions still attain large errors, behavior which is expected if the estimate set of inliers by RANSAC contain outliers and/or present missing inliers, as it was outlined above.


\section{Conclusions and Future Work}
In this work we have proposed a formulation of the non-minimal Relative Pose problem whose dual problem admits a closed-form solution given the primal solution. This allows us to build a very fast \textit{a posteriori} certifier for candidate primal solutions.
We have provided the first certifiable pipeline for the Essential matrix estimation which, given a set of $N$ correspondence pairs: (1) generates the initial guess (\eg with the \eightpt algorithm or simply with the identity matrix); (2) refines the initialization with a local, iterative method operating directly on the Essential Matrix manifold; and last, (3) certifies the optimality of the returned solution with our novel certification procedure, which employs the proposed closed-form expression for dual candidates. Extensive experiments on both synthetic and real data under a wide variety of conditions support our claims.
The preliminary results obtained with our Matlab implementation show highly promising for hardening into a much faster C++ implementation, that we intend to explore and benchmark in future work. 
Further, we contemplate the exploration of fast dual solvers leveraging tighter relaxations in order to tackle the failure cases the current certifier may present. 

\section*{Declaration of Competing Interest}
The authors declare no conflict of interest. The funding entities had no role in the design of
the study; in the collection, analyses, or interpretation of data; in the writing of the manuscript, or in the decision
to publish the results.

\section*{Acknowledgements}
This work was supported by the research project WISER (DPI2017-84827-R), 
as well as by the Spanish grant program FPU18/01526. 
The publication of this paper has been funded by the University of Malaga.


\newpage
\clearpage
\appendices
\onecolumn
\setcounter{MaxMatrixCols}{20}

\section{Equivalence between the objectives in Problem \eqref{eq:originalproblem}} \label{app:equivalence-objectives-QCQP}
Following we prove the equality in the objective function for Problem \eqref{eq:originalproblem}.

Let us express the original cost function as,
\begin{equation}
    \fE = \sum_{i=1}^N \big (f_i(\Ess) \big )^2,
\end{equation}
where each element in the objective function is defined as,
\begin{equation}
    f_i(\Ess) = \obsi^T \Ess\obsi'.
\end{equation}
The identities
\begin{align}
    & \vec{\matA \matX \matB} = (\boldsymbol{B}^T \kron \boldsymbol{A})\vec{\matX} \\
    & (a ^T \kron b ^T) = (a \kron b ) ^T,
\end{align}
allow us to reformulate each term as 
\begin{equation}
    f_i(\Ess) = (\obsip^T \kron \obsi^T) \vec{\matE} = (\obsip \kron \obsi) ^T \vec{\matE}.
\end{equation}
The contribution of each element $\big (f_i(\Ess) \big)^2$ is then written as
\begin{align}
    \big (f_i(\Ess) \big)^2 &= \big((\obsip \kron \obsi) ^T \vec{\matE} \big) ^2 = \\
    &= \big((\obsip \kron \obsi) ^T \vec{\matE} \big)^T(\obsip \kron \obsi) ^T \vec{\matE} = \\
    &= \vec{\matE} ^T (\obsip \kron \obsi) (\boldsymbol{f}'_i \kron \obsi) ^T \vec{\matE} = \\
    &= \vec{\matE}^T \boldsymbol{C}_i \vec{\matE}, 
\end{align}
where we have defined the PSD matrix $ \boldsymbol{C}_i$ as
\begin{equation}
    \boldsymbol{C}_i = (\obsip\kron \obsi) (\obsip \kron \obsi) ^T  \in \symmPlus{9}, \label{eq:definitionC}
\end{equation}
with $\kron$ being the Kronecker product.

The cost function can be therefore written as
\begin{align}
    \fE &= \sum_{i=1} ^N f_i(\Ess) = \sum_{i=1}^N \vec{\matE}^T \boldsymbol{C}_i \vec{\matE} = \\
    & = \vec{\matE}^T \big(\sum_{i=1}^N \boldsymbol{C}_i \big) \vec{\matE} = \vec{\matE}^T \boldsymbol{C} \vec{\matE}, \label{eq:objectivee}
\end{align}
with $\boldsymbol{C} = \sum_{i=1}^N \boldsymbol{C}_i$.

The expression \eqref{eq:objectivee} can be re-formulated in terms of $\boldsymbol{x} = [\vec{\matE}^T, \trans]^T \in \mathbb{R}^{12}$ by defining the extended matrix \textbf{Q} and padding with zeros:
\begin{equation}
    \boldsymbol{Q} = 
    \begin{bmatrix}
    \sum_{i = 1} ^N \boldsymbol{C}_i  & \boldsymbol{0}_{9\times 3} \\
    \boldsymbol{0}_{3 \times 9} & \boldsymbol{0}_{3 \times 3} 
    \end{bmatrix}. 
\end{equation}
The objective function is therefore expressed as
\begin{equation}
    \fE = \boldsymbol{x}^T \boldsymbol{Qx}, \label{eq:objquad}
\end{equation}
which is the expression that appears in the primal problem \eqref{eq:primalproblem}.

\section{The Relaxation Set $\SetRelaxedEssentialMatrices$ in \eqref{eq:Me:min} is a strict superset of $\SetEssentialMatrices$} \label{app:Er-relaxation-E}
We prove here that the constraint set in \eqref{eq:Me:min} is indeed a relaxation of the space defined in \eqref{eq:Me:EEt} and therefore $\SetEssentialMatrices \subset \SetRelaxedEssentialMatrices$. For that, let us define the matrix $\Ess \Ess^T$ whose entries are given by the constraints in \eqref{eq:Me:min}:
\begin{equation}
    \Ess \Ess^T = 
    \begin{pmatrix}
    t_2^2 + t_3 ^2 & k & -t_1 t_3 \\
    k & t_1^2 + t_3 ^2 & -t_2 t_3 \\
    -t_1 t_3 & -t_2 t_3 & t_1^2 + t_2 ^2 
    \end{pmatrix},
    \quad\forall \Ess \in \SetRelaxedEssentialMatrices, \trans \in \sphere , k \in \Reals{}.
\end{equation}
For $\Ess$ as defined above to belong to $\SetEssentialMatrices$, its determinant must be zero and so will be the determinant of $\Ess \Ess^T$. Applying Laplace's formula and reordering:
\begin{equation}
    \det(\Ess \Ess^T) = 
   -k^2(t_1^2 + t_2^2) + 2 k t_1 t_2 t_3^2 + 2t_1^2 t_2 ^2 t_3 ^2 + t_1 ^2 t_2 ^2 (t_1 ^2 + t_2 ^2), 
\end{equation}
whose roots are given by
\begin{align}
    k_1 &= -t_1 t_2 \\
    k_2 &=  t_1 t_2 \frac{1 + t_3 ^2}{1-t_3^2}.
\end{align}
Since these conditions do not hold in general $\forall k \in \Reals{}$, the matrices defined by \eqref{eq:Me:min} do not need to be rank deficient. Therefore, the set defined by $\SetRelaxedEssentialMatrices$ is larger than $\SetEssentialMatrices$, and thus, it represents a relaxation of the latter, proving the claim of the main document.

\section{Lagrangian Dual Problem of the Primal Problem \eqref{eq:primalproblem}} \label{app:dual-problem-dev}
In order to define the dual problem \cite{boyd2004convex}, we follow the usual procedure. This development was originally given in \cite{Zhao2019} with seven constraints instead of the six employed in this work. The adaptation is therefore trivial, but we include it here for completeness. We start by writing the \textit{Lagrangian function} as
\begin{align}
    \mathcal{L}(\boldsymbol{x}, \bm{\lambda}) &= \boldsymbol{x} ^T \boldsymbol{Q} \boldsymbol{x} + \sum_{i=2} ^6 \lambda_i (- \boldsymbol{x}^T \boldsymbol{A}_ i \boldsymbol{x}) + \lambda_1 (1 - \boldsymbol{x}^T \boldsymbol{A}_ 1 \boldsymbol{x}) = \nonumber \\
    &=
    \boldsymbol{x} ^T \boldsymbol{Q} \boldsymbol{x} + \sum_{i=1} ^6 \lambda_i (- \boldsymbol{x}^T \boldsymbol{A}_ i \boldsymbol{x})  + \lambda_1 = \boldsymbol{x}^T \boldsymbol{M}(\bm{\lambda}) \boldsymbol{x} + \lambda_1, 
\end{align}
where $\bm{\lambda} = \{\lambda_i\}_{i=1} ^6$ are the \textit{Lagrange multipliers} and $\boldsymbol{M}(\bm{\lambda})$ is the \textit{Hessian of the Lagrangian} defined as
\begin{equation}
    \boldsymbol{M}(\bm{\lambda}) \doteq \boldsymbol{Q} - \sum_{i=1} ^6 \lambda_i \boldsymbol{A}_i \label{eq:M}
\end{equation}

By definition, the \textit{Lagrangian dual problem} for \eqref{eq:primalproblem} is
\begin{equation}
    \label{eq:pdualproblem}
    \dOptRlx \doteq  \max_{\bm{\lambda} } \inf_{\boldsymbol{x}} \mathcal{L} (\boldsymbol{x}, \bm{\lambda}) = \max_{\bm{\lambda} } \underbrace{\inf_{\boldsymbol{x}} \Big ( \boldsymbol{x}^T \boldsymbol{M}(\bm{\lambda}) \boldsymbol{x} + \lambda_1 \Big )}_{d (\bm{\lambda}) },
\end{equation}
where $d(\bm{\lambda})$ is known as the \textit{dual function} with $d(\bm{\lambda}) \leq f$ for any value of $\bm{\lambda}$ by definition.  

We note that $\boldsymbol{x}^T\boldsymbol{M}(\bm{\lambda})\boldsymbol{x}$ is a quadratic form whose only finite minimum value is achieved at 0 when $\boldsymbol{M}(\bm{\lambda})\succeq 0$ \ie,
\begin{equation}
    d(\bm{\lambda}) = \inf_{\boldsymbol{x}} \boldsymbol{x}^T \boldsymbol{M}(\bm{\lambda})\boldsymbol{x} + \lambda_1= 
    \begin{cases}
      \lambda_1, & \boldsymbol{M}(\bm{\lambda})\succeq 0 \\
      -\infty , & \text{otherwise}
    \end{cases}
\end{equation}
Since we are trying to find the maximum, we can restrict the problem to the finite values, \ie, the Hessian of the Lagrangian is positive semidefinite.
Hence, we can write the (unconstrained) dual problem in \eqref{eq:pdualproblem} as the constrained SDP problem,

  \begin{align}
    \dOptRlx &=  \max_{\bm{\lambda} } \lambda_1  \\
    & \text{subject to  }  \boldsymbol{M}(\bm{\lambda})\succeq 0 , \nonumber
\end{align}
which is the problem that appears in Theorem \eqref{eq:dualproblem}, proving the claim.

\section{Existence and Uniqueness of the Lagrange Multipliers for the Dual Problem \eqref{eq:dualproblem}} \label{app:unique-solution-LICQ}
We want to point out that the following statements have been adapted to our concrete problem, where only equality quadratic constraints are present. We refer the reader to \cite{boyd2004convex}, \cite{nocedal2006numerical} for a full characterization of the general case.

That being said, we proceed as follows.
Any pair of primal and dual optimal points must satisfy the KKT conditions \cite{boyd2004convex} to have strong duality (assuming both the objective and the set of constraints are differentiable). 
For the problem in \eqref{eq:primalproblem} and a pair of primal-dual optimal points ($\vStarX, \vStarLambda$), the KKT conditions (first-order necessary conditions) read:
\begin{align}
    &\text{(1) Primal feasibility } \quad \boldsymbol{A}_i\vStarX = \boldsymbol{0}, \forall i = 1, ..., m \\
    &\text{(2) Stationary} \quad \nabla (\boldsymbol{x}^{\star T}\boldsymbol{Q}\vStarX) + \sum_{i=1}^m \lambda_i^{\star}\nabla  (-\boldsymbol{x^{\star T}} \boldsymbol{A}_i\vStarX) = 0, \label{eq:KKTstationary}
\end{align}
where $m$ is the number of equality constraints.

For the relation in \eqref{eq:KKTstationary} to be a necessary optimality condition, some restrictions must be applied to the equality constraints of the primal problem. These conditions are known as \emph{constraint qualifications} (CQ)~\cite[Sec.~12.2]{nocedal2006numerical}\cite{wachsmuth2013licq}. Many CQ's have been proposed in the literature; in this work, however, we are interested in those that characterize the set of dual solutions (\ie the Lagrange multipliers) of the dual problem. The strongest CQ is the so-called Linear Independence Constraint Qualification (LICQ), which assures both the \emph{existence} and the \emph{uniqueness} of the Lagrange multipliers~\cite{wachsmuth2013licq}, \ie it implies that the set of dual solutions is a singleton.

LICQ holds if the gradients  of the \emph{active constraint} $\{\nabla  (- \boldsymbol{x}^{\star T}\boldsymbol{A}_i\vStarX)\}_{i=1}^m$ are linearly independents. If LICQ holds, the dual point $\vStarLambda$ that satisfies the KKT conditions is unique and exists, which is the cornerstone of our optimality certifier. We note that weaker constraint qualifications exist which still guarantee strong duality; however they do not in general guarantee the uniqueness of Lagrange multipliers~\cite{wachsmuth2013licq}.
Therefore, for our optimality certifier to be purposive, we must then assure that LICQ holds and list the degenerated cases. These two aspects are tackled next.

\subsection{Assuring Linear Independence Constraint Qualification}

The third KKT condition for optimality in \eqref{eq:KKTstationary} is explicitly written as the linear system in $\vStarLambda$ as:
\begin{equation}
    \boldsymbol{Qx} ^{\star} = \sum_{i=1}^m\lambda_i^{\star}\boldsymbol{A}_i\vStarX = \Jset(\boldsymbol{x^{\star}})\vStarLambda , \label{eq:jacobian}
\end{equation}
where we have defined $$\Jset(\boldsymbol{x^{\star}}) \doteq [\boldsymbol{A}_1\vStarX, \boldsymbol{A}_2\vStarX, \dots, \boldsymbol{A}_m\vStarX] \in \mathbb{R}^{12 \times m}$$ the Jacobian (\uts) of the constraint matrices evaluated at $\boldsymbol{x^{\star}}$. Hence, the LICQ condition assures that if the matrix $\Jset(\boldsymbol{x^{\star}})$ is full rank, then the Lagrange multipliers $\vStarLambda$ are unique, \ie, the linear system has either one or zero solutions. Note that if the system has no solution, one can always find the "closest" point in the least-squares sense.

We prove now that the original set of constraints given in \cite{Zhao2019} is not suitable for our optimality certifier as is. Consider this original set $\{\boldsymbol{A}_i\}_{i=1}^7$ and its associated Jacobian $\Jset(\vHatX)$ evaluated at a feasible primal point 
$\vHatX = [\vec{\hatE}^T, \hat{\trans}^T] ^T $, where $\vec{\hatE} = [\hat{e_1}, \hat{e_4},\hat{e_7}, \hat{e_2}, \hat{e_5}, \hat{e_8}, \hat{e_3}, \hat{e_6}, \hat{e_9}]^T$. $\Jset(\vHatX)$ has the following structure:
\begin{align}
    &\Jset(\vHatX) = \nonumber\\
    &\begin{pmatrix}
    0 & \hat{e_4}/ 2 & \hat{e_1}  &\hat{e_7}/2 & 0 & 0 & 0 \\
    0 & \hat{e_1}/2  &  0 & 0 & \hat{e_4} & \hat{e_7}/2 & 0 \\
    0 & 0 & 0 & \hat{e_1}/2 & 0 & \hat{e_4} / 2 & \hat{e_7} \\
    0 & \hat{e_5}/ 2 & \hat{e_2} &   \hat{e_8}/2 & 0 & 0 & 0 \\
    0 &  \hat{e_2}/2 & 0  & 0 & \hat{e_5} &  \hat{e_8}/2 & 0 \\
    0 & 0 & 0   & \hat{e_2}/2 & 0 & \hat{e_5} / 2 & \hat{e_8} \\
    0 & \hat{e_6} / 2 & \hat{e_3}  &\hat{e_9}/2 & 0 & 0 & 0 \\
    0 & \hat{e_3}/2 & 0 &  0 & \hat{e_6} & \hat{e_9}/2 & 0 \\
    0 & 0 &  0 &  \hat{e_3}/2 &  0 & \hat{e_6} / 2 & \hat{e_9} \\
    \hat{t}_{1} & \hat{t}_{2}/ 2 & 0 &   \hat{t}_{3}/ 2 & -\hat{t}_{1} & 0 & -\hat{t}_{1} \\
     \hat{t}_{2} & \hat{t}_{1}/ 2 & - \hat{t}_{2}  & 0 & 0 & \hat{t}_{3} /2 & - \hat{t}_{2} \\
    \hat{t}_{3} & 0 &  -\hat{t}_{3} & \hat{t}_{1}/ 2 & -\hat{t}_{3} &  \hat{t}_{2}/2 & 0 \\
    \end{pmatrix}
\end{align}

Next, we show how $\Jset(\vHatX)$ is indeed (column) rank deficient for all primal feasible $\vHatX$. For that, we need to find a non-null subspace $\Phi(\vHatX) \in \mathbb{R}^{7 \times r}$, where $7 - r$ is the rank of $\Jset(\vHatX)$, such that $ \Jset(\vHatX)  \Phi(\vHatX) = \boldsymbol{0}_{12 \times \text{r}}$ for all primal feasible point $\vHatX$. After that, we will show that this nullspace is indeed one-dimensional ($r = 1$).

\bigskip
\textbf{Proof that the matrix $\Jset(\vHatX)$ is rank deficient}:
Let us assume that the matrix $\Jset(\vHatX)$  is rank deficient with nullspace given by the 7D vector $\bm{\Phi} = [\Phi_1, \dots, \Phi_{7}]^T $ (possible dependent on the feasible point $\vHatX$, although in what follows we will avoid this dependence on the formulation for the sake of clarity), such that $\Jset(\vHatX)\bm{\Phi} = \boldsymbol{0}_{12}$.
First, recall that any essential matrix has as left nullspace the translation vector $\trans$ \cite{hartley2003multiple}:
\begin{equation}
    \Ess = \cross{t} \rot \implies \trans^T\Ess = \trans^T \cross{t} \rot = - \big (\cross{t} ^T \trans\big) ^T \rot = \boldsymbol{0}_3.  \label{eq:leftnullspaceE}
\end{equation}

From \eqref{eq:leftnullspaceE} one can obtain the following three equalities:
\begin{align}
     & \hat{t}_{1}  \hat{e}_{1} + \hat{t}_{2}  \hat{e}_{4} +  \hat{t}_{3}  \hat{e}_{7} = 0 \\
     & \hat{t}_{1}  \hat{e}_{2} + \hat{t}_{2}  \hat{e}_{5} +  \hat{t}_{3}  \hat{e}_{8} = 0 \\
     & \hat{t}_{1}  \hat{e}_{3} + \hat{t}_{2}  \hat{e}_{6} +  \hat{t}_{3}  \hat{e}_{9} = 0 
\end{align}

Let us denote the rows of the matrix  $\Jset(\vHatX)$ by $\boldsymbol{j}_i, i = 1, ..., 12$. Then, the explicit expressions for $\Jset(\vHatX)\bm{\phi} = \boldsymbol{0}_{12}$ are:

\begin{align}
    & \boldsymbol{j}_1 \boldsymbol{\phi} = \Phi_3 \hat{e}_1 + \Phi_2 \hat{e}_4/ 2 +\Phi_4 \hat{e}_7/ 2  = 0 \label{eq:Jfull1} \\
    & \boldsymbol{j}_4 \boldsymbol{\phi} = \Phi_3 \hat{e}_2 + \Phi_2 \hat{e}_5/ 2 +\Phi_4 \hat{e}_8/ 2  = 0 \label{eq:Jfull2} \\
    & \boldsymbol{j}_7 \boldsymbol{\phi} = \Phi_3 \hat{e}_3 + \Phi_2 \hat{e}_6/ 2 +\Phi_4 \hat{e}_9/ 2  = 0 \label{eq:Jfull3} \\
     & \boldsymbol{j}_2 \boldsymbol{\phi} = \Phi_4 \hat{e}_1 / 2 + \Phi_6 \hat{e}_4/ 2 +\Phi_7 \hat{e}_7  = 0 \label{eq:Jfull4} \\
    & \boldsymbol{j}_5 \boldsymbol{\phi} = \Phi_4 \hat{e}_2 / 2+ \Phi_6 \hat{e}_5/ 2 +\Phi_7 \hat{e}_8  = 0 \label{eq:Jfull5} \\
    & \boldsymbol{j}_8 \boldsymbol{\phi} = \Phi_4 \hat{e}_3 / 2 + \Phi_6 \hat{e}_6/ 2 +\Phi_7 \hat{e}_9 = 0 \label{eq:Jfull6} \\
     & \boldsymbol{j}_3 \boldsymbol{\phi} = \Phi_2 \hat{e}_1/2 + \Phi_5 \hat{e}_4 +\Phi_6 \hat{e}_7/ 2  = 0 \label{eq:Jfull7} \\
    & \boldsymbol{j}_6 \boldsymbol{\phi} = \Phi_2 \hat{e}_2/2 + \Phi_5 \hat{e}_5 +\Phi_6 \hat{e}_8/ 2  = 0 \label{eq:Jfull8} \\
    & \boldsymbol{j}_9 \boldsymbol{\phi} = \Phi_2 \hat{e}_3/2 + \Phi_5 \hat{e}_6 +\Phi_6 \hat{e}_9/ 2  = 0 \label{eq:Jfull9} \\
    & \boldsymbol{j}_{10} \boldsymbol{\phi} = \phi_1 \hat{t}_1 + \phi_2 \hat{t}_{2}/2 + \phi_4 \hat{t}_{3}/2 - \phi_5\hat{t}_{1} - \phi_7 \hat{t}_{1} = 0  \label{eq:Jfull10}\\
    & \boldsymbol{j}_{11} \boldsymbol{\phi} = \phi_1 \hat{t}_2  + \phi_2 \hat{t}_{1}/2 - \phi_3\hat{t}_{2} + \phi_6 \hat{t}_{3}/2 - \phi_7 \hat{t}_{2} = 0 \label{eq:Jfull11} \\
    & \boldsymbol{j}_{12} \boldsymbol{\phi} = \phi_1 \hat{t}_3  - \phi_3 \hat{t}_{3} + \phi_4 \hat{t}_1 / 2 - \phi_5\hat{t}_{3} + \phi_6\hat{t}_{2}/2 = 0 \label{eq:Jfull12}
\end{align}

From Equations \eqref{eq:Jfull1},  \eqref{eq:Jfull2},  \eqref{eq:Jfull3}, one can see that a feasible parameterization for $\phi_2, \phi_3, \phi_4$ is given by: 
\begin{equation}
    \phi_2 = 2 \hat{t}_2 \alpha_2 \quad \phi_3 = \hat{t}_1 \alpha_2 \quad \phi_4 = 2 \hat{t}_3 \alpha_2,
\end{equation} 
by the relation given in \eqref{eq:leftnullspaceE}, being $\alpha_2 \in \Reals{}$ an unknown scalar. Following the same argument with \eqref{eq:Jfull4},  \eqref{eq:Jfull5},  \eqref{eq:Jfull6}, one obtain that:  
\begin{equation}
    \phi_4 = 2 \hat{t}_1 \alpha_3 \quad \phi_6 = 2\hat{t}_2 \alpha_3 \quad \phi_7 = \hat{t}_3 \alpha_3, \quad\alpha_3 \in \Reals{}.
\end{equation} 
Further, Equations \eqref{eq:Jfull6},  \eqref{eq:Jfull7},  \eqref{eq:Jfull8} yield a similar relation
\begin{equation}
   \phi_2 = 2 \hat{t}_1 \alpha_4\quad \phi_5 = \hat{t}_2 \alpha_4 \quad \phi_6 = 2\hat{t}_3 \alpha_4, \quad \alpha_4 \in \Reals{}.
\end{equation} 
Since all the elements of the vector $\bm{\phi}$ must be compatible, all the parameterizations must attain the same value. Concretely, we have two different expressions for $\phi_2, \phi_4, \phi_6$ and hence,
\begin{align}
    \text{From } \phi_2:\ & 2 \hat{t}_2 \alpha_2 = 2 \hat{t}_1 \alpha_4 \implies \alpha_2 = \hat{t}_1 \gamma, \quad \alpha_4 = \hat{t}_2 \gamma, \quad \gamma \in \Reals{} \\
    \text{From } \phi_4:\ & 2 \hat{t}_3 \alpha_2 = 2 \hat{t}_1 \alpha_3 \implies \alpha_2 = \hat{t}_1 \beta, \quad \alpha_3 = \hat{t}_3 \beta, \quad \beta \in \Reals{} \\
    \text{From } \phi_6:\ & 2 \hat{t}_2 \alpha_3 = 2 \hat{t}_3 \alpha_4 \implies \alpha_3 = \hat{t}_3 \psi, \quad \alpha_4 = \hat{t}_2 \psi, \quad \psi \in \Reals{}.
\end{align}
Further, $\gamma = \beta = \psi$ since all the relations must hold at the same time. Let us fix this common scale to one without loss of generality. Therefore, we obtain the explicit expressions for $\phi_2, \dots, \phi_7$ as:
\begin{align}
    & \phi_2 = 2 \hat{t}_1 \hat{t}_2, \quad \phi_3 = \hat{t}_1 ^2, \quad \phi_4 = 2 \hat{t}_1 \hat{t}_3 \\
    & \phi_5 = \hat{t}_2^2, \quad \phi_6 = 2 \hat{t}_2 \hat{t}_3, \quad \phi_7 = \hat{t}_3 ^2. 
\end{align}

Introducing these explicit forms into \eqref{eq:Jfull10}, \eqref{eq:Jfull11}, \eqref{eq:Jfull12}:
\begin{align}
    &  \phi_1 \hat{t}_1 + \hat{t}_{1} \hat{t}_{2}^2 + \hat{t}_{1} \hat{t}_{3}^2 - \hat{t}_{1} \hat{t}_{2}^2 - \hat{t}_{1} \hat{t}_{3}^2 = 0 \implies  \phi_1 \hat{t}_1 = 0 \\
    & \phi_1 \hat{t}_2 + \hat{t}_{1} \hat{t}_{2}^2 - \hat{t}_{1} \hat{t}_{2}^2 + \hat{t}_{2} \hat{t}_{3}^2 - \hat{t}_{2} \hat{t}_{3}^2 = 0 \implies  \phi_1 \hat{t}_2 = 0\\
    & \phi_1 \hat{t}_3 -\hat{t}_{1} \hat{t}_{3}^2 + \hat{t}_{1} \hat{t}_{3}^2 - \hat{t}_{3} \hat{t}_{2}^2 + \hat{t}_{3} \hat{t}_{2}^2 = 0 \implies  \phi_1 \hat{t}_3 = 0
\end{align}
Since $\trans^T \trans \neq 0$, in order to fulfill all the relations it is necessary that $\phi_1 = 0$. 
The 7D vector which lies in the nullspace of $\Jset(\vHatX)$ takes the final form:
\begin{equation}
    \bm{\phi} = [0,\ 2\hat{t}_1\hat{t}_2,\ \hat{t}_1^2,\ 2\hat{t}_1 \hat{t}_3,\ \hat{t}_2^2,\  2\hat{t}_2 \hat{t}_3,\ \hat{t}_3^2] ^T.
\end{equation}
Since $\bm{\phi}$ depends on $\trans \neq \boldsymbol{0}_3$, it is not null by construction. 
We note that the same vector was given in \cite{Zhao2019} for a different purpose, although the explicit construction was not provided there. 
Moreover, we can explicitly give the dependence as:
\begin{align}
    & \boldsymbol{0} = \Jset (\vHatX)\boldsymbol{\phi}(\vHatX) = \sum_{k=2} ^7 \phi_k {\boldsymbol{A}}_k \vHatX \Leftrightarrow  \\
    & \Leftrightarrow \phi_i {\boldsymbol{A}}_i\vHatX = - \sum_{j=2, j \neq i} ^7 \phi_j {\boldsymbol{A}}_j\vHatX. \label{eq:depgradients}
\end{align}
Equation \eqref{eq:depgradients} shows that \emph{any} term $\phi_i {\boldsymbol{A}}_i \vHatX, \text{ for } i = 2, ..., 7$  can be written as a function of the other terms $\phi_j {\boldsymbol{A}}_j \vHatX, \text{ for } j \neq i, j = 2, ..., 7$ for \emph{any} feasible point. Recall that the constraint matrix $\boldsymbol{A}_1$ corresponds with the unitary constraint in $\trans$, while $\{\boldsymbol{A}_i\}_{i=2} ^7$ relate the essential matrix to its left nullspace $\boldsymbol{EE} ^T = \cross{t}\cross{t}^T$. That is, the gradients of the constraints associated with the expression $\boldsymbol{EE} ^T = \cross{t}\cross{t}^T$ are linearly dependents. 
By dropping one of the columns associated with these constraints (here we choose to drop the second column, while we postpone a full analysis about their influence for future projects), we obtain the following (reduced) Jacobian $\JsetIdx(\vHatX) \in \mathbb{R}^{12 \times 6}$. 
\begin{equation}
    \JsetIdx (\vHatX) \doteq 
    \begin{pmatrix}
    0  & \hat{e_1}  &\hat{e_7}/2 & 0 & 0 & 0 \\
    0  &  0 & 0 & \hat{e_4} & \hat{e_7}/2 & 0 \\
    0  & 0  & \hat{e_1}/2 & 0 & \hat{e_4} / 2 & \hat{e_7} \\
    0 & \hat{e_2} &   \hat{e_8}/2 & 0 & 0 & 0 \\
    0 & 0  & 0 & \hat{e_5} &  \hat{e_8}/2 & 0 \\
    0 & 0   & \hat{e_2}/2 & 0 & \hat{e_5} / 2 & \hat{e_8} \\
    0 & \hat{e_3}  &\hat{e_9}/2 & 0 & 0 & 0 \\
    0 &  0 &  0 & \hat{e_6} & \hat{e_9}/2 & 0 \\
    0 &  0  & \hat{e_3}/2 & 0 & \hat{e_6} / 2 & \hat{e_9} \\
    \hat{t}_{1}  & 0 &   \hat{t}_{3}/ 2 & -\hat{t}_{1} & 0 & -\hat{t}_{1} \\
     \hat{t}_{2} &  - \hat{t}_{2}  & 0 & 0 & \hat{t}_{3} /2 & - \hat{t}_{2} \\
    \hat{t}_{3} &  -\hat{t}_{3} & \hat{t}_{1}/ 2 & -\hat{t}_{3} &  \hat{t}_{2}/2 & 0 \\
    \end{pmatrix}
\end{equation}
We have shown that there exists a 7D vector which lies in the nullspace of $\Jset$. Next, we will prove that in fact the nullspace of $\Jset$ is one-dimensional..

\bigskip
\textbf{Proof that the nullspace of $\Jset$ is one-dimensional}:

To proof that $\JsetIdx(\vHatX)$ is full rank and thus the nullspace of the original Jacobian matrix $\Jset(\vHatX)$ is one-dimensional, let us assume that there exists a vector $ \Reals{6} \ni \xi(\vHatX) \doteq [\xi_1, \xi_2, \xi_3, \xi_4, \xi_5, \xi_6] ^T$ such that $\JsetIdx(\vHatX) \xi(\vHatX) = \boldsymbol{0}_6$. In what follows, we will omit the dependence of $\bm{\xi}(\vHatX)$ on the feasible point $\vHatX$ and simply write $\bm{\xi}$ for clarity.  Hence:
\begin{align}
    &  \xi_2\hat{e}_{1} + \xi_3 \hat{e}_{7}/ 2 = 0  \label{eq:Jredr1} \\
    & \xi_2\hat{e}_{2} + \xi_3\hat{e}_{8}/ 2 = 0    \label{eq:Jredr2} \\
    & \xi_2\hat{e}_{3} + \xi_3\hat{e}_{9}/ 2 = 0    \label{eq:set-first-eq}\\
    &  \xi_4\hat{e}_{4} + \xi_5 \hat{e}_{7}/ 2 = 0   \label{eq:Jredr4} \\
    & \xi_4\hat{e}_{5} + \xi_5 \hat{e}_{8}/ 2 = 0  \label{eq:Jredr5} \\
    & \xi_4\hat{e}_{6} + \xi_5 \hat{e}_{9}/ 2 = 0 \label{eq:set-second-eq}\\
    & \xi_3 \hat{e}_{1}/ 2 + \xi_5\hat{e}_{4}/2 + \xi_6 \hat{e}_{7} = 0   \label{eq:Jredr7} \\
    & \xi_3 \hat{e}_{2}/ 2 + \xi_5\hat{e}_{5}/2 + \xi_6 \hat{e}_{8} = 0   \label{eq:Jredr8} \\
    & \xi_3 \hat{e}_{3}/ 2 + \xi_5\hat{e}_{6}/2 + \xi_6 \hat{e}_{9} = 0   \label{eq:set-third-eq}\\
    %
    & \xi_1 \hat{t}_{1} + \xi_3 \hat{t}_{3}/2 - \xi_4 \hat{t}_{1}  - \xi_6 \hat{t}_{1} = 0  \label{eq:Jredr10} \\
    & \xi_1 \hat{t}_{2} - \xi_2 \hat{t}_{3}/2 - \xi_5 \hat{t}_{3}/ 2  - \xi_6 \hat{t}_{2} = 0   \label{eq:Jred11} \\
    & \xi_1 \hat{t}_{3} - \xi_2 \hat{t}_{3}/2 + \xi_3 \hat{t}_{1}/2  - \xi_4 \hat{t}_{3} + \xi_5 \hat{t}_{2}/ 2 = 0  \label{eq:set-fourth-eq} 
\end{align}

Four different set of equations can be discerned depending on the elements from $\bm{\xi}$ involved: Equations \eqref{eq:Jredr1},\eqref{eq:Jredr2}, \eqref{eq:set-first-eq} involving $\xi_2, \xi_3$ (the set $\Cfirst$); Equations \eqref{eq:Jredr4},\eqref{eq:Jredr5}, \eqref{eq:set-second-eq} involving $\xi_4, \xi_5$ (the set $\Csecond$); Equations \eqref{eq:Jredr7},\eqref{eq:Jredr8}, \eqref{eq:set-third-eq} involving $\xi_3, \xi_5, \xi_6$  (the set $\Cthird$); and Equations \eqref{eq:Jredr10},\eqref{eq:Jred11}, \eqref{eq:set-fourth-eq} involving $\xi_1, \xi_2, \xi_3, \xi_4, \xi_5, \xi_6$ (the set $\Call$). Further, given this block-alike structure, we can treat four different cases separately.

\smallskip
\textbf{Case I}:
Consider the set $\Cfirst$ and form the linear system in $\xi_2, \xi_3$ as:
\begin{equation}
    \underbracket{\begin{pmatrix}
    \hat{e}_{1} & \hat{e}_{7}/2  \\
    \hat{e}_{2} & \hat{e}_{8} / 2 \\
    \hat{e}_{3} & \hat{e}_{9} / 2
    \end{pmatrix}}_{\boldsymbol{G}_{13}}
    \begin{pmatrix}
    \xi_2 \\
    \xi_3
    \end{pmatrix}
    = \boldsymbol{0}_3.
\end{equation}
If the 2D vector $[\xi_2, \xi_3]^T$ is non-null, then the matrix $\boldsymbol{G}_{13}$ has one singular value equal to zero or equivantely, the matrix $\boldsymbol{G}_{13}^T\boldsymbol{G}_{13}$ has one eigenvalue equal to zero. Note that if it has 2 eigenvalues equal to zero, then the matrix is null, \ie $\hat{e}_{1} = \hat{e}_{2} = \hat{e}_{3} = \hat{e}_{7} = \hat{e}_{8} = \hat{e}_{9} = 0$ and thus, the first and third of $\Ess$ are null. This is not an essential matrix, as it will be explained in the next subsection. For now, just consider the case when $\boldsymbol{G}_{13}^T\boldsymbol{G}_{13}$ has one zero eigenvalue, \ie the first and third rows of the essential matrix $\Ess$ are linear dependent.

For simplicity, let us denote the rows of $\Ess$ by $\{\boldsymbol{e}_i \in \Reals{3}\}_{i=1}^3$. 
If $\boldsymbol{e}_1$ or $\boldsymbol{e}_3$ are null, then the essential matrix with these rows is \emph{not} an essential matrix (we refer again the reader to the next subsection for these cases).
This implies that $\boldsymbol{e}_1 = \alpha\boldsymbol{e}_3$ with $\alpha \in \Reals{} / \{0 \}$ and the essential matrix will have two identical rows up to (signed) scale. Note that $\alpha = 0$ is excluded since otherwise $\boldsymbol{e}_1 = 0 \boldsymbol{e}_3 = \boldsymbol{0}_3$. We will show in the next subsection that a $3 \times 3$ matrix with one (or two) zero rows cannot be an essential matrix.
Considering the definition of essential matrix in \eqref{eq:Me:[t]xR} and the above-mentioned condition about the first and third rows, we obtain the following relation between the elements in $\rot$ and $\trans$:
\begin{equation}
    \underbracket{\begin{pmatrix}
    r_1 & r_4 & r_7 \\
    r_2 & r_5 & r_8 \\
    r_3 & r_6 & r_9 
    \end{pmatrix}}_{\rot^T}
    \begin{pmatrix}
    \alpha t_2 \\
    -\alpha t_1 - t_3 \\
    t_2
    \end{pmatrix}
    = \boldsymbol{0}_3.
\end{equation}
Since $\rot$ is a rotation matrix, it is invertible\footnote{
Recall that any rotation matrix fulfills $\rot\rot^T = \boldsymbol{I}_3$, and hence, it's invertible. A null eigenvalue prevents this.}, which implies that the 3D vector $[-t_2,  -t_1 \pm t_3 ,    \pm t_2]^T$ must be zero for the system to hold, \ie $t_2 = 0, t_3 = - \alpha t_1$. The translation vector $\trans$ takes the form $ \trans = [t_1, 0, -\alpha t_1]^T / ||\trans||_2$.
%

We are left to show how this result yields to the trivial solution for the nullspace of $\JsetIdx$.

Recall the original set of Equations. First we can compute from the set $\Cfirst$: $\xi_2 = - \frac{1}{2}\xi_3\pseudo{\boldsymbol{e}_1}\boldsymbol{e}_3 = -\frac{1}{2\alpha}\xi_3$. Let us assume that the second and third rows of $\Ess$ are not dependent (otherwise the rank of the essential matrix will be one and hence, not a essential matrix by definition). Therefore, $\xi_4 = \xi_5 = 0$ (set $\Csecond$), which leads thought the set $\Cthird$ to the relation
$\xi_6 = - \frac{1}{2}\xi_3\pseudo{\boldsymbol{e}_3}\boldsymbol{e}_1 = - \frac{\alpha}{2} \xi_3 = \alpha^2\xi_2$. 
Finally, the equations in $\Call$ read:
\begin{align}
     & \xi_1 \hat{t}_{1} - \alpha \xi_ 2 \hat{t}_{3} - \alpha^2 \xi_2 \hat{t}_1= 0 \\
    & \xi_1 \hat{t}_{2} - \xi_2 \hat{t}_{3}  - \alpha ^2\xi_2 \hat{t}_{2} = 0 \label{eq:second-final-first-case} \\
    & \xi_1 \hat{t}_{3} - \xi_2  \hat{t}_{3} - \alpha \xi_2 \hat{t}_{1} = 0. 
\end{align}
Note that the left-hand side of \eqref{eq:second-final-first-case} with the specified translation $ \trans$ takes the form: $\xi_1 0 - \xi_2 \hat{t}_{3}  - \alpha ^2\xi_2 0 = - \xi_2 \hat{t}_{3}$ which is only equal to zero when $\xi_2 = \xi_6 = \xi_3 = 0$. Hence. the only vector that lies in the nullspace of $\boldsymbol{G}_{13}^T \boldsymbol{G}_{13}$ is the nullvector, \ie it is full rank and so is $\boldsymbol{G}_{13}$. Further, $\xi_1 = 0$ since $\trans \neq \boldsymbol{0}_3$, the vector $\bm{\xi} = \boldsymbol{0}_6$ and for this first case, the matrix $\JsetIdx$ is full rank.

\smallskip
\textbf{Case II}:
Similarly, consider the set of Equations $\Csecond$. We form the linear system in $\xi_4, \xi_5$ as:
\begin{equation}
    \underbracket{\begin{pmatrix}
    \hat{e}_{4} & \hat{e}_{7}/2  \\
    \hat{e}_{5} & \hat{e}_{8} / 2 \\
    \hat{e}_{6} & \hat{e}_{9} / 2
    \end{pmatrix}}_{\boldsymbol{G}_{46}}
    \begin{pmatrix}
    \xi_4 \\
    \xi_5
    \end{pmatrix}
    = \boldsymbol{0}_3 
    \Leftrightarrow
    \xi_4
    \begin{pmatrix}
    \hat{e}_{4} \\
    \hat{e}_{5} \\
    \hat{e}_{6}
    \end{pmatrix}
    = 
    -\frac{1}{2}\xi_5 
     \begin{pmatrix}
    \hat{e}_{7} \\
    \hat{e}_{8} \\
    \hat{e}_{9}
    \end{pmatrix}.
\end{equation}
Therefore, the above linear system degenerates when the second and third row of $\Ess$ agree in direction (signed scale): $\boldsymbol{e}_2 = \beta \boldsymbol{e}_3$, with $\beta \in \Reals{} / \{0\}$, where one again $\beta = 0$ has been discarded for leading to a degenerate $3 \times 3$ matrix. Following the same procedure, one obtain that this configuration corresponds with a translation vector of the form $\trans = [0, t_2, -\beta t_2] ^T / ||\trans||_2$. 
We follow a similar procedure to show that the only solution for the nullspace of $\JsetIdx$ for this case is the trivial (all zero) one.

From the original set of Equations, and since in this case the second and third rows of $\Ess$ agree in direction, one has that $\xi_4 = - \frac{1}{2}  \pseudo{\hat{\boldsymbol{e}}_2}\hat{\boldsymbol{e}}_3\xi_5 = - \frac{1}{2\beta} \xi_5$ and $\xi_2 = \xi_3 = 0$ from the second and first sets, respectively $\Csecond$, $\Cfirst$. We relate $\xi_6$ with $\xi_5$ by the expressions in $\Cthird$ as $\xi_6 = - \frac{1}{2} \pseudo{\hat{\boldsymbol{e}}_3}\hat{\boldsymbol{e}}_2 \xi_5 = - \frac{\beta}{2}\xi_5 = \beta^2 \xi_4$.

Finally, the set $\Call$ takes the form:
\begin{align}
     & \xi_1 \hat{t}_{1} - \xi_ 4 \hat{t}_{1}  - \beta^2 \xi_4 \hat{t}_{1}= 0 \\
    & \xi_1 \hat{t}_{2} + \beta \xi_4 \hat{t}_{3}  - \beta^2 \xi_4 \hat{t}_{1} = 0 \label{eq:second-final-second-case} \\
    & \xi_1 \hat{t}_3 - \xi_4 \hat{t}_3 - \beta \xi_4 \hat{t}_2 = 0
    \label{eq:second-final-second-case-2}
\end{align}
While the first equation is trivially satisfied with $\trans$, we obtain from \eqref{eq:second-final-second-case}: 
\begin{equation}
    0 = \xi_1 \hat{t}_{2} - \beta^2 \xi_4 \hat{t}_{2} \implies \xi_1 = \beta^2 \xi_4.
\end{equation}
Incorporating this relation into \eqref{eq:second-final-second-case-2}, we finally obtain:
\begin{align}
     0 = \beta^2 \xi_4\hat{t}_3 - \xi_4 \hat{t}_3 + \xi_4 \hat{t}_3 = \beta^2 \xi_4\hat{t}_3 \implies \xi_4 = 0,
\end{align}
since neither $\hat{t}_3 = 0$ (otherwise $\hat{t}_2 = 0$ and $\trans = \boldsymbol{0}_3$, which is not feasible) nor $\beta = 0$ (otherwise the third row of the essential matrix will be null, and therefore, not an essential matrix). Back-substituting, we can finally affirm that the only vector which lies in the nullspace of $\boldsymbol{G}_{46}$ is the null vector, thus the matrix is full rank. Further, for this case $\JsetIdx$ is full rank since $\bm{\xi} = \boldsymbol{0}_6$.

\smallskip
\textbf{Case III}: 
Consider now the Equations $\Cthird$. One can obtain a similar (linear) system in $\xi_3, \xi_5, \xi_6$:
\begin{equation}
\label{eq:sys-third-case}
    \underbracket{\begin{pmatrix}
    \hat{e}_{1}/2 & \hat{e}_{4}/2 & \hat{e}_{7} \\
    \hat{e}_{2}/2 & \hat{e}_{5}/2 & \hat{e}_{8} \\
    \hat{e}_{3}/2 & \hat{e}_{6}/2 & \hat{e}_{9} \\
    \end{pmatrix}
    }_{\Ess^T}
    \begin{pmatrix}
    \xi_3 \\
    \xi_5 \\
    \xi_6
    \end{pmatrix}
    = \boldsymbol{0}_3.
\end{equation}

Let us assume that the vector $[\xi_3, \xi_5, \xi_6]^T$ is non null. Since one of the rows is a linear combination of the other two, we can derive that nor the first or the second rows have the same direction than the third row. From the set $\Cfirst$, we have that $\xi_2 = \xi_3 = 0$. Further, from $\Csecond$ we obtain a similar result $\xi_4 = \xi_5 = 0$ and hence, $\xi_6$ as well. The vector $[\xi_3, \xi_5, \xi_6]^T$ is in fact null, which contradicts our previous assumption. Further, from the set of Equations in $\Call$, we see that we require $\xi_1 = 0$ since $\trans \neq \boldsymbol{0}_3$ by definition, \ie the obtained nullvector for this case that \emph{also} fulfills all the set of equations is null and $\JsetIdx$ is full rank. 

As an additional note, notice that the system in \eqref{eq:sys-third-case} has as (right) nullvector any 3D vector with the general form $[\hat{t}_1 \alpha, \hat{t}_2 \beta, \hat{t}_3\gamma]^T$, where, as before, $[\hat{t}_1, \hat{t}_2, \hat{t}_3]^T$ is the translation vector and $\alpha, \beta, \gamma \in \Reals{}$. However, in order to fulfill the rest of the equations, we require that $\alpha = \beta = \gamma = 0$, which agrees with our previous development and is nevertheless a valid selection of parameters.

\smallskip
\textbf{Case IV}:
In this last case, we know from the previous scenarios that the rows of the essential matrix $\Ess$ do not agree in direction between them (since otherwise one of the previous cases will hold). We will treat the cases in which one full row is zero in the following Section, as a degeneracy. Thus, for this case $\xi_2 = \xi_3 = \xi_4 = \xi_5 = \xi_6 = 0$ and since $\trans \neq \boldsymbol{0}_3$, we also obtain that $\xi_1 = 0$. Once again, the nullvector is the only vector that lies in the nullspace of $\JsetIdx$, therefore it is full rank.  

\bigskip
\emph{To wrap-up this Section}, we have proved that (1) the Jacobian matrix $\Jset$ is rank deficient; and (2) that its nullspace is one-dimensional, and hence $\JsetIdx$ (the Jacobian $\Jset$ without the second column) is full rank and suitable for our optimality certifier. 
We show next under which circumstances (if there exist) $\JsetIdx(\vHatX)$ is rank deficient.

\subsection{Degenerate Cases}

The matrix $\JsetIdx(\vHatX)$ has an empty (right) nullspace in general. We analyze here under which circumstances the matrix degenerates. Concretely, we seek the feasible primal points (if they exist) that drop the rank of the matrix. For this task, we will study each column of $\JsetIdx(\vHatX)$ individually (Cases I, II and III). Given the pattern in the matrix $\JsetIdx(\vHatX)$ , we only need to study three different cases. Last, Case IV tackles the degeneracy of the rows of $\JsetIdx(\vHatX)$.

\textbf{Case I}:
The first column is only null if all the entries in $\trans$ are zero, which is not a feasible primal point ($\trans = [0, 0, 0]^T \implies \trans^T\trans \neq 1$) and can be safely discarded.

\textbf{Case II}:
For the second column to be null, all the entries must be also zero. Hence, the translation takes the form $\trans = [1, 0, 0]^T $ (and similar for the other cases) and the first row of the (possible) essential matrix is $[0, 0, 0]^T$, which for any rotation matrix with the form
\begin{equation}
     SO(3) \ni \rot =
     \begin{pmatrix}
     r_1 & r_2 & r_3 \\
     r_4 & r_5 & r_6 \\
     r_7 & r_8 & r_9
     \end{pmatrix},
\end{equation}
lead to the following chain of equalities:
\begin{equation}
    \Ess = 
    \begin{pmatrix}
    0 & 0 & 0 \\
    e_4 & e_5 & e_6 \\
    e_7 & e_8 & e_9
    \end{pmatrix}
    \overset{(a)}{=} 
    \begin{pmatrix}
    -r_4 & -r_5 & -r_6 \\
    r_1 & r_2 & r_3 \\
    0 & 0 & 0
    \end{pmatrix}
\end{equation}
where the first equality is obtained by substituting the zero elements and the second, from the definition of essential matrix in \eqref{eq:Me:[t]xR}. One can see that for the equality $(a)$ to hold (both matrices are equal by construction), the corresponding elements must be equal, that is, the second row of the rotation matrix $\rot$ must be zero and hence, the rotation matrix will have a null eigenvalue, which is not possible. Columns fourth and sixth yield the same conclusion. 

\medskip
\textbf{Case III}:
Last, a similar result is obtained for the third and fifth columns. Considering the former, we work with a translation vector of the form $\trans = [0, 1, 0]^T$, while the putative essential matrix has only one non-zero row (the second one), hence it has two zero singular values (and by definition, it is not an essential matrix). Further, considering its form as a function of the rotation matrix:
\begin{equation}
    \Ess = 
    \begin{pmatrix}
    0 & 0 & 0 \\
    e_4 & e_5 & e_6 \\
    0 & 0 & 0
    \end{pmatrix}
    \overset{(a)}{=} 
    \begin{pmatrix}
    r1 & r8 & r9 \\
    0 & 0 & 0 \\
    -r1 & -r2 & -r3
    \end{pmatrix}
\end{equation}
In this case, for the equality \textit{(a)} to hold, the first and third rows of the rotation matrix $\rot$ must be zero, \ie it will have two null eigenvalues, which once again it's not possible.

\medskip
\textbf{Case IV}:
Finally, we tackle under which circumstances the rank of the matrix drops because of the rows. Note that since $\JsetIdx(\vHatX) \in \Reals{12 \times 6}$, at least 6 rows of the matrix must be zero at the same time. Given its structure, this is not possible.
\begin{itemize}
    \item If one of the last three rows is null, then $\trans = \boldsymbol{0}_3$ which is not a feasible point.
    \item If the third, sixth or/and ninth rows is zero, then the associated row in the essential matrix $\Ess$ will be also zero, which has been proved to be infeasible in the Case II.
    \item If one of row pairs $\{1, 2\}$,  $\{4, 5\}$ or  $\{7,8\}$ are zero, we have again that one of the rows of $\Ess$ null.
\end{itemize}
Note that any combination of the above-mentioned cases will yield the same result.

\bigskip
\emph{To wrap-up this Section}, the matrix $\JsetIdx(\vHatX)$ does not present degenerate cases and thus, the system in \eqref{eq:systeminlambdahat} is never under-determined, having always one or zero solutions, in which case one can always find the closer solution in the least-squares sense.

\section{Euclidean Operators for the Proposed Riemannian Optimization of Problem \eqref{eq:originalproblem}} \label{app:euclidean-operators-manifold}
Recall the Riemannian optimization problem in \eqref{eq:originalproblem}:

\begin{equation}
    \fOptRlx  = 
    \min_{\Ess \in \Me} \underbrace{\vec{\matE}^T \boldsymbol{C} \vec{\matE}}_{\fE}, \quad \Me \subset \Reals{3 \times 3}.
\end{equation}

Since $\boldsymbol{C} \in \symmPlus{9}$, the cost function $\fE$ of our problem is a positive semidefinite quadratic function when considered as a function of the ambient Euclidean space $\Reals{3 \times 3}$.
The Euclidean gradient and Hessian-vector product can be identified from a Taylor expansion of this cost function using the concept of Fréchet derivative (for a similar derivation, see \cite{briales2017cartan}). Consider the point $\Ess \in \Reals{3 \times 3}$, the scalar $t \in \Reals{}$ and the direction $\boldsymbol{U} \in \Reals{3 \times 3}$, then:

\begin{align}
    f(\Ess + t\boldsymbol{U}) &= \tr( (\vec{\Ess + t \boldsymbol{U}})^T \boldsymbol{C}(\vec{\Ess + t \boldsymbol{U}})) = \\
    & = \tr( \vec{\matE}^T \boldsymbol{C}\vec{\matE}) + 2t \tr(\vec{\matE}^T\boldsymbol{C}\vec{\matU}) + \nonumber \\
    & + t^2 \tr(\vec{\matU}^T \boldsymbol{C}\vec{\matU}) = \\
    & = \fE + t\inner{2\boldsymbol{C}\vec{\matE}, \vec{\matU}} + \frac{1}{2}t^2 \inner{2\boldsymbol{C}\vec{\matU}, \vec{\matU}} = \\
    & = \fE + t\inner{\nabla \fE, \vec{\matU}} + \frac{1}{2}t^2 \inner{\nabla^2 \fE [\boldsymbol{U}], \vec{\matU}}
\end{align}
where we note that the $\vec{\cdot}$ operator is linear: $\vec{\boldsymbol{A} + c\boldsymbol{B}} = \vec{\matA} + c \vec{\matB}, \forall c \in \Reals{}, \boldsymbol{A, B} \in \Reals{n \times m}$ and for any vector $\boldsymbol{a} \in \Reals{m}, \boldsymbol{b} \in \Reals{m}$ and a symmetric matrix $\symm{m \times m} \ni \boldsymbol{C} = \boldsymbol{C}^T$, it holds that:
\begin{align}
    \tr(\boldsymbol{a}^T C \boldsymbol{b}) = \tr\big ( (\boldsymbol{a}^T C \boldsymbol{b}) ^T \big) = \tr(\boldsymbol{b}^T C^T \boldsymbol{a}) = \tr(\boldsymbol{b}^T C \boldsymbol{a}).
\end{align}

Therefore the Euclidean operators as a function of $\Ess$ are explicitly given by:
\begin{equation}
    \nabla \fE = 2\boldsymbol{C}\vec{\matE}, \quad \nabla^2 \fE [\boldsymbol{U}] = 2\boldsymbol{C}\vec{\matU}.
\end{equation}

Despite its age, the essential matrix manifold has been usually defined by means of alternative variables, usually rotation matrices (see for example \cite{tron2017space, helmke2007essential}). This means that the characterization of the Riemannian counterparts of the gradient and Hessian-vector product is not direct (as, for example, in \cite{briales2017cartan}) and one first need to express the problem as a function of these auxiliary variables, which may not be straightforward. Luckily, Riemannian optimization suites, such as \manopt, allow to specify the problem in $\Ess$ and transform it into the suitable representation under-the-hood, avoiding the corresponding mathematical effort associated with this re-formulation.

\section{Performance of the Proposed Certifiable Pipeline for Noise Levels 0.5 and 2.5 pixels} \label{app:results-noise-extra}
In this Section we provide the error in rotation for the proposed pipeline initialized with the \eightpt algorithm for the noise levels 0.5 and 2.5 \pixels in Figure \eqref{fig:error-rot-05} and \eqref{fig:error-rot-25} respectively.
\begin{figure*}
    \centering
    \begin{subfigure}[t]{0.49\textwidth}
        \centering
        \includegraphics[width=\figureSize]{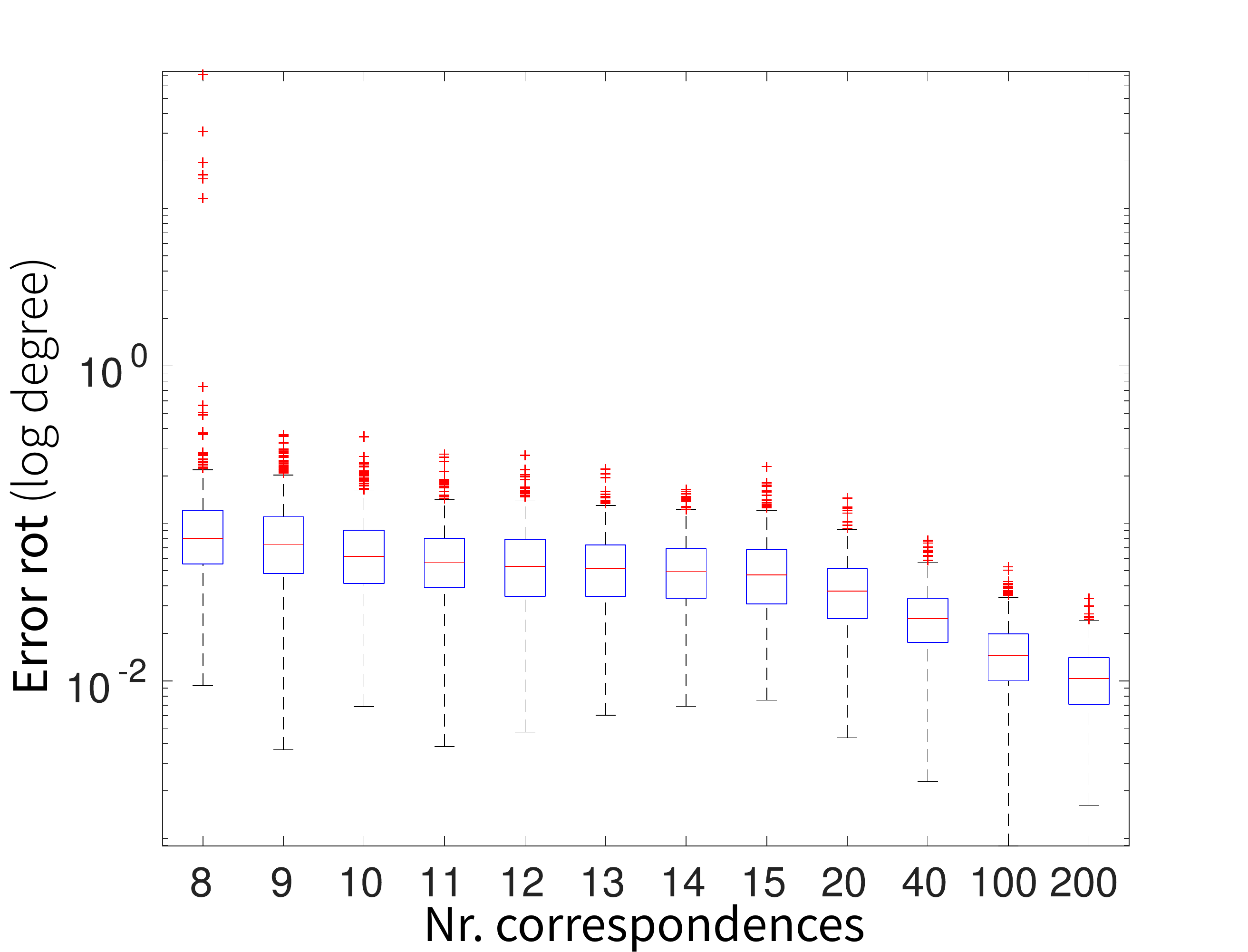}
        \caption{}
        \label{fig:error-rot-05}
    \end{subfigure}
    \begin{subfigure}[t]{0.49\textwidth}
        \centering
        \includegraphics[width=\figureSize]{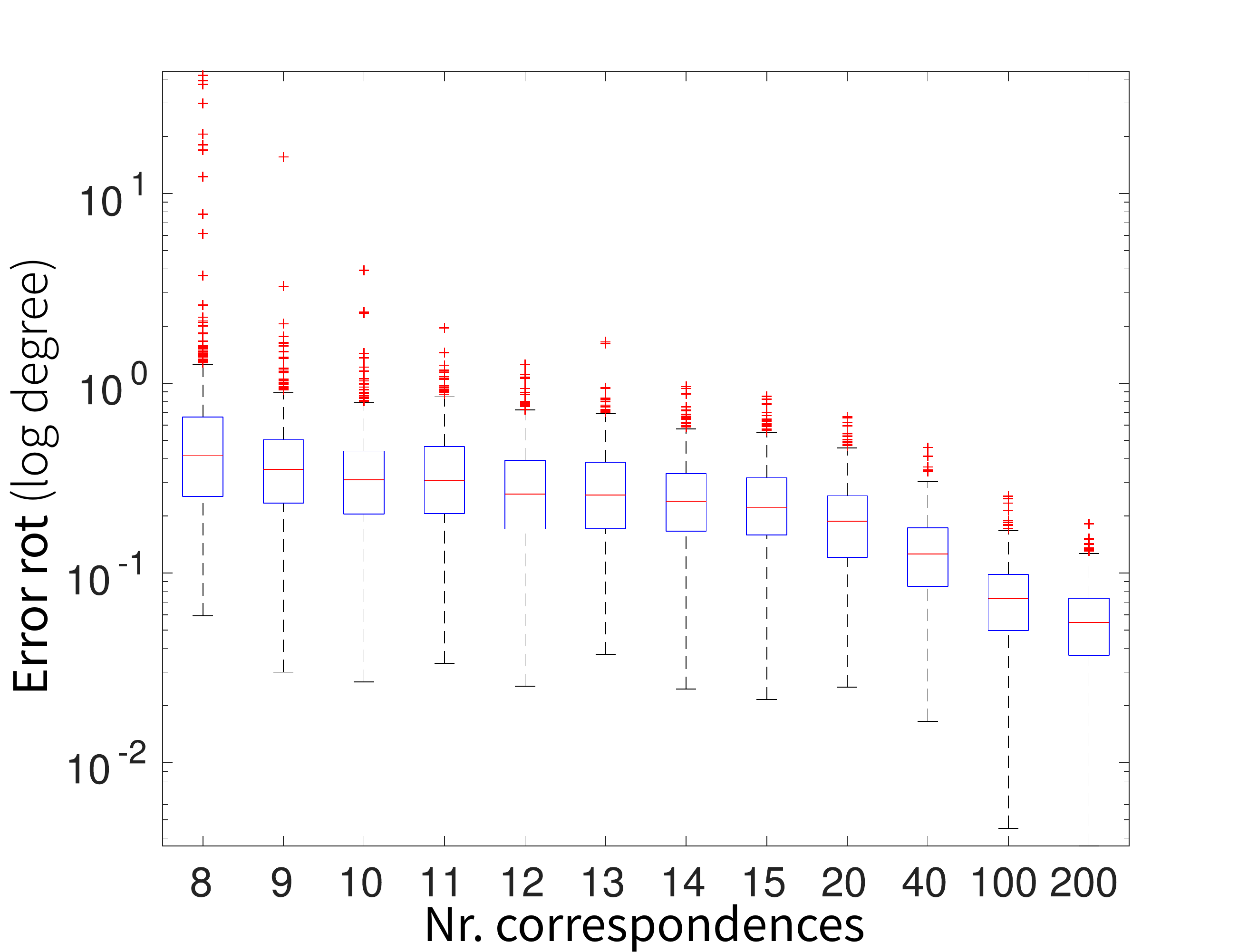}
        \caption{}
        \label{fig:error-rot-25}
    \end{subfigure}
    \caption{(a) We plot the error in rotation for the instances of the relative pose problem with noise 0.5 \pixels and (b) 2.5 \pixels. Note the logarithmic scale in the \textsc{Y} axis.}
\end{figure*}

\section{Performance of the Proposed Certifiable Pipeline for Different FoV and Maximum Parallax Values} \label{app:results-FOV-parallax-extra}

In this Section, we provide the results for the proposed pipeline for the experiments with fixed noise 0.5 \pixels and 
varying FoV and maximum parallax. 
We show the percentage of cases in which 
our algorithm could certify optimality (Figure \eqref{fig:FOVPARALLAXopt}. Figure \eqref{fig:FOVerror} depicts the error in rotation (degrees) for each FoV considered 70, 90, 120 and 150 degrees. 
In these cases, the images have size $1120$, $1600$, $2770$, and $5971$, respectively. 
The principal point for each case is placed at the center of the image plane.

On the other hand, Figure \eqref{fig:PARALLAXerror} 
depicts the error in rotation (degrees) for each parallax considered: 
1.0, 1.4, 2.5 and 4.0 meters.

\begin{figure*}
    \centering
    \begin{subfigure}[t]{0.49\textwidth}
        \centering
        \includegraphics[width=\figureSize]{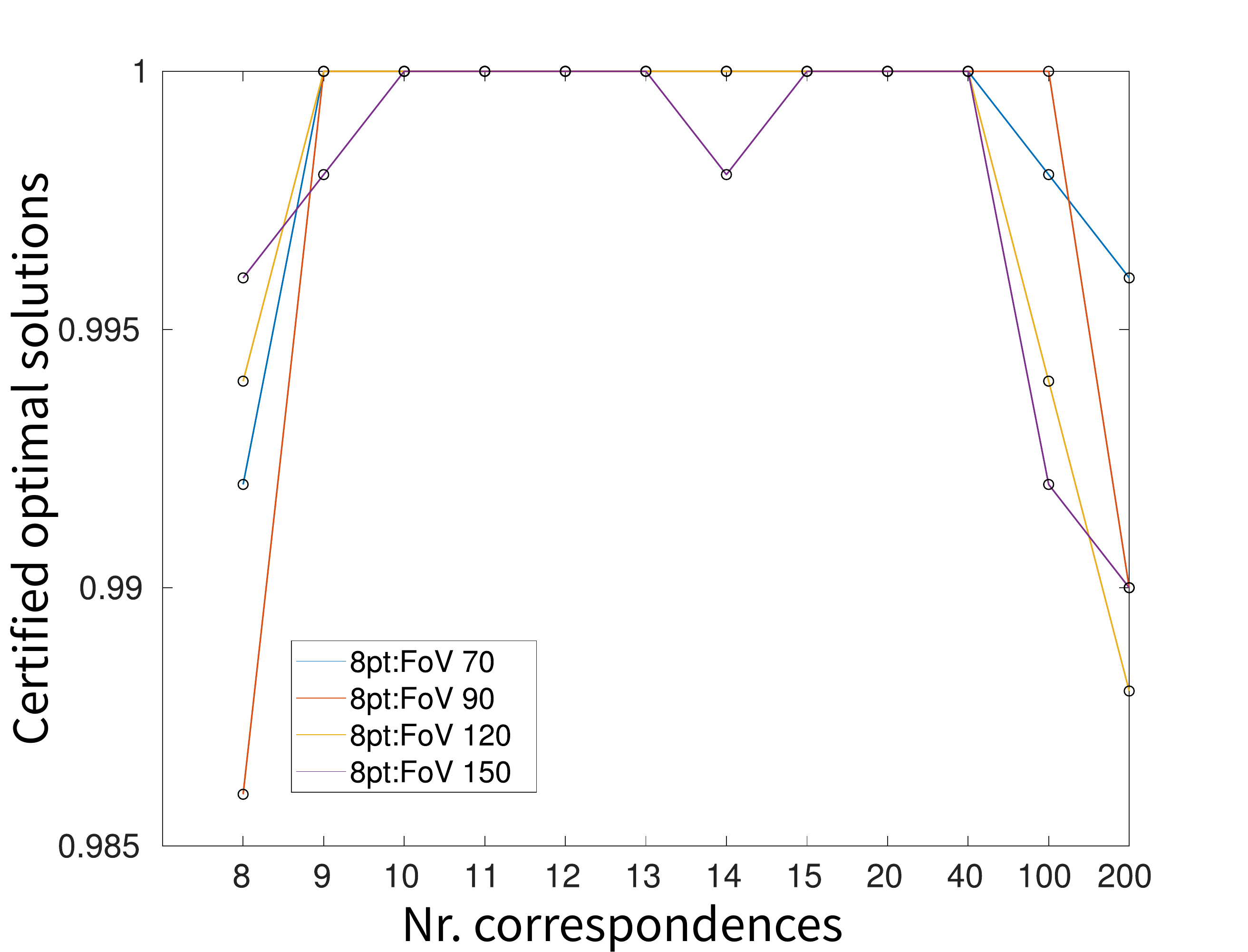}
        \caption{}
    \end{subfigure}
    \begin{subfigure}[t]{0.49\textwidth}
        \centering
        \includegraphics[width=\figureSize]{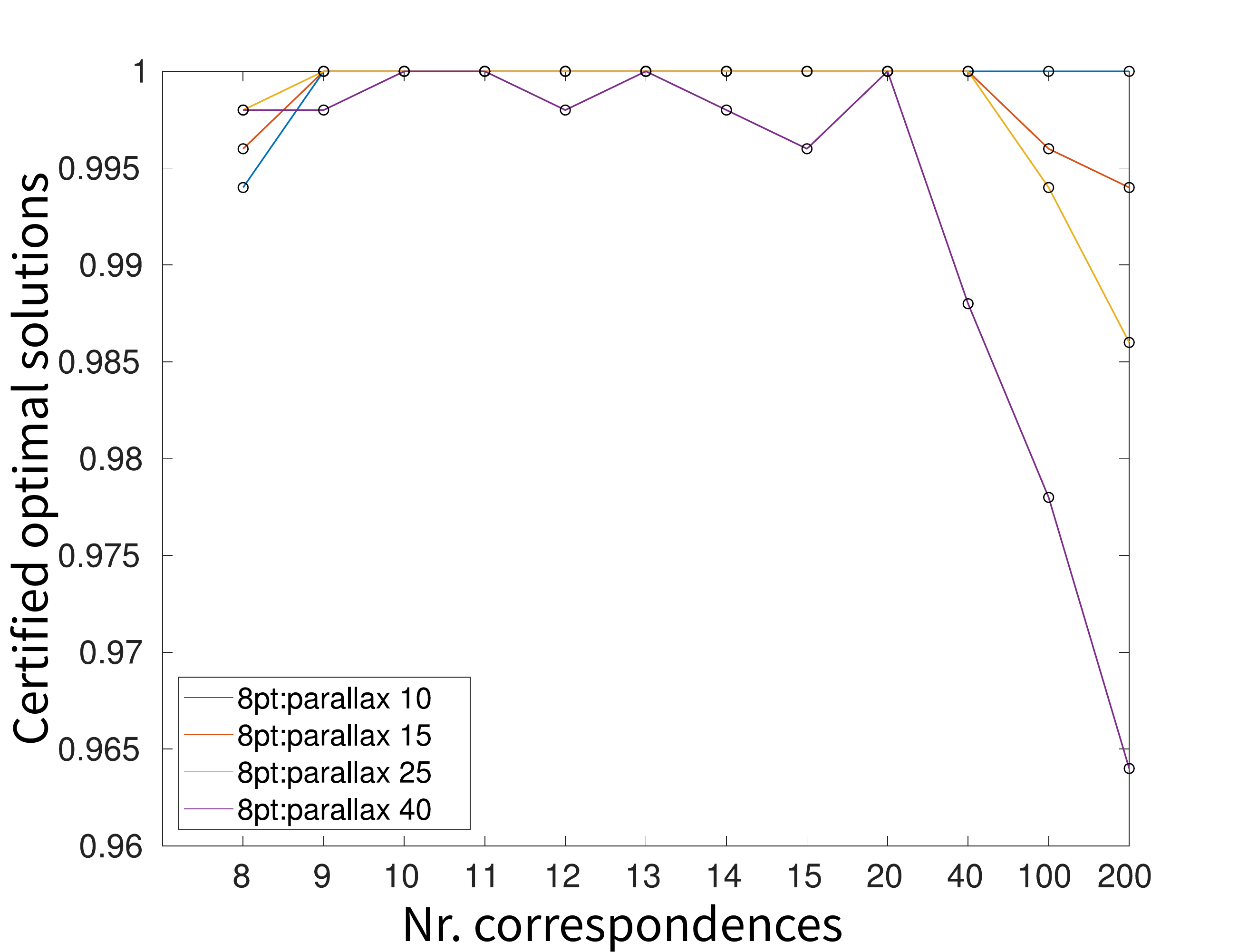}
        \caption{}
    \end{subfigure}
    \caption{(a) We plot the percentage of cases in which the algorithm could certify optimality for instances of the problem with fixed noise 0.5 \pixels  and varying FoV and (b) maximum parallax.}
    \label{fig:FOVPARALLAXopt}
    \centering
    \begin{subfigure}[t]{0.24\textwidth}
        \centering
        \includegraphics[width=\figureSize]{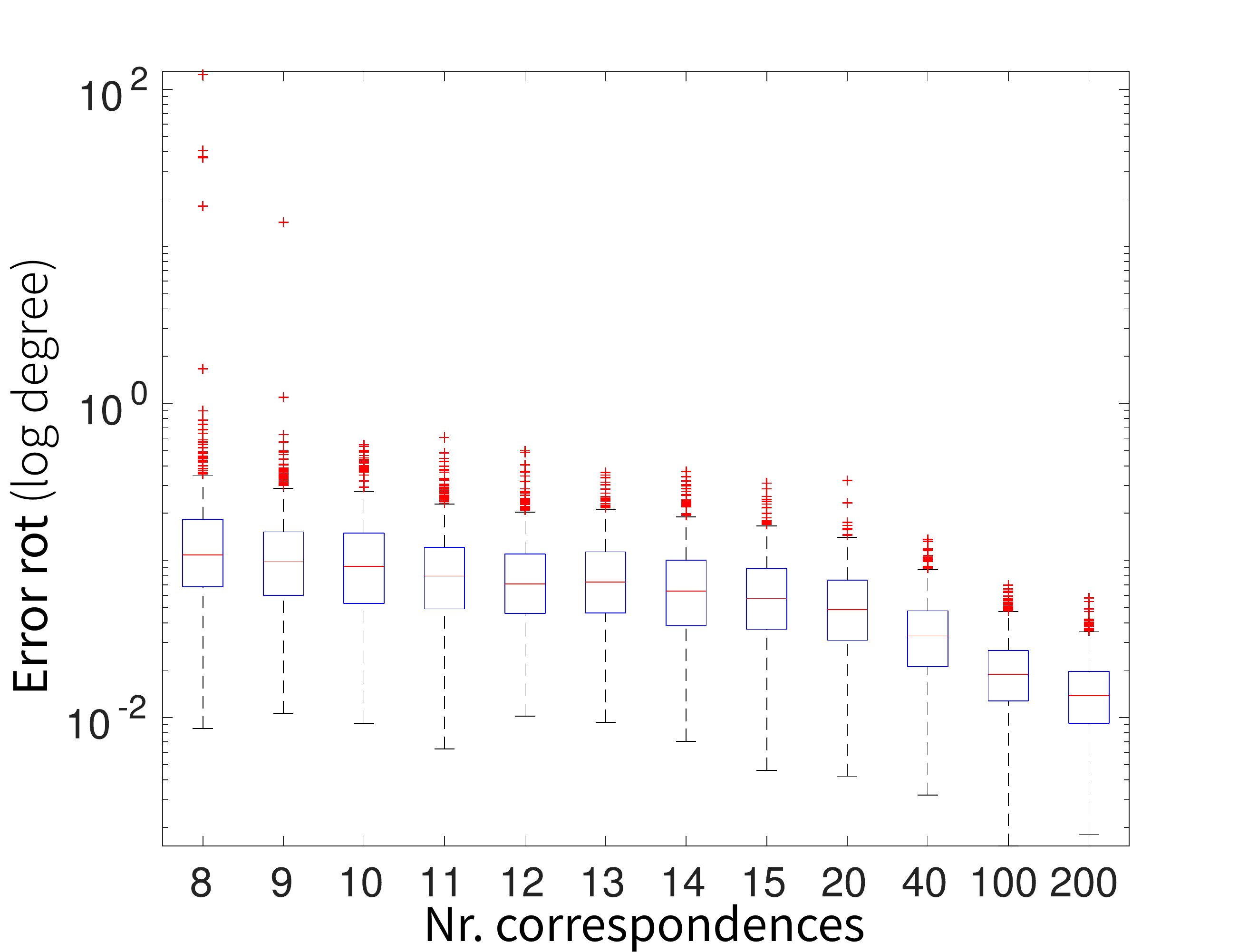}
        \caption{}
    \end{subfigure}
    \begin{subfigure}[t]{0.24\textwidth}
        \centering
        \includegraphics[width=\figureSize]{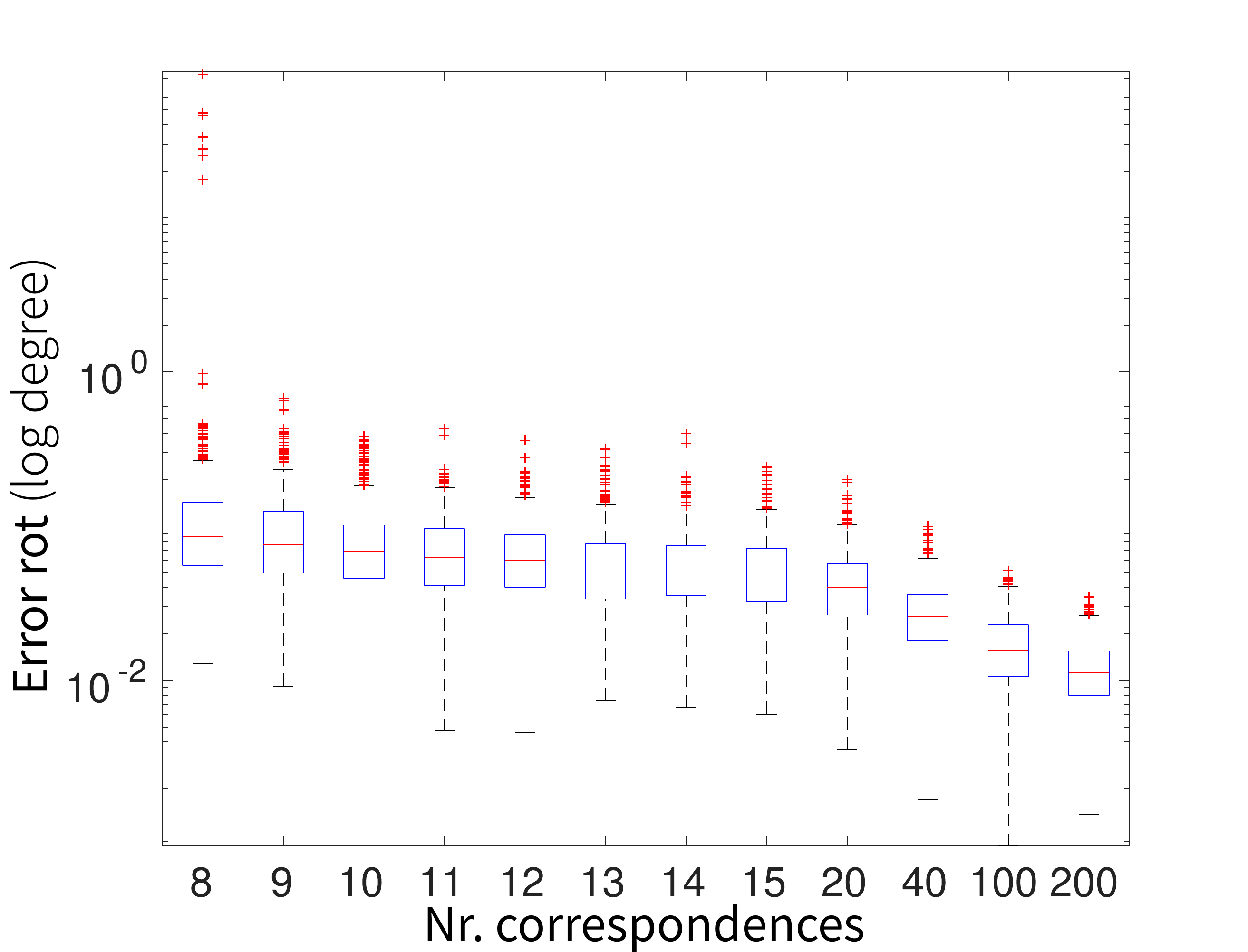}
        \caption{}
    \end{subfigure}
    \centering
    \begin{subfigure}[t]{0.24\textwidth}
        \centering
        \includegraphics[width=\figureSize]{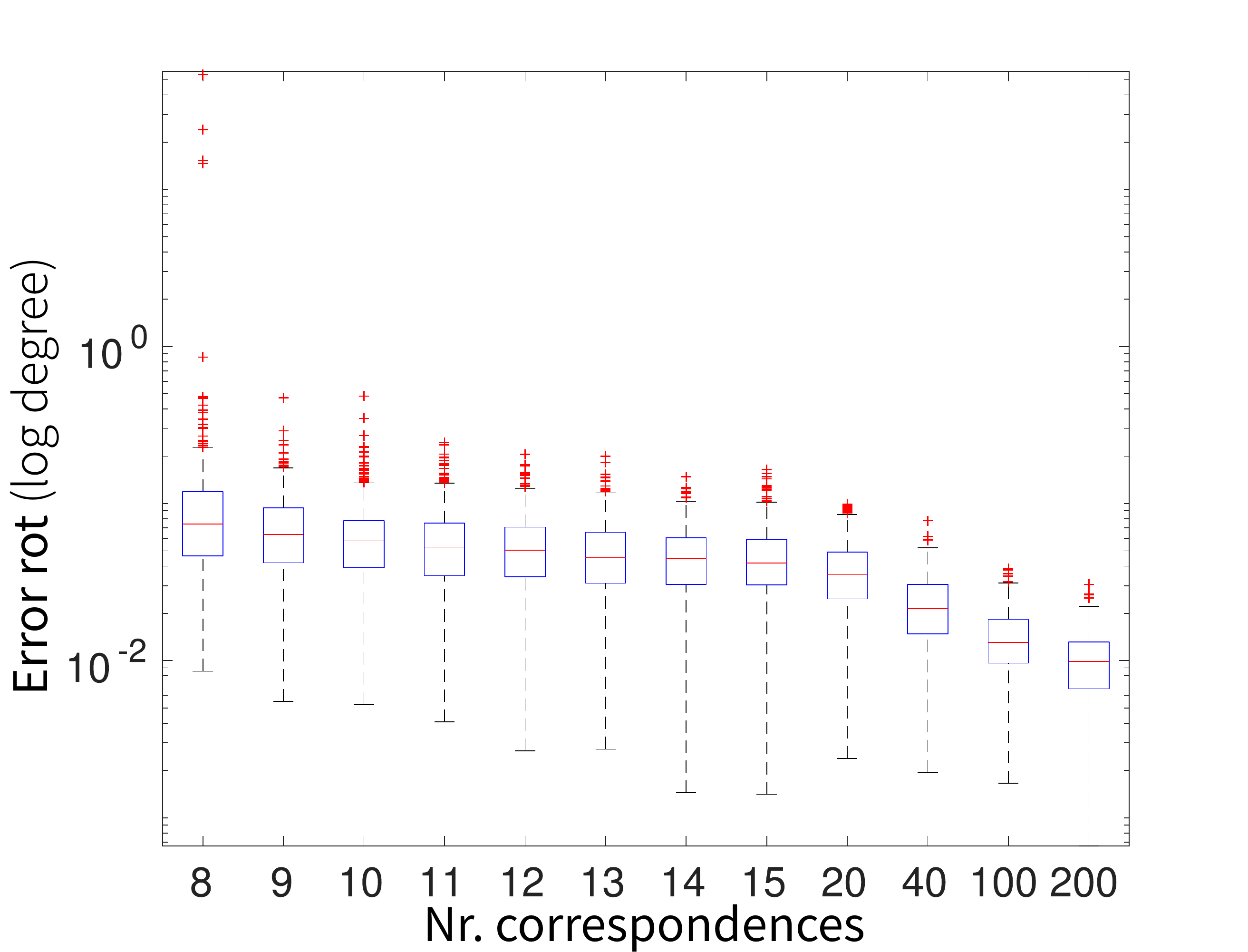}
        \caption{}
    \end{subfigure}
    \begin{subfigure}[t]{0.24\textwidth}
        \centering
        \includegraphics[width=\figureSize]{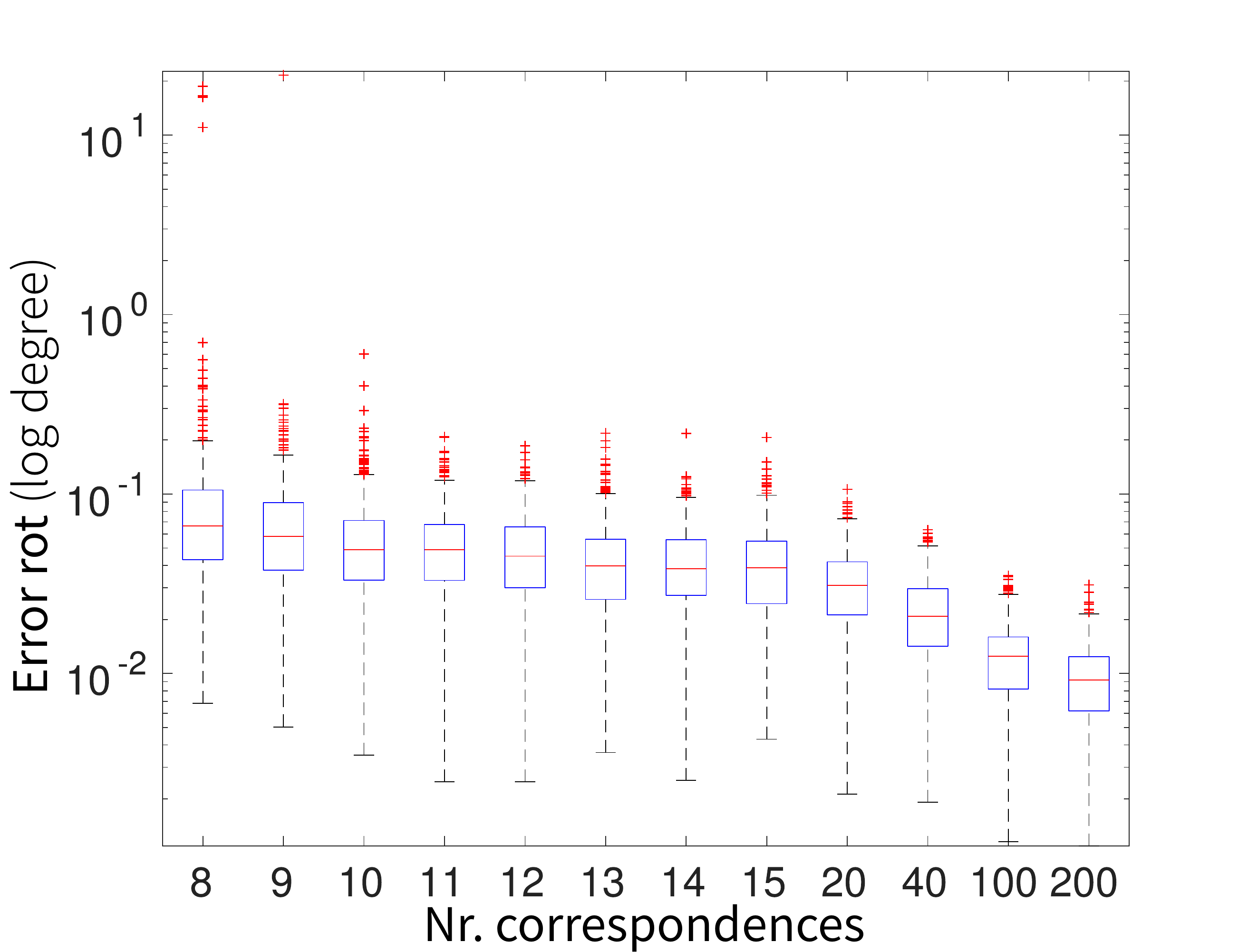}
        \caption{}
    \end{subfigure}
    \caption{Error in rotation for instances of the problem with fixed level of noise 0.5 \pixels and varying FoV (in degrees): 70 (a); 90 (b); 120 (c); and 150 (d).}
    \label{fig:FOVerror}

    \centering
    \begin{subfigure}[t]{0.24\textwidth}
        \centering
        \includegraphics[width=\figureSize]{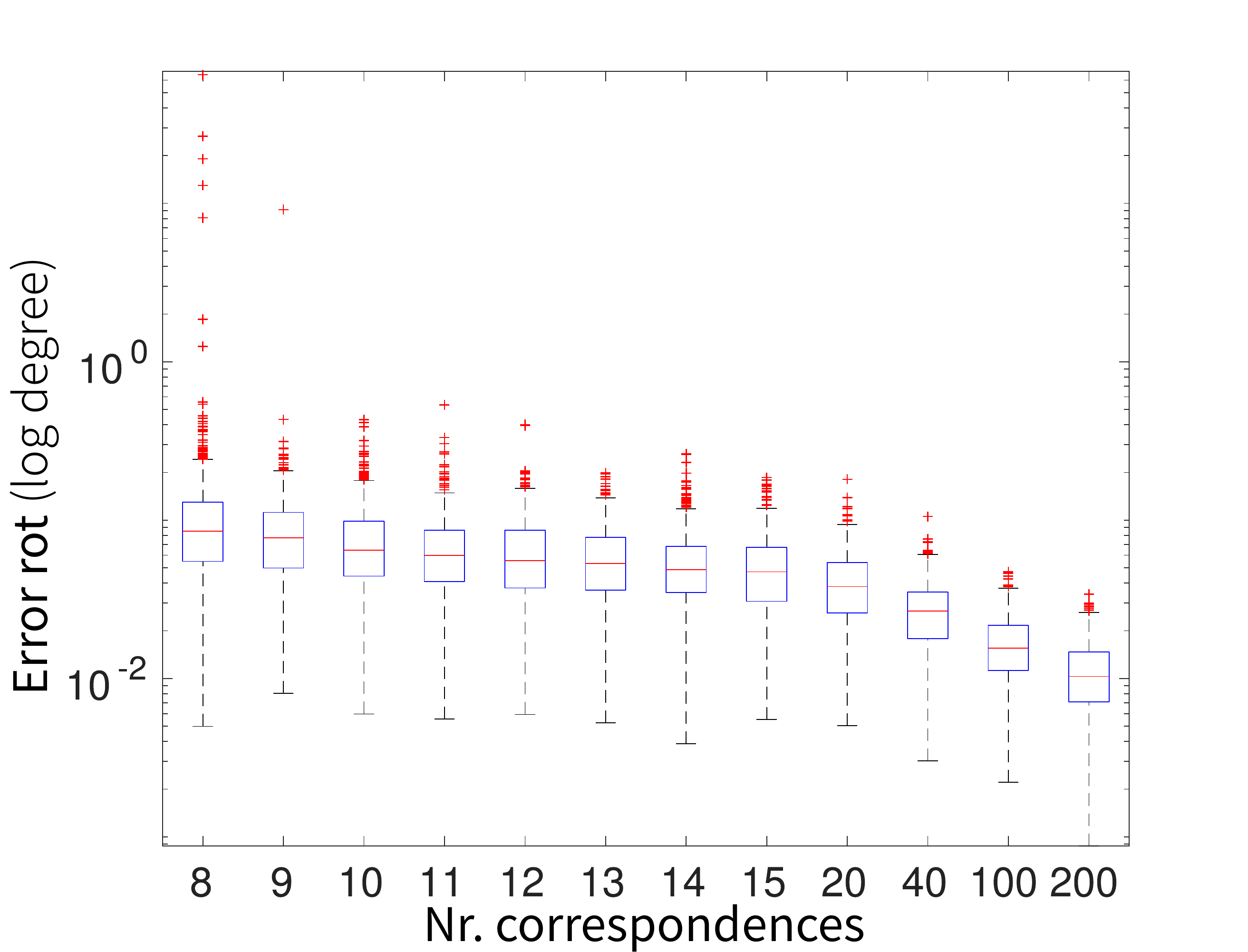}
        \caption{}
    \end{subfigure}
    \begin{subfigure}[t]{0.24\textwidth}
        \centering
        \includegraphics[width=\figureSize]{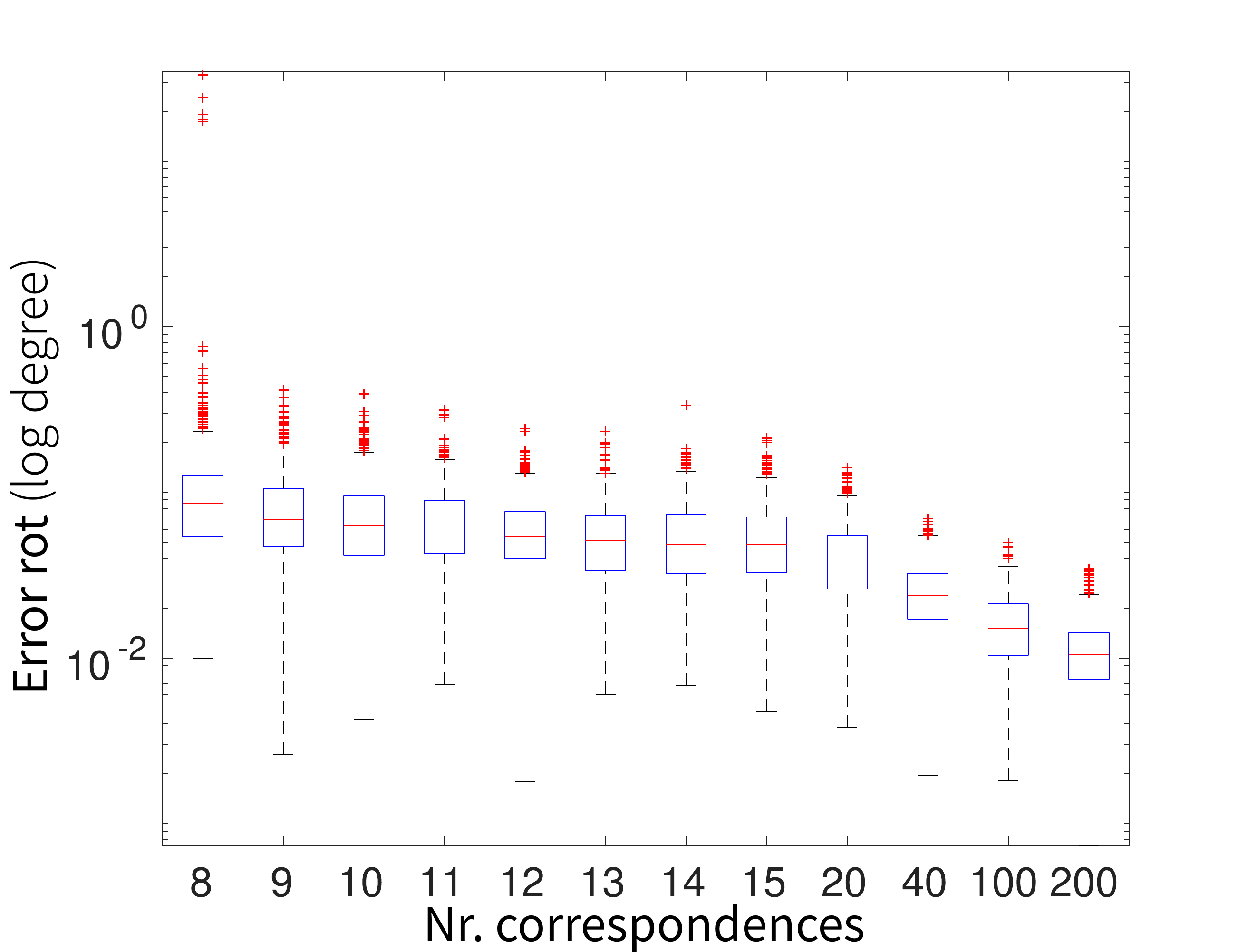}
        \caption{}
    \end{subfigure}
    \centering
    \begin{subfigure}[t]{0.24\textwidth}
        \centering
        \includegraphics[width=\figureSize]{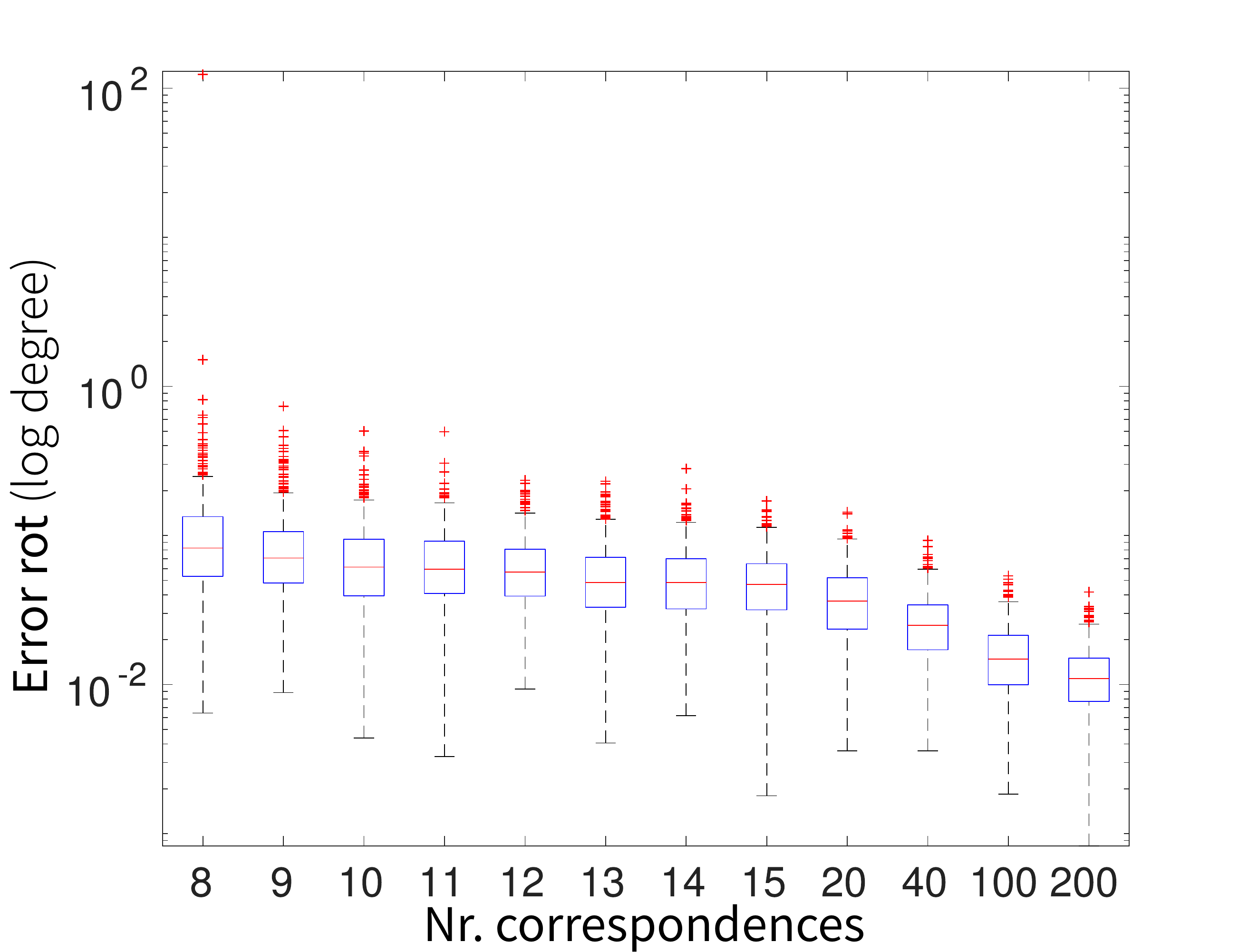}
        \caption{}
    \end{subfigure}
    \begin{subfigure}[t]{0.24\textwidth}
        \centering
        \includegraphics[width=\figureSize]{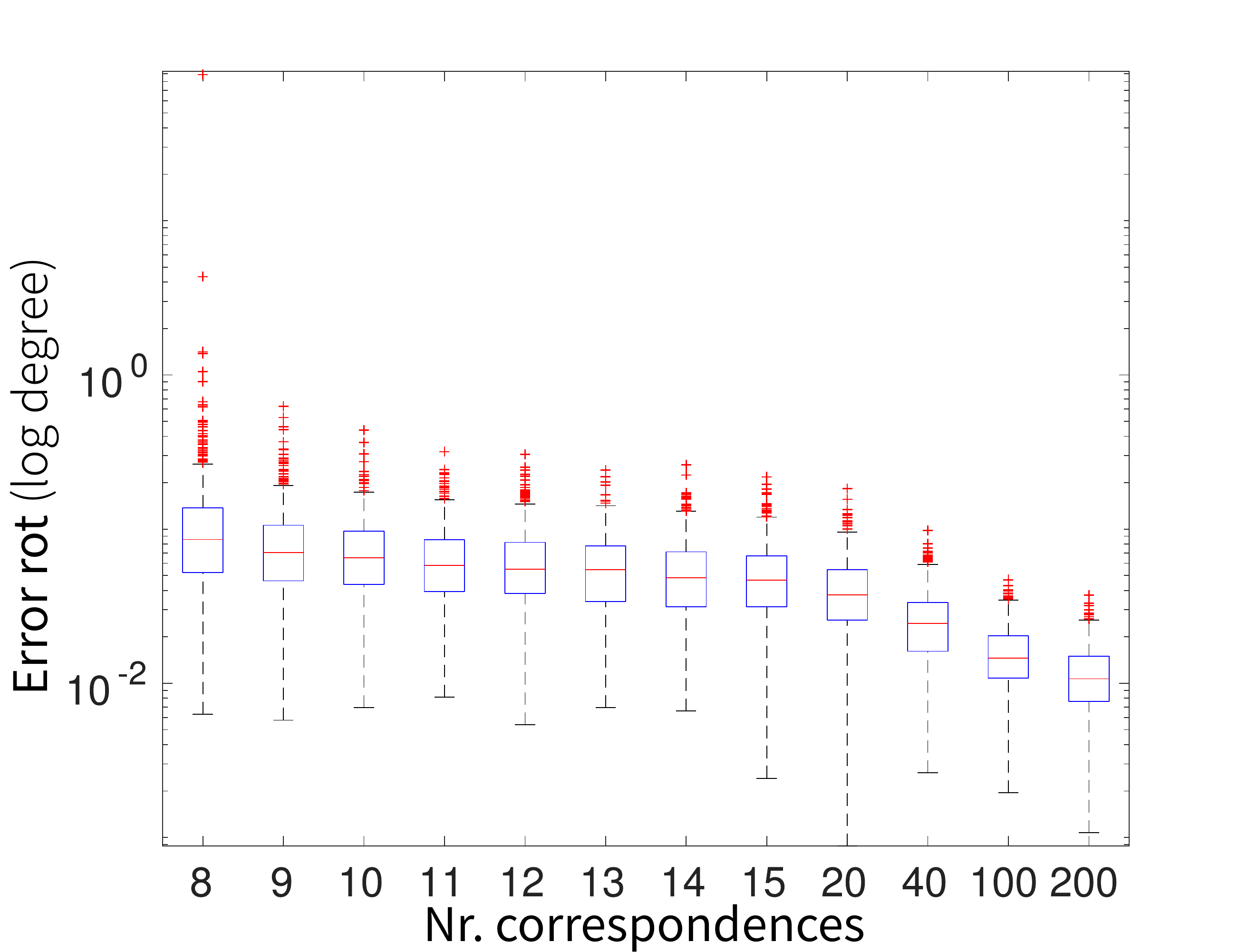}
        \caption{}
    \end{subfigure}
    \caption{Error in rotation for instances of the problem with fixed level of noise 0.5 \pixels and varying maximum parallax (in meters): 1.0 (a); 1.5 (b); 2.5 (c); and 4.0 (d).}
    \label{fig:PARALLAXerror}
\end{figure*}

\end{document}